\documentclass[11pt]{article}
\usepackage[toc,page]{appendix}
\usepackage{yfonts}
\usepackage{bbm}
\usepackage[numbers, compress]{natbib}

\usepackage{latexsym}
\usepackage{amsmath}
\usepackage{amssymb}
\usepackage{amsthm}
\usepackage{epsfig}
\usepackage[dvipsnames]{xcolor}
\usepackage{graphicx}
\usepackage{bm}
\usepackage{mathtools}

\usepackage[ruled,linesnumbered]{algorithm2e}

\usepackage[shortlabels]{enumitem}

\usepackage[top=1in, bottom=1in, left=1in, right=1in]{geometry}

\renewcommand{\P}{\mathbb{P}}
\newcommand{\E}{\mathbb{E}}

\newcommand{\Z}{\mathbb{Z}}
\newcommand{\R}{\mathbb{R}}

\newcommand{\N}{\mathbb{N}}

\newcommand{\eps}{\varepsilon} 

\def\id{{\mathbf I}}

\newcommand{\<}{\langle}
\renewcommand{\>}{\rangle}
\newcommand{\sign}{\text{sign}}

\newcommand{\op}{{\rm op}}

\def\sT{{\mathsf T}}
\def\bzero{{\boldsymbol 0}}

\newtheorem{theorem}{Theorem}
\newtheorem*{theorem*}{Theorem}
\newtheorem{lemma}{Lemma}

\newtheorem{assumption}{Assumption}
\newtheorem{definition}{Definition}
\newtheorem{proposition}{Proposition}

\newtheorem{claim}{Claim}
\newtheorem{conjecture}{Conjecture}
\newtheorem{corollary}{Corollary}

\theoremstyle{definition}

\newtheorem{remark}{Remark}

\usepackage{bbm}

\DeclareSymbolFont{rsfs}{U}{rsfs}{m}{n}
\DeclareSymbolFontAlphabet{\mathscrsfs}{rsfs}

\def\bI{{\boldsymbol I}}
\def\bH{{\boldsymbol H}}

\def\bM{{\boldsymbol M}}

\def\bP{{\boldsymbol P}}

\def\bR{{\boldsymbol R}}

\def\bU{{\boldsymbol U}}
\def\bV{{\boldsymbol V}}
\def\bW{{\boldsymbol W}}

\def\ba{{\boldsymbol a}}

\def\be{{\boldsymbol e}}

\def\bg{{\boldsymbol g}}

\def\bk{{\boldsymbol k}}
\def\bs{{\boldsymbol s}}

\def\bu{{\boldsymbol u}}
\def\bv{{\boldsymbol v}}
\def\bw{{\boldsymbol w}}
\def\bx{{\boldsymbol x}}

\def\bz{{\boldsymbol z}}

\def\bdelta{{\boldsymbol\delta}}

\def\bTheta{{\boldsymbol \Theta}}

\def\hf{{\hat f}}

\def\de{{\rm d}}

\def\de{{\rm d}}
\def\Unif{{\rm Unif}}

\def\cO{{\mathcal O}}

\def\cC{{\mathcal C}}

\def\cL{{\mathcal L}}
\def\cF{{\mathcal F}}

\def\cS{{\mathcal S}}

\def\cO{{\mathcal O}}

\def\cH{{\mathcal H}}

\def\Unif{{\sf Unif}}
\def\normal{{\sf N}}
\def\proj{{\mathsf P}}

\def\NN{{\sf NN}}

\def\naturals{{\mathbb N}}

\def\normal{{\sf N}}
\def\proj{{\mathsf P}}

\def\Unif{{\sf Unif}}
\def\normal{{\sf N}}

\def\proj{{\mathsf P}}

\def\NN{{\sf NN}}

\def\naturals{{\mathbb N}}
\def\proj{{\mathsf P}}

\def\bz{{\boldsymbol z}}
\def\proj{{\mathsf P}}
\def\He{{\rm He}}

\def\bdelta{{\boldsymbol\delta}}

\def\be{{\boldsymbol e}}
\def\bu{{\boldsymbol u}}
\def\bg{{\boldsymbol g}}

\def\bv{{\boldsymbol v}}

\def\bTheta{{\boldsymbol \Theta}}

\def\NN{{\sf NN}}

\def\naturals{{\mathbb N}}

\def\bw{{\boldsymbol w}}
\def\de{{\rm d}}
\def\bx{{\boldsymbol x}}

\def\bW{{\boldsymbol W}}
\def\ba{{\boldsymbol a}}
\def\cF{{\mathcal F}}
\def\Unif{{\rm Unif}}
\def\bU{{\boldsymbol U}}
\def\bV{{\boldsymbol V}}
\def\bz{{\boldsymbol z}}
\def\proj{{\mathsf P}}
\def\He{{\rm He}}

\def\normal{{\sf N}}

\def\be{{\boldsymbol e}}
\def\bu{{\boldsymbol u}}
\def\bg{{\boldsymbol g}}

\def\bH{{\boldsymbol H}}

\def\cS{{\mathcal S}}

\def\be{{\boldsymbol e}}
\def\bu{{\boldsymbol u}}
\def\bg{{\boldsymbol g}}

\def\bTheta{{\boldsymbol \Theta}}

\def\bP{{\boldsymbol P}}

\def\tbu{\Tilde \bu}

\def\hf{\hat f}

\def\bR{{\boldsymbol R}}

\def\cC{\mathcal{C}}

\def\cH{\mathcal{H}}

\def\bzeta{{\boldsymbol \zeta}}

\def\ind{\mathbbm{1}}

\def\eps{\varepsilon}

\def\Span{\mathsf{span}}

\def\Leap{\mathrm{Leap}}

\def\tbw{\widetilde \bw}
\def\obg{\overline{\bg}} 
\def\og{\overline{g}}
\def\bm{\boldsymbol{m}}
\def\oT{\overline{T}}

\def\bones{{\boldsymbol 1}}
\def\grad{\text{grad}}

\def\obw{\overline{\bw}}
\def\tbw{\widetilde{\bw}}

\def\uD{\underline{D}}
\def\uM{\underline{M}}

\def\tbv{\widetilde{\bv}}
\def\tv{\widetilde{v}}
\def\op{\overline{p}}
\def\up{\underline{p}}
\def\oD{\overline{D}}
\def\oM{\overline{M}}

\def\tw{\widetilde{w}}
\def\tC{\widetilde{C}}
\def\ow{\overline{w}}
\def\otau{\overline{\tau}}

\def\ot{\overline{t}}
\def\oD{\overline{D}}

\def\bU{{\boldsymbol U}}

\DeclarePairedDelimiter{\ceil}{\lceil}{\rceil}
\DeclarePairedDelimiter{\floor}{\lfloor}{\rfloor}

\def\balpha{{\boldsymbol \alpha}}
\newcommand{\poly}{\mathrm{poly}}

\newcommand{\fNN}{\hf_{\NN}}

\usepackage{hyperref}
\hypersetup{
    colorlinks,
    linkcolor={blue!80!black},
    citecolor={green!50!black},
}
\colorlet{linkequation}{blue}

\title{
SGD learning on neural networks: \\
leap complexity and saddle-to-saddle dynamics
}

\author{Emmanuel Abbe\thanks{Mathematics Institute, EPFL}, \;\; Enric Boix-Adser\`a\thanks{Department of Electrical Engineering and Computer Science, MIT}, \;\; Theodor Misiakiewicz\thanks{Department of
    Statistics, Stanford University} }

\begin{document}

\maketitle

\begin{abstract}%
   We investigate the time complexity of SGD learning on fully-connected neural networks with isotropic data. We put forward a complexity measure,
{\it the leap}, which measures how “hierarchical” target functions are. 
For $d$-dimensional uniform Boolean or isotropic Gaussian data, our main conjecture states that the time complexity to learn a function $f$ with low-dimensional support is $$\Tilde \Theta (d^{\max(\mathrm{Leap}(f),2)}) \,\,.$$   
We prove a version of this conjecture for a class of functions on Gaussian isotropic data and 2-layer neural networks, under additional technical assumptions on how SGD is run. We show that the training  sequentially learns the function support with a saddle-to-saddle dynamic. Our result departs from \cite{abbe2022merged} by going beyond leap 1 (merged-staircase functions), and by going beyond the mean-field and gradient flow approximations that prohibit the full complexity control obtained here.
Finally, we note that this gives an SGD complexity for the full training trajectory that matches that of Correlational Statistical Query (CSQ) lower-bounds. 
\end{abstract}

\tableofcontents

\section{Introduction}\label{intro}

Deep learning has emerged as the standard approach to exploiting massive high-dimensional datasets. At the core of its success lies its capability to learn effective features with fairly blackbox architectures without suffering from the curse of dimensionality. To explain this success, two structural properties of data are commonly conjectured: (i) a {\it low-dimensional} structure that SGD-trained neural networks are able to adapt to; (ii) a {\it hierarchical} structure that neural networks can leverage with SGD training. In particular,  
\begin{description}
\item[From a statistical viewpoint:] A line of work \citep{bach2017breaking,schmidt2020nonparametric,kohler2016nonparametric,bauer2019deep} has investigated the sample complexity of learning with deep neural networks, decoupled from computational considerations. By directly considering global solutions of empirical risk minimization (ERM) problems over arbitrarily large neural networks and sparsity inducing norms, they showed that deep neural networks can overcome the curse of dimensionality on classes of functions with low-dimensional and hierarchical structures. However, this approach does not provide efficient algorithms: instead, a number of works have shown computational hardness of ERM problems \citep{blum1988training,klivans2009cryptographic,daniely2014average} and it is unclear how much this line of work can inform practical neural networks, which are trained using SGD and variants.

\item[From a computational viewpoint:] A line of work in computational learning theory has provided time- and sample-efficient algorithms for learning  Boolean functions with low-dimensional structure,
based on their sparse Fourier spectrum
\citep{mansour1994,odonnell_2014}. However, these algorithms are a priori quite different from SGD-trained neural networks. While unconstrained architectures can emulate any efficient learning algorithms \citep{abbe2020universality,abbe2021power}, it is unclear whether more `standard' neural networks can succeed on these same classes of functions or whether they require additional structure that pertains to hierarchical properties.

\end{description}

Thus, an outstanding question emerges from the current state of affairs:

\begin{center}
{\em For neural networks satisfying ``regularity assumptions'' (e.g., fully-connected, isotropically initialized layers), are there structural properties of the data that govern the time complexity of SGD learning? How does SGD exploit these properties in its training dynamics?}
\end{center}

Here the key points are: (i) the ``regularity'' assumption, which prohibits the use of unorthodox neural networks that can emulate generic PAC/SQ learning algorithms as in \cite{abbe2020universality,abbe2021power}; (ii) the requirement on the time complexity, which prohibits direct applications of infinite width, continuous time or infinite time analyses as in \cite{chizat2018global,chizat2020implicit}. We discuss in Section \ref{related} the various works that  made progress towards the above, in particular regarding single- and multi-index models. We now specify the setting of this paper.

\paragraph{IID inputs and low-dimensional latent dimension.} 
We focus on the following class of data distributions. First of all, we consider IID inputs, i.e., 
\begin{align}
\bx=(x_1, \dots, x_d) \sim \mu^{\otimes d},
\end{align}
and we focus on the case where $\mu$ is either  $\mathcal{N}(0,1)$ or $\Unif(\{+1,-1\})$, although we expect that other distributions would admit a similar treatment. Incidentally, the latter distribution is of interest in reasoning tasks related to Boolean arithmetic or logic \citep{saxton2019analysing,pvr,l2r}. 
We now make a key assumption on the target function, that of having a {\it low latent dimension}, i.e., $f_*(\bx)=h_*(\bz)$ where $\bz=\bM\bx$ and
\begin{equation}\label{eq:support_bool_gauss}
\begin{aligned}
&\text{(Gaussian case) \,\,\,  $\bM$ a $P \times d$ dimensional, real-valued matrix such that $\bM\bM^{\top}=\bI_P$}\\
&\text{(Boolean case) \,\,\, $\bM$ a $P \times d$ dimensional, $\{0,1\}$-valued matrix such that $\bM\bM^{\top}=\bI_P$}
\end{aligned}
\end{equation}
with the assumption that 
$P=O_d(1)$. In other words, the target function has a large ambient dimension but depends only on a finite number of latent coordinates. In the Gaussian case the coordinates are not known because of a possible rotation of the input, and in the Boolean case the coordinates are not known because of a possible permutation of the input.

Data with large ambient dimension but low latent dimension have long been a center of focus in machine learning and data science. It is known that kernel methods cannot exploit low latent dimension, i.e., it was proved in \cite{hsu2021approximation,hsudimension,kamath_dim,abbe2022merged} that any kernel method needs a number of features $p$ or samples $n$ satisfying 
\begin{align}
\min(n,p) \ge \Omega(d^D)
\end{align}
in order to learn a Boolean function as above with degree $D=O_d(1)$. 
In other words, for kernel methods $D$ controls the sample complexity irrespective of any potential additional structural properties of $f_*$ (e.g., hierarchical properties). On the other hand, it is known that this is not the limit for deep learning, which can break the $d^D$ curse, as discussed next. 

\paragraph{The example of staircases.} Consider the following example: $\bx \sim \Unif (\{+1 , - 1\}^d)$ is drawn from the hypercube and the target function is $4$-sparse, either
 \[
 h_{*,1} (\bz)= z_1 + z_1z_2 + z_1z_2z_3 + z_1z_2z_3z_4 \, , \qquad \mbox{or} \qquad  h_{*,2} (\bz) = z_1 z_2 z_3 z_4 \, .
 \]
The first function is called a vanilla staircase of degree 4 \citep{abbe2021staircase,abbe2022merged}. The second is a monomial of degree 4.
Each of these functions induces a function class under the permutation of the variables (i.e., one can consider the class of all monomials on any 4 of the $d$ input variables, and similarly for staircases).
One can verify that these function classes have similar approximation and statistical complexity because of the low-dimensional structure, but have different computational complexity because of the hierarchical structure. For example, under the Correlational Statistical Query (CSQ) model of computation \citep{bendavid1995learning,kearns1998efficient,bshouty2002using}, the first class has CSQ dimension $\Theta(d)$ versus $\Theta(d^4)$ for the second class\footnote{See Section~\ref{csq} for more details on CSQ.}.

Consider now learning these two functions with online-SGD\footnote{Online-SGD means that on each SGD iteration a fresh sample $(\bx^t,y^t)$ is used.} on a two-layer neural network  $\hf_{\NN} (\bx ; \bTheta) = \sum_{j \in [M]} a_j \sigma ( \< \bw_j , \bx \> + b_j)$.
For online-SGD in the mean-field scaling, it was shown in \cite{abbe2022merged} that $ h_{*,1}$ can be learned in $\Theta_d (d)$ steps, while $h_{*,2}$ cannot be learned in $O_d(d)$ steps, but it was not shown in which complexity $h_{*,2}$ could be learned. How can we understand this? At initialization, the gradient of the neural network has correlation with each monomial of order $O(d^{-(k-1)/2})$ inside the support and $O(d^{-(k+1)/2})$ outside the support (for a degree $k$-monomial). In the first case, the gradient has $O(1)$ correlation with the first monomial $z_1$ and can learn the first coordinate, then the second coordinate becomes easier to learn using the second monomial, and so on and so forth. In the second case, the correlation is of order $d^{-3/2}$ on the support and SGD needs to align to all $4$ coordinates at once, which takes more time.  Indeed, we will see that to align with $k$ coordinates at once, we need $\Tilde O (d^{k-1})$ steps, which matches (up to logarithmic factors) the computational lower bound of CSQ algorithms.

In this paper, we treat these functions in a unified manner, with the ``leap'' complexity measure, where $h_{*,1}$ and $h_{*,2}$ are leap-1 and leap-4 functions respectively. Leap-$k$ functions will be learned in $\tilde{\Theta}(d^{\max(k-1,1)})$ online-SGD steps. Note that going from staircase functions having leap-1 to more general functions of arbitrary (finite) leaps is highly non-trivial. This is because the mean-field  gradient flow used in \cite{abbe2022merged} cannot be used beyond the scaling of $O(d)$ steps, as required for $k > 1$, because the mean-field PDE approximation breaks down \citep{mei2019mean}.

\begin{figure}
    \centering
    \includegraphics[width=0.87\textwidth]{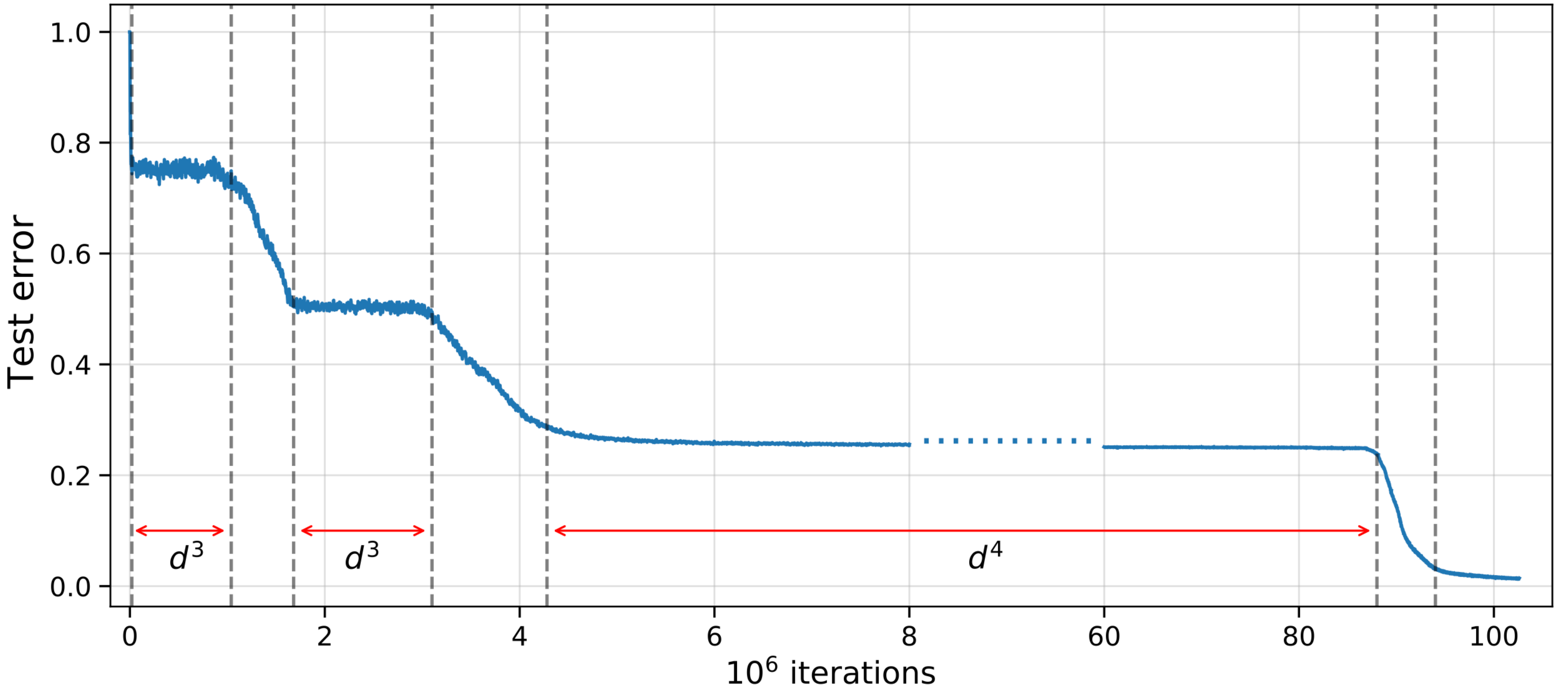}
    \caption{Test error versus the number of online-SGD steps to learn $h_*(\bz) = z_1 + z_1 z_2 \cdots z_5 + z_1 z_2 \cdots z_9 + z_1z_2 \cdots z_{14}$ in ambient dimension $d=100$ on the hypercube. We take $M = 300$ neurons with shifted sigmoid activation and train both layers at once with constant step size $0.4/d$. The SGD dynamics follows a saddle-to-saddle dynamic and sequentially picks up the support and monomials $z_1$ in roughly $d$ steps, $z_1z_2 \cdots z_5$ in $d^3$ steps (leap of size $4$), $z_1z_2 \cdots z_9$ in $d^3$ steps (leap of size $4$) and $z_1z_2 \cdots z_{14}$ in $d^4$ steps (leap of size $5$).}
    \label{fig:StoS_punchy}
\end{figure}

\subsection{The leap complexity}\label{sec:leap}
We now define the leap complexity. Any function in $L^2 (\mu^{\otimes P})$ can be expressed in the orthogonal basis of $L^2 (\mu^{\otimes P})$, i.e., the Hermite or Fourier-Walsh basis for $\mu \sim \normal (0,1)$ and $\mu \sim \Unif (\{+1,-1\})$ respectively, 
\begin{align}\label{eq:decompo_monomials}
h_*(\bz) =\sum_{S \in \mathcal{Z}^P} \hat{h}_*(S) \chi_S(\bz),
\end{align}
where $\mathcal{Z}=\{0,1\}$ for the Boolean case and $\mathcal{Z}=\mathbb{Z}_+$ for the Gaussian case, $\chi_S(\bz) = \prod_{i \in [P]} \chi_{S_i}(z_i)$, 
\begin{align}
\chi_{S_i}(z_i) =\begin{cases}
z_i^{S_i} &\text{ (Boolean case)}\\
\mathrm{He}_{S_i}(z_i) & \text{ (Gaussian case)}
\end{cases}
\end{align}
where $\mathrm{He}_k$ is the $k$-th Hermite polynomial, $k \in \mathbb{Z}_+$. The leap is given as follows.
\begin{definition}[Leap complexity] \label{def:leap_general} Let $h_*$ be as before with non-zero basis elements given by the subset $\cS (h_*) := \{ S_1 , \ldots , S_m \}$, $m \in \mathbb{Z}_+$. We define the leap complexity of $h_*$ as\footnote{$\Pi_{m}$ is the symmetric group of permutations on $[m]$.} 
\[
\mathrm{Leap}(h_*) : = \min_{\pi \in \Pi_{m}} \max_{i \in [m]} \big\| S_{\pi (i)}\setminus \cup_{j=0}^{i-1} S_{\pi(j)} \big\|_1 \,,
\]
where, for $S_j = (S_j(1) , \ldots , S_j(P) )$ in $\{0,1\}^P$ or $\mathbb{Z}_+^P$ for the Boolean or Gaussian case respectively, 
$\| S_{\pi (i)}\setminus \cup_{j=0}^{i-1} S_{\pi(j)} \|_1 := \sum_{k \in [P]} S_{\pi(i)}(k) \ind \{S_{\pi(j)}(k) = 0, \forall j \in [i-1] \}$, 
with $S_{\pi(0)}=0^P$. We then say that $h_*$ is a leap-$\mathrm{Leap}(h_*)$ function.
\end{definition}
In words, a function $h_*$ is leap-$k$ if its non-zero monomials can be ordered in a sequence such that each time a monomial is added, the support of $h_*$ grows by at most $k$ new coordinates, where each new coordinate is counted with multiplicity in the Gaussian case (and the 1-norm collapses to the cardinality of the difference set in the Boolean case). Note that the definition of leap-$k$ functions on the hypercube generalizes the definition of functions with the merged-staircase property (leap-1 functions) from \cite{abbe2022merged}.

Some examples in the Boolean case,
\[
\begin{aligned}
&\mathrm{Leap} ( z_1 + z_1z_2 + z_1z_2z_3 + z_1z_2z_3z_4 ) = 1\, , \qquad  && \mathrm{Leap} ( z_1 + z_2 + z_2z_3 z_4 ) = 2\, ,  \\
&\mathrm{Leap} ( z_1 + z_1z_2z_3 + z_2z_3 z_4 z_5z_6z_7) = 4\, , \qquad &&\mathrm{Leap} ( z_1 z_2 z_3 + z_2 z_3 z_4) = 3\, ,
\end{aligned}
\]
and on isotropic Gaussian data,
\[
\begin{aligned}
&\mathrm{Leap} ( \He_k (z_1) ) = \Leap(\He_1(z_1) \He_1 (z_2) \cdots \He_1(z_k)) = k \, ,  \\
&\mathrm{Leap} ( \He_{k_1} (z_1) + \He_{k_1} (z_1) \He_{k_2} (z_2) + \He_{k_1} (z_1) \He_{k_2} (z_2)\He_{k_3} (z_3) ) = \max(k_1 , k_2, k_3) \, ,  \\
&\mathrm{Leap} ( \He_2 (z_1) + \He_2 (z_2 ) + \He_2 (z_3) + \He_3(z_1)\He_8(z_3) ) = 2\, .
\end{aligned}
\]

\subsection{Summary of our contributions}
\label{sec:summary}

 \paragraph{Overview.} This paper puts forward a general conjecture characterizing the time complexity of SGD-learning on regular neural networks with isotropic data of low latent dimension. The key quantity that emerges to govern the complexity is the \textit{leap} (Definition \ref{def:leap_general}). 
This gives a formal measure of ``hierarchy'' in target functions, going beyond spectrum sparsity and emerging from the study of SGD-trained regular networks. The paper then proves a specialization of the conjecture to a representative class of functions on Gaussian inputs, but  for 2-layer neural networks and with certain technical assumptions on how SGD is run. The two main innovations of the proof are (i) a full control of the time complexity of SGD learning on a fully-connected network (without infinite width or continuous time approximations); (ii) going beyond one-step gradient analyses and showing that the leap controls the entire learning trajectory due to a sequential learning mechanism (saddle-to-saddle). We also provide experimental evidence towards the more general conjecture with vanilla SGD and derive CSQ lower-bounds for noisy GD that match our achievability bounds. 

\begin{conjecture}\label{conj:leap}
Let $f_* : \R^d \to \mathbb{R}$ in $L_2(\mu^{\otimes d})$ for $\mu$ either $\mathcal{N}(0,1)$ or $\Unif\{+1,-1\}$ satisfying the low-latent-dimension hypothesis $f_* (\bx) = h_* ( \bM \bx)$ in \eqref{eq:support_bool_gauss} for some $P=O_d(1)$. 
Let $\hf_{\NN}^{t}$ be the output of training a fully-connected neural network with $\mathrm{poly}(d)$-edges and rotationally-invariant weight initialization with $t$ steps of online-SGD on the square loss. Then, for all but a measure-$0$ set of functions (see below), the risk is bounded by
\begin{align*}
R(\hf_{\NN}^t) := \E_{\bx}\Big[\big(\hf_{\NN}^t(\bx) - f_*(\bx)\big)^2\Big] \le \eps \;\;\; \text{ if and only if } \;\;\; t = \tilde{\Omega}_d ( d^{\mathrm{Leap}(h_*)-1 \vee 1}) \mathrm{poly}(1/\eps)\,.
\end{align*}
So the total time complexity\footnote{The total time complexity is given by computing for each neuron and SGD step, a gradient in $d$ dimensions.} is $\tilde{\Omega}_d ( d^{\mathrm{Leap}(h_*)\vee 2 }) \mathrm{poly}(1/\eps)$ for bounded width/depth networks\footnote{The polylogs in $\tilde{\Omega}$ might not be necessary in the special case of $\mathrm{Leap}=1$ as indicated by \cite{abbe2022merged}.}. \end{conjecture}

The ``measure-0'' statement in the conjecture means the following. For any set $\{S_1,\ldots,S_m\}$ of nonzero (Fourier or Hermite) basis elements, the conjecture holds for all $h_*$ with $\cS (h_*) = \{S_1 , \ldots , S_m\}$ in the decomposition \eqref{eq:decompo_monomials}, except for a set of coefficients $\{ (\hat h_* (S_i ))_{i\in [m]}  \}\subset \R^m$ of Lebesgue-measure $0$. 
This part of the conjecture is needed for Boolean functions, since it was proved in \cite{abbe2022merged} that a measure-$0$ set of ``degenerate'' leap-1 functions on the hypercube are not learned in $\Theta(d)$ SGD-steps by 2-layer neural networks in the mean-field scaling. However, we further conjecture that in the case of Gaussian data the measure-$0$ modification can be removed if we instead use a rotationally invariant version of the leap. See discussion in Appendix \ref{app:discussion_leap}.

\begin{remark}
We believe that the conjecture (in particular the time complexity scaling) holds for more general architectures than those with isotropically-initialized layers, as long as enough `regularity' assumptions are present at initialization (prohibiting the type of `emulation networks' used in \cite{abbe2020universality}). Note that it is not enough to ask for only the first layer to be initialized with a rotationally-invariant distribution, as this may be handled by using emulation networks on the subsequent layers, but weaker invariances of subsequent layers (e.g., permutation subgroups) may suffice.   
\end{remark}

\paragraph{Formal results.} In order to prove a formal result as close as possible to the general conjecture, we rely on the following specifications: (1) 2-layer NN with smooth activations, (2) layer-wise SGD, (3) projected gradient steps, and (4) a representative subclass of functions on Gaussian isotropic data. We refer to Section \ref{sec:learning_leap_SGD} for the details of the formal results. We also provide in Section \ref{csq} lower-bounds for kernels and CSQ algorithms on regular neural networks. In particular, our results show that SGD-training on fully-connected neural networks achieves the optimal $d^{\Theta(\Leap(h_*))}$ computational complexity of the best CSQ algorithm on this class of sparse functions, going beyond kernels.

The characterization obtained in this paper implies a relatively simple picture for learning low-dimensional functions with SGD on neural networks:

\begin{description}
\item[Picking up the support.] SGD sequentially learns the target function with a saddle-to-saddle dynamic. Specifically, in the first phase, the network learns the support that is reachable by the monomial(s) of lowest degree, and fits these monomials to produce a first descent in the loss. Then, iteratively, each time a new set of coordinates is added to the support, with cardinality bounded by the leap $L$, SGD takes at most $\tilde{\Theta}(d^{\max(L-1,1)})$ steps to identify the new coordinates, before escaping the saddle with another loss descent that fits the new monomials. Thus the dynamic moves from saddle points to saddle points, with plateaus of length corresponding to the leap associated with each saddle. See Figure \ref{fig:StoS_punchy} for an illustration.\footnote{As we will discuss, we only show a saddle-to-saddle behavior when the leaps are of increasing size as the dynamics progress, so it is theoretically open whether it holds in the more general setting where leaps can decrease.} 

\item[Computational time.] Since the total training time is dominated by the time to escape the saddle with the largest leap, our results imply a $\tilde{\Theta} (d^{\Leap(h_*) \vee 2})$ time complexity. This scaling matches the CSQ lower-bounds from Section \ref{csq}, which are also exponential in the leap. Thus SGD on regular neural networks and low latent dimensional data is shown to achieve a time scaling to learn that matches that of optimal CSQ algorithms; see Section \ref{csq} for further discussions.  

\item[Curriculum learning.] SGD on regular neural networks implicitly implements a form of `adaptive curriculum' learning. SGD first picks up low-level features that are computationally and statistically easier to learn, and by picking up these low level features, it makes the learning of higher-level features in turn easier. As mentioned in the examples of Table \ref{tab:comparison}: learning $z_1 \cdots z_{2k}$ takes $\widetilde{\Theta} (d^{2k-1})$ sample complexity (leap-$2k$ function). But if we add an intermediary monomial to our target to create $z_1 \cdots z_k + z_1 \cdots z_{2k}$,  then it takes $\widetilde{\Theta} (d^{k-1}) $ steps to learn (leap-$k$ function). If we have a full staircase, it only requires $\Theta(d)$ (leap-$1$ function). This thus gives an adaptive learning process that follows a curriculum learning procedure where features of increasing complexity guide the learning. 
\end{description}

Finally we note that we considered here the setting of online-SGD, and a natural question is to consider how the picture may change under ERM (several passes with the same batch of samples). The ERM setting is however harder to analyze. We consider this to be an important direction for future works. Note that our results imply a sample complexity equal to the number of SGD steps $n = t = \widetilde{\Theta} (d^{\max(\Leap - 1,1)})$. In ERM, we reuse samples and consequently reduce the sample complexity. We conjecture in fact that $n = \widetilde{\Theta} (d^{\max(\Leap /2, 1)})$ is optimal for ERM. Furthermore, this paper considers the case of low-dimensional functions $P = O_d(1)$, which allows to focus on the dependency on $d$ in the time-complexity of SGD. A natural future direction is to extend these results to larger $P$. See Appendix \ref{app:beyond_sparse} for further discussion.

\begin{table}[t!]
  \begin{center}
    \label{tab:table1}
    \def\arraystretch{1.2}
    \begin{tabular}{|c|c|c|c|} 
      \hline
      $h_* (\bz)=$ &  $ z_1 z_2 \cdots z_{2k} $ & $z_1 z_2 \cdots z_{k} + z_1z_2 \cdots z_{2k}$ & $ z_1 + z_1 z_2  + \ldots + z_1z_2 \cdots z_{2k} $ \\
      \hline
      Kernels & $ \Omega (d^{2k} )$ & $ \Omega (d^{2k} )$ & $ \Omega (d^{2k} )$ \\
      \hline
      SGD on NN & $ \Tilde \Theta (d^{2k-1})$ & $\Tilde \Theta (d^{k -1})$ & $\Theta (d )$  \\
      \hline
    \end{tabular}

        \caption{Sample size $n$ to fit $f_* (\bx) = h_* (\bM\bx)$. SGD-trained neural networks implicitly implement an `adaptive' or 'curriculum' learning scheme, by exploiting lower degree monomials to efficiently learn higher degree monomials.
        }\label{tab:comparison}
  \end{center}

\end{table}

\subsection{Related works}\label{related}

A string of works \citep{allen2019can,li2020learning,daniely2020learning,allen2020backward,suzuki2020benefit,ba2022high,ghorbani2019limitations,telgarsky2022feature} has explored the power of learning with neural networks beyond neural tangent kernels \citep{jacot2018neural}. In particular, much attention has been devoted to learning multi-index models \citep{chen2020towards,abbe2022merged,nichani2022identifying,barak2022hidden,damian2022neural,mousavi2022neural,bietti2022learning,refinetti2021classifying}, i.e., functions that only depend on a small number of (unknown) relevant directions $\E[y | \bx]=h_* ( \< \bu_1 , \bx\> , \ldots , \< \bu_P , \bx \>)$. These functions offer a simple setting where we expect to see a large benefit of non-linear `feature learning' (aligning the weights of the neural networks with the sparse support), compared to fixed-feature methods (kernel methods). The conjectural picture described in our paper offers a unified framework to understand learning multi-index functions with SGD-trained regular neural networks on square loss. For example, \cite{mousavi2022neural} considers learning monotone single-index functions, which is a special case of learning leap-$1$ functions, and shows that they can be learned in $\tilde{\Theta} (d)$ online-SGD steps. \cite{damian2022neural} considers learning a low-rank polynomial on Gaussian data, with null Hermite-$1$ coefficients and full rank Hessian, which implies that the polynomial is a leap-$2$ function. They show that it can be learned in $n = \Theta (d^2)$ samples with one-gradient descent step on the first layer weights, while we conjecture (and show for a subset of those polynomials) that $\tilde{\Theta} (d)$ online-SGD steps is sufficient. \cite{barak2022hidden} considers learning degree-$k$ monomials on the hypercube and shows that $n = d^{O(k)}$ samples are sufficient, using one gradient descent step on the first layer, while we conjecture (and prove in the Gaussian case) a tighter scaling of $\tilde{\Theta} (d^{k-1})$ online-SGD steps\footnote{Note that a kernel method can learn degree-$k$ monomials with $n = \Theta (d^k)$ samples, and this tighter analysis is necessary to obtain a separation here.}. \cite{bietti2022learning} considers a single-index leap-$k$ function on Gaussian data and obtains the tight scaling $\tilde{\Theta} (d^{k-1})$ with a neural network where all first layer weights are equal. An important innovation of our work compared with these previous results is that we show a \textit{sequential learning mechanism} with several learning phases, which prevents the use of single-index models \citep{bietti2022learning} or one gradient-descent step analysis \citep{daniely2020learning,barak2022hidden,damian2022neural}.

 In parallel, several works have studied the dynamics of SGD in simpler non-convex models in high dimensions \citep{ge2015escaping,tan2019online,chen2019gradient,arous2021online}. Our analysis relies on a similar drift plus martingale decomposition of online-SGD as in \cite{tan2019online,arous2021online}. In particular, the leap complexity is related to the \textit{information-exponent} introduced in \cite{arous2021online}. The latter considers a single-index model trained with online-SGD on a non-convex loss and the information exponent captures the scaling of the correlation between the model at a typical initialization and the global solution. \cite{arous2021online} showed that, with information exponent $k$, online-SGD requires $\widetilde{\Theta} (d^{k-1 \vee 1})$ steps to converge, similarly to the scaling presented in this paper. However our analysis and the definition of the leap complexity differ from \cite{arous2021online} in two major ways. First, our model is not a single parameter model, so a much more involved analysis is required for the dynamics. Second, the information exponent is a coefficient that only applies at initialization, while the leap-complexity is a measure of targets that controls the entire learning trajectory (our neural networks visit several saddles during training). 
 
 See Appendix \ref{add_refs} for further references.

\section{Lower bounds on learning leap functions}\label{csq}

Linear methods such as kernel methods suffer exponentially in the degree of the target function, and cannot use the ``hierarchical'' structure to learn faster. This was proved in \cite{abbe2022merged} for the Boolean case, and this work extends the result to the Gaussian case:

\begin{proposition}[Lower bound for linear methods; informal statement of Propositions~\ref{prop:degree-linear-boolean} and \ref{prop:degree-linear-gaussian}]\label{prop:informal-linear-lower-bound}
Let $h_*$ be a degree-$D$ polynomial over the Boolean hypercube (resp., Gaussian measure). Then there are $c_{h_*}, \eps_{h_*} > 0$, such that any linear method needs $c_{h_*} d^D$ samples to learn $f_*(\bx) = h_*(\bM \bx)$ to less than $\eps_{h_*} > 0$ error, where $\bM$ is an unknown permutation (resp., rotation) as in \eqref{eq:support_bool_gauss}.
\end{proposition}

Consider now the Correlational Statistical Query (CSQ) model of computation \citep{bendavid1995learning,bshouty2002using}. A CSQ algorithm accesses the data via expectation queries, plus additive noise. We show that for CSQ methods the query complexity scales exponentially in the \textit{leap} of the target function, which can be much less than the degree. %

\begin{proposition}[Lower bound for CSQ methods; informal statement of Propositions~\ref{prop:leap-csq-boolean} and \ref{prop:isoleap-csq-gaussian}]\label{prop:informal_CSQ}
In the setting of Proposition~\ref{prop:informal-linear-lower-bound}, the CSQ complexity of learning $f_*$ to less than $\eps_{h_*}$ error is at least $c_{h_*} d^{\Leap(h_*)}$ in the Boolean case, and at least $c_{h_*} d^{\Leap(h_*) / 2}$ in the Gaussian case.
\end{proposition}

The scaling $d^{\Leap (h_*)}$ for Boolean functions in Proposition \ref{prop:informal_CSQ} matches the total time complexity scaling in Conjecture \ref{conj:leap}. For Gaussian functions, we only prove $d^{\Leap (h_*)/2}$ scaling. However we conjecture that the same scaling as the Boolean case should hold.
See Appendix \ref{app:lower_bounds} for details.

\begin{remark}
We note that the above lower-bounds are for CSQ models or noisy population-GD models, and not for online-SGD since the latter takes a single sample per time step. Our proof does show a correspondence between online-SGD and population-GD, but without the additional noise. It is however intriguing that the regularity in the network model for online-SGD appears to act comparatively in terms of constraints to a noisy population-GD model (on possibly non-regular architectures), and we leave potential investigations of such  correspondences to future work (see also discussion in Appendix \ref{app:beyond_CSQ}). Further, we note that the correspondence to CSQ  may not hold beyond the finite $P$ regime. First there is the `extremal case' of learning the full parity function, which is efficiently learnable in CSQ (with 0 queries) but not necessarily with online-SGD on regular networks: \cite{abbe2022non} shows it is efficiently learnable by a i.i.d.\ Rademacher$(1/2)$ initialization, but not necessarily by a Gaussian isotropic initialization. Further, the positive result of the Rademacher initialization may disappear under proper hyperparameter `stability' assumptions. Beyond this extremal case, a more important nuance arises for large $P$: the fitting of the function on the support may become costly for regular neural networks in certain cases. For example, let $g : [P] \to \{0,1\}$ be a known function and consider learning $f_*(\bx)$ which depends on $P$ unknown coordinates as $h_*(\bz) = \sum_{i =1 }^P i z_i + \prod_{i = 1}^P z_i^{g(i)}$. This is a leap-1 function where the linear part reveals the support and the permutation, and with a parity term on the indices such that $g(i) = 1$. In this case, SGD on a regular network would first pick up 
the support, and then have to express a potentially large degree monomial on that support, which may be hard if $P$ is large (i.e., $P \gg 1$). The latter part may be non trivial for SGD on a regular network, while, since $g$ is known, it would require 0 queries for a CSQ algorithm once the permutation was determined from learning the linear coefficients.    
\end{remark}

\section{Learning leap functions with SGD on neural networks}\label{sec:learning_leap_SGD}

Let $h_* : \R^P \to \R$ be a degree-$D$ polynomial. We consider learning $f_* (\bx) = h_* (\bz)$ on isotropic Gaussian data, where $\bz = \bM \bx$ is the covariate projected on a $P$-dimensional latent subspace, using online-SGD on a two-layer neural network. At each step $t$, we get a new fresh (independent) sample $(\bx^t, y^t)$ where $\bx^t \sim \normal (0,\id_d)$ and $y^t = h_*(\bz^t) + \eps^t$, with additive noise $\eps^t$ independent and $K$-sub-Gaussian. For the purpose of our analysis, we assume that $\bz$ is a subset of $P$ coordinates of $\bx$ instead of a general subspace. This limitation of our analysis is because of an entrywise projection step that we perform during training for technical reasons, and which makes the training algorithm non-rotationally equivariant (see description of the algorithm below). Since we assume that $\bz$ is a subset of the coordinates, without loss of generality we choose $\bz$ to be the first $P$ coordinates of $\bx$.

\subsection{Algorithm}

We use a 2-layer neural network with $M$ neurons and weights $\bTheta = (a_j,b_j,\bw_j)_{j \in [M]} \in \R^{M (d+2)}$:
\begin{equation}\label{eq:2-NN}
\hf_{\NN} (\bx ; \bTheta) =  \sum_{j \in [M]}  a_j \sigma ( \< \bw_j , \bx \> + b_j) \, ,
\end{equation}
    We consider the following assumption on the activation function:\footnote{This is satisfied, for example, by the shifted sigmoid $\sigma(z) = 1 / (1 + e^{-z+c})$ for almost all shifts $c$.}
\begin{assumption}%
\label{ass:sigma_I} Let $\sigma : \R \to \R$ be an activation function that satisfies the following conditions. There exists a constant $K>0$ such that $\sigma$ is $(D+3)$-differentiable with $\| \sigma^{(k)} \|_{\infty} \leq K$ for $k = 0 , \ldots , (D+3)$ and $| \mu_k (\sigma) | > 1/K$ for $k = 0 , \ldots , D$, where $\mu_k (\sigma) = \E_G [ \He_k (G) \sigma(G) ] = \E_G [ \sigma^{(k)} (G) ] $ is the $k$-th Hermite coefficient of $\sigma$ and $G \sim \normal (0,1)$.
\end{assumption}

We train $\hf_{\NN}$ using online-SGD on the squared loss $\ell ( y, \hat y) =  \frac{1}{2} (y - \hat{y})^2$, with the goal of minimizing the population risk:
\begin{equation}\label{eq:test_error}
R ( \bTheta ) = \E_{\bx} \Big[ \ell \big( f_* (\bx) , \hf_{\NN} ( \bx ; \bTheta ) \big) \Big] \, .
\end{equation}

For the purposes of the analysis, we make two modifications to SGD training. First, we train layerwise: training $\{\bw_j\}_{j \in [M]}$ and then $\{a_j\}_{j \in [M]}$, while keeping the biases $\{ b_j\}_{j \in [M]}$ frozen during the whole training. Second, during the training of the first layer weights $\{\bw_j\}_{j \in [M]}$, we project the weights in order to ensure that they remain bounded in magnitude. See Algorithm~\ref{alg:sgd-training} for pseudocode, and see below for a detailed explanation. These modifications are not needed in practice for SGD to learn, as we demonstrate in our experiments in Figure~\ref{fig:StoS_punchy} and Appendix~\ref{app:experiments}.

\RestyleAlgo{ruled}%
\SetAlgoVlined%
\LinesNumbered%
\newcommand\mycommfont[1]{\small\ttfamily\textcolor{blue}{#1}}
\SetCommentSty{mycommfont}
\begin{algorithm}
\caption{Layerwise online-SGD with init scales $\kappa,\rho > 0$, learning rates $\eta_1,\eta_2 > 0$, step counts $\oT_1,\oT_2 > 0$, second-layer ridge-regularization $\lambda_a > 0$, and projection params $r,\Delta > 0$}\label{alg:sgd-training}

$a^0_j \sim \Unif (\{\pm \kappa \})$, $b^0_j \sim \Unif ([-\rho,+\rho ])$, $\bw_j^0 \sim \Unif ( \{ \pm 1/\sqrt{d} \}^d )$ $\;\;$ \tcp*[h]{Initialization}

\For($\quad$\tcp*[h]{Train first layer with projected SGD}){$t = 0$ {\bfseries to} $\oT_1-1$, \mbox{\bfseries and all } $j \in [M]$}{
$\tbw_j^{t+1} = \bw_j^t - \eta_1 \cdot \grad_{\bw_j^t} \ell \big( y^t , \hf_{\NN} (\bx^t ; \bTheta^t ) \big)$, where $\grad$ is spherical gradient in \eqref{eq:first-layer-sphere-grad}

$\bw_j^{t+1} = $ projection of $\tbw_j^{t+1}$ defined in  \eqref{eq:first-layer-projection}

$a_j^{t+1} = a_j^t$, $b_j^{t+1} = b_j^t$
}

\For(\quad\tcp*[h]{Train second layer with SGD}){$t = \oT_1$ {\bfseries to} $\oT_1+\oT_2-1$, \mbox{\bfseries and all } $j \in [M]$}{
$\bw_j^{t+1} = \bw_j^t$, $b_j^{t+1} = b_j^t$

$a_j^{t+1} = (1-\lambda_a)a_j^t - \eta_2 \cdot \frac{\partial}{\partial a_j^t} \ell \big( y^t , \hf_{\NN} (\bx^t ; \bTheta^t ) \big)$
}
\end{algorithm}

Analyzing layerwise training is a fairly standard tool in the theoretical literature to obtain rigorous analyses; it is used in a number of works, including \cite{daniely2020learning,barak2022hidden,damian2022neural,abbe2022merged}. In our setting, layerwise training allows us to analyze the complicated dynamics of neural network training, but it also leads to a major issue. During the training of the first layer, the target function $f_*$ is not fully fitted because we do not train the second layer concurrently. Therefore the first-layer weights continue to evolve even after they pick up the support of $f_*$. This is a challenge since we must train the first-layer weights for a large number of steps, and so they can potentially grow to a very large magnitude, leading to instability.\footnote{We emphasize that this problem is due to layerwise training, since in practice if we train both layers at the same time the residual quickly goes to zero after the support is picked up, and so the first-layer weights stop evolving and remain bounded in magnitude (see Appendix~\ref{app:experiments}).}

We correct the issue by projecting each neuron's first-layer weights $\bw_j$ to ensure that the coordinates do not blow up. First, we keep the ``small'' coordinates of $\bw_j$ on the unit sphere, i.e., for some parameter $r > 0$, we define the ``small'' coordinates for neuron $j$ at time $t$ by $S_{j,0} = [d]$ and 
$$S_{j,t} = \{i \in [d] : |\tilde{w}_{j,i}^{t'}| < r \mbox{ for all } 1 \leq t' \leq t\}.$$ We project these coordinates on the unit sphere using the operator $\proj^{t}_{j}$ defined by
\begin{equation}\label{eq:projection_on_St}
\big( \proj^{t}_{j} \bw_j^t  \big)_i =  w_{j,i}^t \quad \text{if $i \not\in S_{j,t}$}; \qquad \big( \proj^{t}_{j} \bw_j^t  \big)_i = \frac{w_{j,i}^t}{\| \cS_t (\bw_j^t) \|_2 } \quad \text{if $i \in S_{j,t}$},
\end{equation}
and use the spherical gradient with respect to the sphere $\| \cS_t (\bw_j^t) \|_2 = 1$, i.e., for any function $f$,
\begin{align}\label{eq:first-layer-sphere-grad}
\grad_{\bw_j^t} f(\bw_j^t) = \nabla_{\bw_j^t} f(\bw_j^t) - \cS_t (\bw_j^t) \< \cS_t (\bw_j^t) , \nabla_{\bw_j^t} f(\bw_j^t) \> \, .
\end{align}
In order to keep the ``large'' coordinates $i \not\in S_{j,t}$ from growing too large, we project them onto the $\ell_\infty$ ball of radius $\Delta$ for some $\Delta > r$, and denote this projection by $\proj_{\infty}$. In summary, the projection performed in the training of the first layer can be written compactly as
\begin{align}\label{eq:first-layer-projection}
\bw_j^{t+1} =   \proj^{t+1}_{j} \proj_{\infty} \tbw_j^{t+1}\,.
\end{align}

In the second phase, the training of the second layer weights $\ba$ is by standard SGD (without projection) with added ridge-regularization term $\frac{\lambda_a}{2} \|\ba\|^2$ to encourage low-norm solutions.

\subsection{Learning a single monomial}

We first consider the case of learning a single monomial with Hermite exponents $k_1, \ldots, k_P \geq 1$:
\[
h_* (\bz) = \He_{k_1} (z_1) \He_{k_2} (z_2) \cdots \He_{k_P} (z_P)\, .
\]
We assume $D = k_1 + \ldots + k_P \geq 2$ (the case $D=1$ is straightforward). $h_*$ is a leap-$D$ function. We start by proving that, during the first phase, the first layer weights grow in the directions of $z_1,\ldots,z_P$ which are the variables in the support of the target function.

\begin{theorem}[First layer training, single monomial, sum of monomials]\label{thm:alignment_one_monomial}
Assume $\sigma$ satisfy Assumption \ref{ass:sigma_I}. 
Then for $0 < r < \Delta$ sufficiently small (depending on $D,K$) and $\rho \leq \Delta$ the following holds. For any constant $C_*>0$, there exist $C_i$ for $i= 0, \ldots , 6$, that only depend on $D,K$ and $C_*$ such that
\[
\begin{aligned}
\qquad \oT_1 = C_0 d^{D-1} \log(d)^{C_0}\, , \qquad \eta_1 = \frac{1}{C_1 \kappa d^{D/2} \log(d)^{C_1} }\, , \qquad \kappa \leq \frac{1}{C_2 d^{C_2}}\,, \\
\end{aligned}
\]
and for $d$ large enough that $r \geq C_0 \log(d)^{C_0} / \sqrt{d}$, the following event holds with probability at least $1 - Md^{-C_*}$. For any neuron $j \in [M]$,
\begin{itemize}
    \item[(a)] Early stopping: $| w_{j,i}^{t} - w_{j,i}^0| \leq C_3 /\sqrt{d \log (d)}$ for all $i \in [d]$ and $t \leq \oT_1/ (C_4 \log (d)^{C_4})$.
\end{itemize}

And for any neuron $j \in [M]$ such that $a_j^0 \mu_D (\sigma) (w_{j,1}^0)^{k_1} \cdots (w_{j,P}^0)^{k_P} > 0$, 
\begin{itemize}
    \item[(b)] On the support: $\big\vert  w_{j,i}^{\oT_1}  - \sign (w_{j,i}^0) \cdot \Delta \big\vert \leq C_5 /\sqrt{d \log (d)}$ for $i = 1, \ldots , P$.

    \item[(c)] Outside the support: $| w_{j,i}^{\oT_1} - w_{j,i}^0| \leq C_6 r^2 /\sqrt{d}$ for $i = P+1, \ldots , d$, and $\sum_{ i >P} (w_{j,i}^{\oT_1})^2 = 1$.
\end{itemize}
\end{theorem}

Theorem \ref{thm:alignment_one_monomial} shows that after the end of the first phase, the coordinates $\bw_j^{\oT_1}$ aligned with the support $\bz$ are all close to $\pm \Delta$ with the same signs as $w_{j,1}^0 , \ldots , w_{j,P}^0$ as long as $(w_{j,1}^0)^{k_1} \cdots (w_{j,P}^0)^{k_P} > 0$ has the same sign as $a_j^0 \mu_D (\sigma) $ at initialization. Furthermore, the correlation with the support only appears at the end of the dynamics, and does not appear if we stop early. 

The proof of Theorem \ref{thm:alignment_one_monomial} follows a similar proof strategy as \cite{arous2021online}, namely a decomposition of the dynamics into a drift and martingale terms with  information exponent $D$. However, our problem is multi-index, and the analysis will require a tighter control of the different contributions to the dynamics as the dynamics move from saddle to saddle. An heuristic explanation for this result can be found in Appendix \ref{app:intuition_monomial}. The complete proof of Theorem \ref{thm:alignment_one_monomial} is deferred to Appendix \ref{app:proof_alignment_one_monomial}.

The second layer weights training amounts to studying SGD on a linear model and is standard. The typical strategy consists in showing that the target function can be fitted with low-norm second-layer weights $\| \ba_*\|_2$ (see for example \cite{daniely2020learning,barak2022hidden,damian2022neural,abbe2022merged}). Because of the way we prove alignment of the first layer weights (weights $\pm \Delta$ on the support coordinates), we only prove this fitting for two specific monomials\footnote{For general monomials, we would require more diversity on the first layer weights (for example, adding randomness on the $\ell_\infty$ projection such that the weights are $\pm \beta \Delta$ with $\beta \sim \Unif([1/2,3/2])$). Again, these caveats are due to our proof technique to show alignment of the first layer weights.}.

\begin{corollary}[Second layer training, single monomial]\label{cor:learning_one_monomial}
Let $h_* (\bz) = z_1 \cdots z_D$ or $h_* (\bz) = \He_D (z_1)$ and assume $\sigma$ satisfies Assumption \ref{ass:sigma_I}. For any constants $C_*>0$ and $\eps>0$, there exist $C_i$ for $i= 0, \ldots , 11$, that only depend on $D, K$ and $C_*$ such that taking width $M = C_0 \eps^{-C_0}$,  bias initialization scale $\rho =\eps^{C_1} / C_{1}$, and $\Delta = \eps^{C_1} / C_1$ and second-layer initialization scale $\kappa = \frac{1}{C_2 M d^{C_2}}$, and second-layer regularization $\lambda_a = M\eps / C_3$, and , and $r = \eps^{C_{4}} / C_{4} $, and
\[
\begin{aligned}
\oT_1 =&~ C_5 d^{D-1} \log(d)^{C_5}\, , \qquad &&\eta_1 = \frac{1}{C_6 \kappa d^{D/2}\log(d)^{C_6}}\, , \\
\oT_2 =&~ C_7 \eps^{-C_7} \, , \qquad &&\eta_2 = 1/(C_8 M \eps^{-C_8} )\, ,
\end{aligned}
\]
for $d \geq C_9 \eps^{-C_9}$ we have with probability at least $1 - d^{-C_*} - \eps$:
\begin{itemize}
    \item[(a)] At the end of the dynamics,
    \[
    R ( \bTheta^{\oT_1 + \oT_2} ) \leq \eps \, .
    \]

    \item[(b)] If we train the first layer weights for $\oT_1 ' \leq \oT_1 / (C_{10} \log(d)^{C_{10}}) $ steps and for $M \leq C_{10} \log(d)$, then we cannot fit $f_*$ using the second-layer weights, i.e., 
    \[
    \min_{\ba \in \R^M} \E_{\bx } \Big[ \Big( f_*(\bx) - \sum_{j \in [M]} a_j \sigma ( \< \bw_j^{\oT_1'} , \bx \>) \Big)^2 \Big] \geq 1 - \frac{\log(d)}{d^D}\, .
    \]
\end{itemize}
\end{corollary}

This result suggests that the dynamics of SGD with one monomial can be decomposed into a `search phase' (plateau in the learning curve) and a `fitting phase' (rapid decrease of the loss) similarly to \cite{arous2021online}. SGD progressively aligns the first layer weights with the support, with little progress, and as soon as SGD picks up the support, the second layer weights can drive the risk quickly to $0$. Because of the layer-wise training, we only show in Corollary \ref{cor:learning_one_monomial}.(b) that with early stopping on the training of the first layer weights, we cannot approximate the function $f_*$ at all using the second layer weights (hence, we cannot learn it even with infinite number of samples). The proof of Corollary \ref{cor:learning_one_monomial} is in Appendix \ref{app:second_layer_monomial}.

\subsection{Learning multiple monomials}\label{sec:multimonomials}

We now consider $h_*$ with several monomials in its decomposition. In order to simplify the statement and the proofs, we will specifically consider the case of nested monomials
\begin{equation}\label{eq:nested_monomials}
\begin{aligned}
h_* (\bz) =&~ \sum_{l = 1}^L \prod_{s \in [P_l]} \He_{k_s} (z_s) \, , 
\end{aligned}
\end{equation}
where $0 =:P_0 < P_1 < P_2 < \ldots <P_L =: P$ and $k_1 , \ldots , k_{P}$ are positive integers.
For $l \in [L]$, we denote $D_l = k_{P_{l-1}+1} + \ldots + k_{P_{l}}$, and $D = \max_{l \in [L]} D_l$ the size of the biggest leap (such that $h_*$ is a leap-$D$ function), $\oD_l = D_1+ \ldots + D_L$ and $\oD := D_L$ the total degree of the polynomial $h_*$. We will assume that $\min_{l \in [L]} D_l \geq 2$ (i.e., leap of size at least $2$ between monomials). This specific choice for $h_*$ allows for a more compact proof, similar to Theorem \ref{thm:alignment_one_monomial}. However, the compositionality of $h_*$ is not a required structure for the sequential alignment to hold and we describe in Appendix \ref{app:extending_analysis} how to modify the analysis for more general\footnote{However, our current proof techniques do not allow for fully general leap functions: e.g., $h_* (\bz) = \He_2(z_1) \He_3(z_2) - \He_3(z_1) \He_2 (z_2)$ has its two monomials pushing the $\bw_j$'s in two opposite directions.} $h_*$.

We first prove that the first-layer weights grow in the relevant directions during training.

\begin{theorem}[First layer training]\label{thm:sequential_alignment} Let $h_* : \R^P \to \R$ be defined as in Eq.~\eqref{eq:nested_monomials} and assume $\sigma$ satisfy Assumption \ref{ass:sigma_I}. Then with the same choice of hyperparameters as in Theorem \ref{thm:alignment_one_monomial}, with $D$ now corresponding to the biggest leap, we have with probability at least $1 - Md^{-C_*}$: for any neuron $j \in [M]$ that satisfies $a^0 \mu_{\oD_j } (\sigma) (w_{j,1}^0)^{k_1} \cdots (w_{j, P_l}^0)^{k_{P_l}} > 0$ for all $l \in [L]$,
\begin{itemize}
    \item[(a)] On the support: $\big\vert  w_{j,i}^{\oT_1}  - \sign (w_{j,i}^0) \cdot \Delta \big\vert \leq C_5 /\sqrt{d \log (d)}$ for $i = 1, \ldots , P$.

    \item[(b)] Outside the support: $| w_{j,i}^{\oT_1} - w_{j,i}^0| \leq C_6 r^2 /\sqrt{d}$ for $i = P+1, \ldots , d$ and $\sum_{i >P} (w_{j,i}^{\oT_1} )^2 = 1$.
\end{itemize}
\end{theorem}

The proof follows by showing the sequential alignment of the weights to the support: with high probability and for each neurons satisfying the sign condition at initialization, it takes between $d^{\frac{D_l +D}{2}- 1}/(C \log(d)^{C})$ and $ d^{\frac{D_l +D}{2}- 1} C \log(d)^{C}$ steps to align with coordinates $[P_l]$, after having picked up coordinates $[P_{l-1}]$. The proof can be found in Appendix \ref{sec:proof_alignment_several_monomials}. 

While Theorem \ref{thm:sequential_alignment} captures the tight scaling in overall number of steps, it does not capture the number of steps for smaller leaps $D_l < D$ shown in Figure \ref{fig:StoS_punchy} in the case of increasing leaps. In Appendix \ref{sec:adaptive_step_size}, we show that the scaling of $d^{D_l -1}$ steps to align to the next monomial can be obtained by varying the step size, in the case of increasing leaps. Note that in practice, neural networks with constant step size seem to achieve this optimal scaling for escaping each saddle (such as in Figure \ref{fig:StoS_punchy}). Hence, there might be a mechanism in the SGD training that can implicitly control the martingale part of the dynamics, without rescaling the step sizes. However, understanding such a mechanism would require to study the joint training of both layers, which is currently out of reach of our proof techniques.

As in the single monomial case, we consider fitting the second layer weights only for a specific class of functions (where all monomials are multilinear):
\begin{equation}\label{eq:form_fct_fitting_a}
h_* (\bz) = z_1 \cdots z_{P_1} + z_1 \cdots z_{P_2} + \ldots + z_1 \cdots z_{P_L} \, .
\end{equation}
 We require extra assumptions on the activation function to prove that the fitting is possible. The following is an informal statement, and we leave the formal statement and proof to Appendix~\ref{sec:proof_seq_learning}.
 
\begin{corollary}[Second layer training, sum of monomials; informal statement]\label{cor:sequential_learning}
Let $h_*(\bz)$ be as in \eqref{eq:form_fct_fitting_a}. Then there is an activation function $\sigma$ satisfying Assumption \ref{ass:sigma_I} such that for any $\eps> 0$ there are choices of hyperparameters for SGD on a $\poly(1/\epsilon)$-width two-layer network (Algorithm~\ref{alg:sgd-training}) such that for step counts $\oT_1 = \tilde{\Theta}(d^{\Leap(h_*) -1 })$ and $\oT_2 = \poly(1/\epsilon)$ we have with high probability $R ( \bTheta^{\oT_1 + \oT_2} ) \leq \eps$.
\end{corollary}

\section{Discussion}
\paragraph{Summary of contributions:} In this work, we have considered learning multi-index functions over the hypercube or the Gaussian measure. For classical linear methods, the complexity of the task is determined by the \textit{degree} of the target function (Proposition~\ref{prop:informal-linear-lower-bound}). However, for neural networks, we have conjectured that the complexity is determined by the \textit{leap complexity} of the target function (introduced in Definition~\ref{def:leap_general}). This would generalize the result of \cite{abbe2022merged}, which shows the conjecture in the case of $\Leap(h_*) = 1$. As evidence for this conjecture, we have proved lower-bounds showing $d^{\Omega(\Leap(h_*))}$ complexity of learning in the Correlational Statistical Query (CSQ) framework 
(Proposition~\ref{prop:informal_CSQ}). Conversely, we have proved that $d^{O(\Leap(h_*))}$ samples and runtime suffices for a modified version of SGD to successfully learn the relevant indices in the case of ``leap-staircase'' functions of the form \eqref{eq:nested_monomials}, and to fit the function in the case of ``multilinear leap-staircase'' functions of the form \eqref{eq:form_fct_fitting_a}.

\paragraph{Future work:} One direction for future work is to remove the modifications to vanilla SGD used in the analysis (layerwise training and the projection step). Another direction is to prove the conjecture by extending our analysis of the training dynamics to general functions, beyond those of the form \eqref{eq:nested_monomials}. Another direction is to study extensions of the leap complexity measure beyond isotropic input distributions.

\section*{Acknowledgement} Part of this work was supported by the NSF-Simons Research Collaborations on the Mathematical and Scientific Foundations of Deep Learning  (MoDL) Award and the EPFL PhD Exchange Fellowship. EB was also generously supported by
Apple with an AI/ML fellowship. TM also acknowledges the NSF grant CCF-2006489 and the ONR grant N00014-18-1-2729.

\bibliographystyle{amsalpha}
\bibliography{bibliography.bib}

\clearpage

\appendix

\section{Additional numerical simulations}\label{app:experiments}

In Figures~\ref{fig:aesthetic_tanh_stair}, \ref{fig:aesthetic}, \ref{fig:aesthetic_hermite_sum} and \ref{fig:bool_stair148_resnet} we plot the risk versus number of samples for SGD training of 5-layer ResNets with fully-connected layers for various different target functions and for Boolean and Gaussian data. In these plots, the saddle-to-saddle dynamics are visible, which are caused by the neural network sequentially picking up the support using the hierarchical structure of the monomials in the function. In Figures~\ref{fig:bool_stair3} and \ref{fig:gauss_stair3}, we study learning a leap-1 function (merged-staircase function), and we experiment with the effect of adding depth to see its effect on fitting. There is also an interesting edge-of-stability behavior during the ``second-layer fitting'' part, where the loss does not decrease monotonically \citep{cohen2021gradient}. We leave understanding this to future work.

\begin{figure}[h!]
    \centering
    \begin{tabular}{c}
        \includegraphics[scale=0.5]{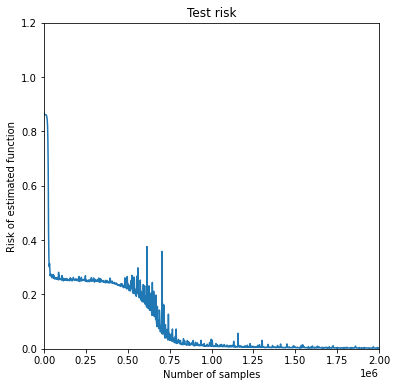} \includegraphics[scale=0.5]{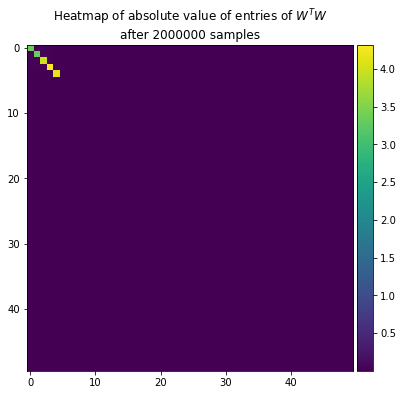}
    \end{tabular}

    \caption{In this figure we consider training a 5-layer ResNet with fully-connected layers with SGD the leap-3 function $h_*(\bz) = 2\cdot\prod_{i=1}^2 \tanh(z_i) + 5\cdot\prod_{i=1}^5 \tanh(z_i)$ with data $\bx \sim \normal(0,I_d)$ and $d = 50$. While our paper considered bounded degree polynomials, the leap complexity, which drives the sequential alignment to the support, also holds for non-polynomial functions. In this case, the leap depends on the first non-zero monomials in the Hermite decomposition. For $h_*$ considered in this plot, we have first a leap of size 2 to align with $x_1,x_2$ followed by a leap of size 3 to align with $x_3x_4x_5$.   In the plot of test risk over time, we indeed see first a short saddle to align with $x_1 , x_2$, followed by a quick decrease of the loss (corresponding to the neural networks fitting $2\tanh(z_1)\tanh(z_2)$). This is followed by a plateau while SGD slowly picks up $x_3,x_4,x_5$ (saddle) and a sharp decrease in the loss when the neural network fit the remainder of $h_*$. 
    We also plot the heatmap of the absolute value of the entries of $\bW^{\top} \bW \in \R^{d \times d}$ where $\bW$ is the first-layer matrix after training. This shows that the first layer indeed picks up the relevant coordinates (first $5$ coordinates) in the support after training.}
    \label{fig:aesthetic_tanh_stair}
\end{figure}

\begin{figure}[h!]
    \centering
    \begin{tabular}{cc}
      \includegraphics[scale=0.5]{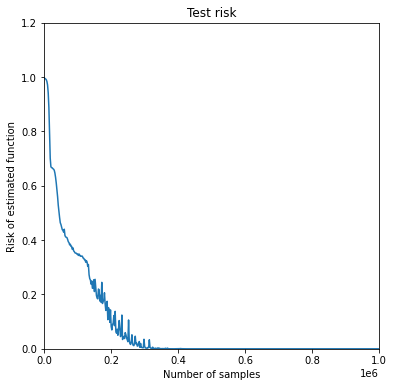}  & 
        \includegraphics[scale=0.5]{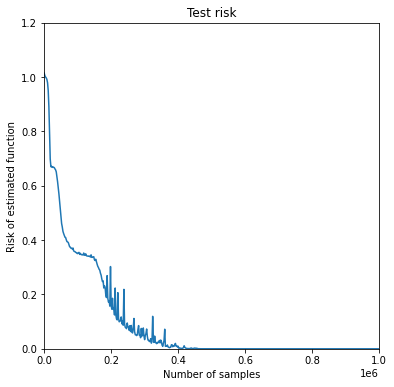}    \\
        (a) Ambient dimension $d = 50$& %
        (b) Ambient dimension $d = 100$ \\
        \includegraphics[scale=0.5]{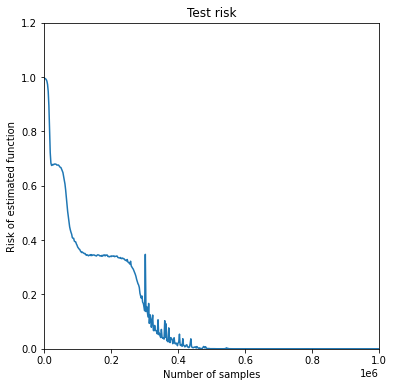} & \includegraphics[scale=0.5]{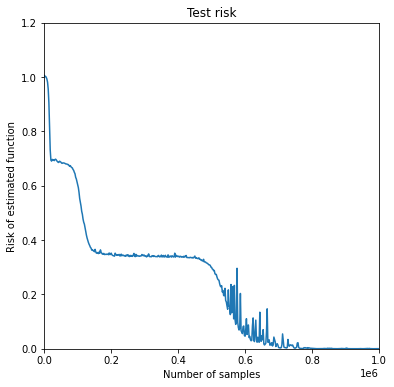} \\
        (c) Ambient dimension $d = 200$ & Ambient dimension $d = 400$
    \end{tabular}

    \caption{In (a)-(d) we show the evolution of the risk for training a 5-layer ResNet with fully-connected layers with SGD to learn the leap-3 function  $h_*(\bz) = z_1 + z_1z_2z_3 + z_1z_2z_3z_4z_5z_6$ with binary hypercube data in ambient dimension $d = 50, 100, 200, 400$, respectively. Notice that the evolution of the risk follows a saddle-to-saddle dynamic. This dynamic becomes more salient as the ambient dimension increases and escaping the saddles dominates the SGD trajectory.}
    \label{fig:aesthetic}
\end{figure}

\begin{figure}[h!]
    \centering
    \begin{tabular}{cc}
        \includegraphics[scale=0.5]{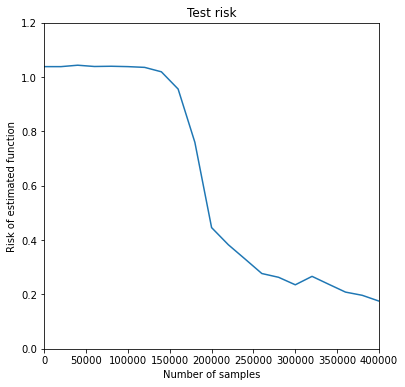} & \includegraphics[scale=0.5]{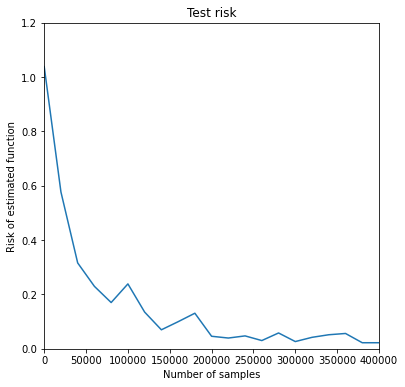} \\
        (a) Leap-3 function $h_*(\bz) = \He_3(z_1)$ & (b) Leap-1 function $h_*(\bz) = \He_1(z_1) + \He_3(z_1)$
    \end{tabular}

    \caption{We consider training a 5-layer ResNet with fully-connected layers with SGD on covariate distribution $\bx \sim \normal(0,I_d)$ with $d = 500$. In (a) we show the risk from learning the leap-3 function $h_*(\bz) = \He_3(z_1)$, and in (b) we show the risk from learning the leap-1 function $h_*(\bz) = \He_1(z_1) + \He_3(z_1)$. Notice that the leap-3 task is much more difficult for SGD, and it gets stuck in a saddle where the loss plateaus. On the other hand, the $\He_1(z_1)$ term in the leap-1 task means that SGD is not stuck in a saddle.}
    \label{fig:aesthetic_hermite_sum}
\end{figure}

\begin{figure}[h!]
    \centering
    \includegraphics[scale=0.5]{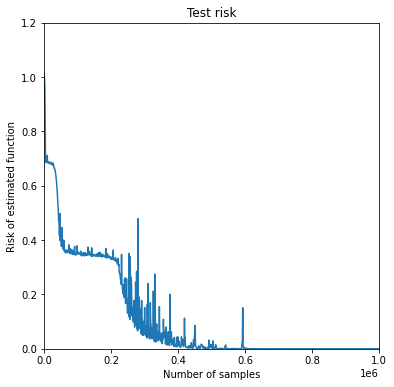}
    \includegraphics[scale=0.5]{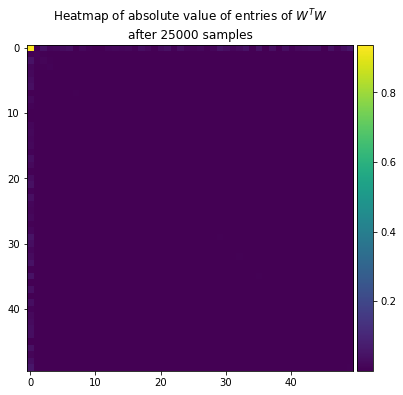}
    \includegraphics[scale=0.5]{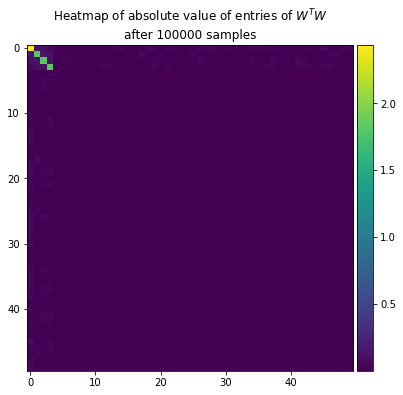}
    \includegraphics[scale=0.5]{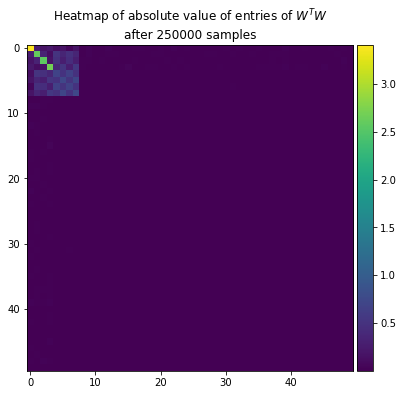}
    \caption{A width-$1000$ 5-layer ResNet network with ReLU activation trained with one-pass SGD with mini-batch size $100$ and step size $0.1$. The data is $\bx \sim \{+1,-1\}^{d}$ for ambient dimension $d = 50$, and $h_*(z) = z_1 + z_1z_2z_3z_4 + z_1z_2z_3z_4z_5z_6z_7z_8$, which is a leap-4 function. We observe saddle-to-saddle dynamics. And we observe that the first layer picks up the relevant support iteratively.}
    \label{fig:bool_stair148_resnet}
\end{figure}

\begin{figure}[h!]
    \centering
    \includegraphics[scale=0.5]{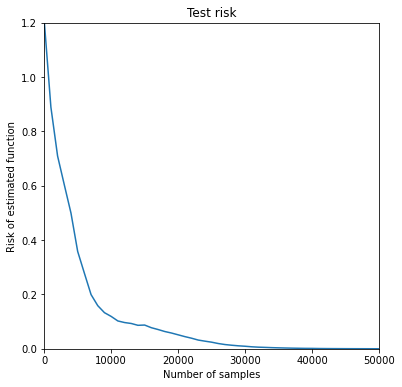}
    \includegraphics[scale=0.5]{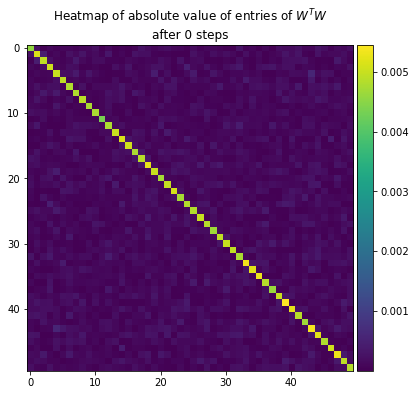}
    \includegraphics[scale=0.5]{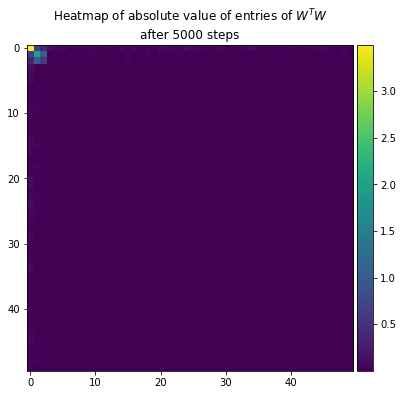}
    \includegraphics[scale=0.5]{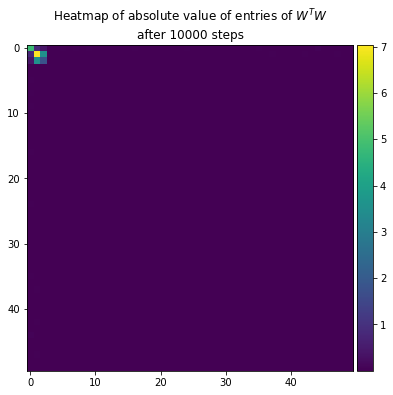}
    \caption{We train a width-$1000$, 2-layer network with sigmoid activation with mini-batch size $100$ and learning rate $0.5$. The data is from the Boolean hypercube with ambient dimension $d = 50$, and the target function is $h_*(z) = z_1 + z_1z_2 + z_1z_2z_3$, which is a leap-1 function. Notice that the weights quickly align to the support of the function (no saddles) after less than 5000 steps.}
    \label{fig:bool_stair3}
\end{figure}

\begin{figure}[h!]
    \centering
    \begin{tabular}{cc}
    \includegraphics[scale=0.5]{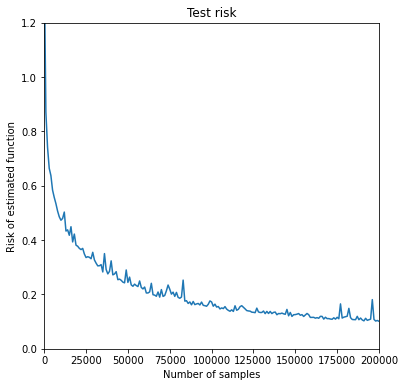} &
    \includegraphics[scale=0.5]{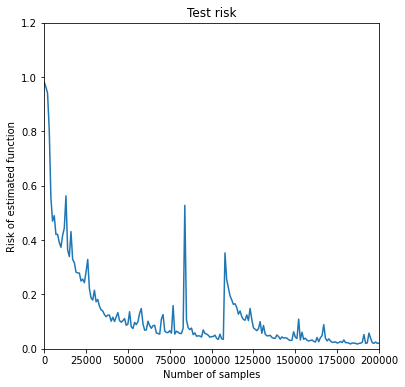} \\
    (a) 2-layer network & (b) depth-5 ResNet
    \end{tabular}
    \caption{We consider either (a) training a width-$1000$, 2-layer network with sigmoid activation or (b) training a width-$1000$, 5-layer ResNet network with ReLU activation and fully-connected layers. Our ambient dimension is $d = 50$, our data is $\bx \sim \normal(0,I_d)$ and our target function is $h_*(z) = z_1 + z_1z_2 + z_1z_2z_3$. This is a leap-1 function, so the weights quickly align to the coordinates $x_1,x_2,x_3$ after a small number of steps. However, the two-layer neural network struggles to fit the different monomials in $h_*$. This can be mitigated by training a deeper network which finds a better fit faster. Hence, besides the alignment phenomenon to the low-dimensional support explored in this paper, it is an interesting question for future work to understand why depth helps in this situation. Note that this is a different phenomenon than the one explored in depth-separation papers such as \cite{safran2022optimization}, which considers learning functions which cannot be efficiently approximated by 2-layer neural networks (here, $h_*$ can be easily approximated with a two-layer network).}
    \label{fig:gauss_stair3}
\end{figure}

\clearpage

\section{Additional discussion from the main text}
\subsection{Additional references}\label{add_refs}

In addition to the references listed in the main text, we further review other relevant papers.

A line of work in computational learning theory studied the complexity of learning Boolean functions under the uniform input distribution. It was realized that functions with concentrated Fourier spectrum can be learned efficiently, both in sample and time complexity using the sparse Fourier algorithm \citep{mansour1994}. Namely, under knowledge of a set of basis elements $\mathcal{S}$ such that $\sum_{S \in \mathcal{S}} f^2(S) \ge 1-\epsilon/2$ for all $f \in \mathcal{F}$, one can learn $\mathcal{F}$ with error $\epsilon$, sample complexity $O((1/\epsilon)|\mathcal{S}|\log (|\mathcal{S}|/\delta))$ and polynomial time complexity if $|\mathcal{S}|$ is polynomial using the sparse Fourier algorithm that estimates the coefficients in $\mathcal{S}$. Many interesting classes of functions fall under this setting, such as juntas, low degree functions, bounded-size or -depth decision trees \citep{odonnell_2014}.  While $\mathcal{S}$ has to be known\footnote{The set knowledge can be relaxed under the query access model using the Kushilevitz-Mansour algorithm (based on the Goldreich-Levin algorithm) that uses a divide-and-conquer procedure to identify the coefficients to be estimated \cite{Kushilevitz-Mansour,GL}.} under the random sample model, no degree constraints are imposed. In particular, the low-degree assumption (degree at most $k$) is just a special case that provides this knowledge (with order $d^k$ time complexity), monomials of degree $k$ or $d-k$ are equivalent in the eye of the sparse Fourier algorithm. This is not necessarily the case for SGD-trained neural networks.

A line of work has considered SGD learning on `unconstrained' neural networks (besides polynomial size) and shows that we can emulate any efficient PAC or SQ algorithm \citep{abbe2020universality,abbe2021power,malach2020implications}. Such networks are far from the practical neural networks used in applications. Against this state of affairs, several works have attempted to derive computational lower bounds on learning with regular neural networks. For example, \cite{inal} shows that for fully connected 2-layer networks, if the initial alignment (INAL) of a network with a Boolean target function     
 (measured by the maximal expected correlation between target and neurons) is not significant, then noisy-GD cannot amplify the correlation to any significant level. This is achieved by showing that a low INAL implies a large minimal degree in the target function (thus a large leap) under some additional conditions. Another work \citep{abbe2022non} uses the permutation, sign-flip, or rotational equivariance of noisy-GD training of fully-connected neural networks to show a lower bound on the number of gradient descent steps required for global convergence, when we have access to population gradients with an additive Gaussian noise. In particular, for leap-$L$ functions on the hypercube and the hypercube, $\Omega (d^{L} \tau^2)$ steps are to shown to be required, where $\tau^2$ is the Gaussian noise variance. This roughly matches the conjecture in this paper in its exponential dependence on the leap -- however, the computational model is different (noisy-GD versus online-SGD).

Finally let's remark that a large body of work in the statistics and machine learning literature has studied the problem of learning multi-index models. These include for example phase retrieval \citep{candes2013phaselift}, intersection of halfspaces \citep{klivans2004learning,vempala2010random} and subspace juntas \citep{vempala2011structure,de2019your}. We refer to \cite{dudeja2018learning,chen2020learning} and references therein for an overview of this line of work. In particular, it is well understood that in order to break the ``curse of dimensionality'', the algorithm needs to estimate the low-dimensional support. In contrast with this line of work, we consider learning these multi-index functions with generic SGD on regular neural networks, with no a priori information on the target function. Surprisingly, we show that this generic algorithm can nearly match the computational complexity of the best CSQ algorithm. Note that specialized algorithms can achieve better sample and computational complexity: for example, \cite{chen2020learning} showed an algorithm that can learn low-rank Gaussian polynomials in $\widetilde{O}_d (d)$ samples and $\widetilde{O}_d (d^3)$ runtime, regardless of the leap-complexity, by going beyond CSQ algorithms.

\subsection{Discussion on the definition of the leap complexity}
\label{app:discussion_leap}

It was noted in \cite{abbe2022merged} that some ``degenerate'' leap-1 functions on the hypercube  are not learned in $\Theta (d)$ SGD-steps. Take for example $h_*(\bz) = z_1 + z_2 + z_3 + z_1z_2z_3$: by permutation symmetry on the support of $h_*$, $O(d)$ steps of SGD will learn first layer weights $\bw_j$ aligned with $(1,1,1)$ on the support $(z_1 ,z_2,z_3)$. SGD will require many more steps to break this symmetry\footnote{We conjecture $\Theta (d\log (d)^C)$ steps are required, see following discussion in the Gaussian case.} and fit $h_*$. \cite{abbe2022merged} circumvents this difficulty under a smoothed complexity analysis, and shows that the set of degenerate leap-$1$ functions has $\{ \hat h_* (S) \}_{S \in \cS} $ of Lebesgue-measure $0$. Alternatively, a possible approach to learn these degenerate cases (for ``axis-aligned'' sparse functions) is to use different random learning rates for each coordinates and break the symmetry in learning.

On the other hand, Gaussian data offers a more natural setting to understand these degenerate functions: indeed, the decomposition \eqref{eq:decompo_monomials} depends on the specific coordinate axis used to define the product of Hermite polynomials. In particular, $\mathrm{Leap} (h_*)$ will depend on the specific basis for this expansion. By rotational symmetry of the Gaussian distribution and equivariance of neural networks with isotropic initialization, the time complexity of SGD will be driven by the following ``isotropic leap'' complexity:
\[
\mathrm{isoLeap}(h_*) = \max_{R \in \cO_P} \mathrm{Leap}(h_*, R) \, ,
\]
where $\mathrm{Leap}(h_*, R)$ corresponds to the leap complexity of Definition \ref{def:leap_general}, where we made the dependency on the specific choice of axis $R$ for the Hermite expansion explicit. Adapting the proofs of \cite{abbe2022merged}, we can show that if $\mathrm{isoLeap}(h_*) > 1$, then $h_*$ cannot be learned in $\Theta (d) $ SGD-steps in the mean-field regime, while if $\mathrm{isoLeap}(h_*) = 1$, then the span of $\bw_j$'s covers the entire support of $h_*$ and not a subspace as for degenerate functions\footnote{Proving full learnability in $\Theta (d)$ steps of functions with $\mathrm{isoLeap}(h_*) = 1$ on Gaussian data is technically challenging and not the purpose of the current paper, and would require a separate analysis, see \cite{abbe2022merged} for details on what it would entail.}. Consider for example the case of $h_* (\bz) = z_1 + z_2 + z_1z_2$: in this coordinate basis, $h_*$ is a leap-1 function\footnote{Note that $\He_1 (x) = x$ and we can rewrite $h_* (\bz) = \He_1(z_1) + \He_1(z_2) + \He_1(z_1)\He_1(z_2)$.}. However, we can consider the following rotation $(u_1,u_2) = (z_1+z_2, z_1 - z_2)/\sqrt{2}$ and $h_* (\bz) = u_1 + \He_2 (u_1)/\sqrt{8} - \He_2 (u_2)/ \sqrt{8}$ in this basis. Therefore $\mathrm{isoLeap}(h_*) = 2$ and $h_*$ cannot be learned in $\Theta(d)$ SGD steps in the mean-field regime.

For technical reasons, we consider in Section \ref{sec:learning_leap_SGD} the support and the decomposition of $h_*$ aligned with the canonical basis of $\bx$. This can be seen as a smoothed complexity setting similar to \cite{abbe2022merged}: almost surely over the Hermite coefficients, the basis that maximizes the leap is the original basis in Eq.~\eqref{eq:decompo_monomials}, i.e., fixing the axis coordinates $R_*$, then $\mathrm{isoLeap}(h_*) = \mathrm{Leap}(h_*, R_*) $ almost surely over the Hermite coefficients $\{\hat h (S) \}_{S \in \cS}$ in the basis aligned with $R_*$. However, we stress here that $\mathrm{isoLeap}$ is the right measure for SGD-complexity in the case of general (not axis-aligned) low-dimensional support.

\subsection{Intuition for the proof of Theorem \ref{thm:alignment_one_monomial}}\label{app:intuition_monomial}

In this section, we give some intuition behind the proof of Theorem \ref{thm:alignment_one_monomial}. The complete proof can be found in Appendix \ref{app:proof_alignment_one_monomial}.

We first consider a simple SGD dynamics, with no projection step, and neglect the biases. We later discuss our choice of algorithm and how the analysis needs to be modified to control the projection step. The dynamics on the first layer weights is now simply given by
\[
\bw_j^{t+1} = \bw_j^t + \eta_1 a_j^0 ( y^t - \hf_{\NN} (\bx^t ; \bTheta^t ) ) \sigma ' ( \< \bw_j^t , \bx^t \> ) \bx^t \, .
\]

\paragraph*{Reduction to a correlation flow:} Recall that we initialize the second layer weights $|a_j^0| = \kappa$. By Assumption \ref{ass:sigma_I}, we have
\[
| \hf_{\NN} (\bx^t ; \bTheta^t ) | \leq \sum |a_j^0| | \sigma ( \< \bw_j^t , \bx^t \>) | \leq M K \kappa \, .
\]
With high probability over a polynomial number of steps, $\| \bx^t \|_\infty \leq C \log (d)$ with $C$ constant chosen sufficiently large. Hence,
\[
\bw_j^{t+1} = \bw_j^t + \eta_1 a^0_j y^t \sigma ' ( \< \bw_j^t , \bx^t \> ) \bx^t + O( M \eta_1 \kappa^2 \log (d) )\, ,
\]
and we can chose $\kappa,\eta_1$ with $\eta_1 \kappa^2$ sufficiently small, while keeping $\eta_1 \kappa$ constant, so that we can neglect the interaction term between the different neurons and get:
\[
\bw_j^{t+1} \approx \bw_j^t + \eta a^0_j y^t \sigma ' ( \< \bw_j^t , \bx^t \> ) \bx^t \, .
\]
This means that in this scaling and with high probability, the derivative of the square loss $\ell(y,\hat{y}) = \frac{1}{2}(y-\hat{y})^2$ in $\hat{y}$ is well-approximated by $\ell' ( y, \hat y) \approx -y$. Under this approximation, the different neurons do not interact while the first-layer weights are being fit, and so the neuron dynamics can be analyzed individually.

\paragraph*{Heuristic derivation of $\oT_1$ and $\eta_1$:}
Let us directly consider the correlation loss and track the dynamics of a unique neuron $(a,\bw)$. We assume that $w^0_1 = \ldots = w^0_d = 1 /\sqrt{d}$, $a^0 = \kappa$ and $ \mu_D (\sigma) >0$. We further make the following heuristic simplification: we assume the dynamics is described by only two parameters
\[
\alpha_1^t = w_1^t = \ldots = w_P^t \, , \qquad \alpha_2^t = w_{P+1}^t = \ldots = w_d^t \, ,
\]
with SGD updates $g_1^t = y^t \sigma ' ( \< \bw^t , \bx^t \>) x_1^t$ and $g_2^t =  y^t \sigma ' ( \< \bw^t , \bx^t \>) x_{P+1}^t$, i.e.,
\[
\begin{aligned}
    \alpha_1^{t+1} =&~ \alpha_1^t + \eta_1 \kappa g_1^t \, , \\
    \alpha_2^{t+1} =&~ \alpha_2^t + \eta_1 \kappa g_2^t \, .
\end{aligned}
\]
The computation follows from a similar strategy as in \cite{tan2019online,arous2021online}: namely, we will decompose  the dynamics into a drift term (deterministic) and a martingale term.
Introduce the population gradient $\og_i^t = \E_{y^t,\bx^t} [ g_i^t]$, $i \in \{1,2\}$. 
We can decompose the dynamics into a sum of population gradients (deterministic drift) and the martingale difference between the stochastic gradients and population gradients (recall that the $(y^t,x^t)$'s are independent):
\begin{equation}\label{eq:heur_dyn}
\begin{aligned}
    \alpha_i^{t+1} =&~ \alpha_i^t + \eta_1 \kappa \og_i^t +  \eta_1 \kappa (g_i^t - \og_i^t) = \alpha_i^0 + \eta_1 \kappa \sum_{s = 0}^{t} \og_i^t +  \eta_1 \kappa \sum_{s = 0}^{t} (g_i^t - \og_i^t ) \, .
\end{aligned}
\end{equation}
By using that for Hermite polynomials $\E_G [ \He_k (G) f(G)] = \E_G [ f^{(k)}  (G)]$ with $G \sim \normal (0,1)$, we can show that
\[
\begin{aligned}
\E [ \He_D (g_1) f (\bu^T \bg) g_1 ] =&~ D u_1^{D-1}\E [ f^{(D-1)} (\bu^T \bg)] + u_1^{D+1} \E [ f^{(D+1)} (\bu^T \bg)] \, ,\\
\E [ \He_D (g_1) f (\bu^T \bg) g_2 ] =&~  u_1^{D} u_2 \E [ f^{(D+1)} (\bu^T \bg)] \, ,
\end{aligned}
\]
where $\bu = (u_1,u_2)$ and $\bg = (g_1,g_2) \sim \normal (0,\id_2)$. We deduce that to leading term (assuming $\| \bw^t \|_2 \approx 1$)
\[
\begin{aligned}
    \og_1^t =&~ \E [ h_* (\bz) \sigma' (\< \bw^t , \bx \>) x_1] \approx (\alpha_1^t)^{D-1} \E [  \sigma^{(D)} (\| \bw^t \| G)] \approx \mu_D (\sigma) (\alpha_1^t)^{D-1} \, , \\
    \og_2^t = &~ \E [ h_* (\bz) \sigma' (\< \bw^t , \bx \>) x_{P+1}] \approx (\alpha_1^t)^{D} \alpha_2^t \E [  \sigma^{(D+2)} (\| \bw^t \| G)] \approx \mu_{D+2} (\sigma) (\alpha_1^t)^{D} \alpha_2^t\, .
\end{aligned}
\]

Let us now control the different contributions to the dynamics:
\begin{itemize}
    \item[(i)] \textbf{Martingale part:} By Doob's maximal inequality for martingales, we have with high probability
    \[
    \sup_{1 \leq t \leq \oT_1 -1} \left\vert \eta_1 \kappa \sum_{s = 0}^t  (g_i^t - \og_i^t) \right\vert \lesssim \eta_1 \kappa \sqrt{\oT_1} \, .
    \]
    We choose $\eta_1 \kappa$ so that we can neglect the martingale contribution during the entire dynamics by taking $\eta_1 \kappa \sqrt{\oT_1} \lesssim \alpha_i^0 = d^{-1/2}$.

    \item[(ii)] \textbf{Drift part for $\alpha_1$:} We now neglect the martingale term and write for all $0 \leq t  \leq \oT_1 -1$
    \begin{equation}\label{eq:drift_alpha1}
    \alpha_1^{t+1} \approx \alpha_1^0 + \eta_1 \kappa \sum_{s=0}^t \og_1^t \approx \alpha_1^0 + \eta_1 \kappa \mu_D (\sigma) \sum_{s=0}^t (\alpha_1^s)^{D-1} \, .
    \end{equation}
    We can study this sequence (see \cite{arous2021online}) and show that
    \[
    \alpha_1^{t} \approx \frac{1}{ \left( (\alpha_1^0)^{-(D-2)} - \eta_1 \kappa \mu_D (\sigma) t \right)^{1/(D-2)}} \, .
    \]
    In order for $\alpha^{\oT_1}_1 \approx 1$, we need to take $\eta_1 \kappa \mu_D (\sigma ) \oT_1 \gtrsim (\alpha_1^0)^{-(D-2)} = d^{D/2 - 1}$.

    \item[(iii)] \textbf{Drift part for $\alpha_2$:} Again, by neglecting the martingale contribution and for $0 \leq t \leq \oT_1 - 1$,
    \[
    \alpha_2^{t+1} \approx \alpha_2^0 + \eta_1 \kappa \mu_{D+2} (\sigma)  \sum_{s=0}^t (\alpha_1^s)^{D} \alpha_2^s \, .
    \]
    We can show that this sequence is bounded by
    \begin{equation}\label{eq:alpha2_bound}
    \begin{aligned}
    \ln \left( \frac{\alpha_2^{t+1}}{\alpha_2^0} \right) \lesssim&~ \eta_1 \kappa \mu_{D+2} (\sigma)  \sum_{s=0}^t (\alpha_1^s)^{D}\\
    \leq&~ \frac{\mu_{D+2} (\sigma)}{\mu_D (\sigma)} \eta_1 \kappa \mu_D (\sigma) \sum_{s=0}^t (\alpha_1^s)^{D-1} \leq \frac{\mu_{D+2} (\sigma)}{\mu_D (\sigma)} \alpha^{t+1}_1 \, ,
    \end{aligned}
    \end{equation}
    where we used Eq.~\eqref{eq:drift_alpha1} in the last inequality. 
\end{itemize}

We deduce from Eq.~\eqref{eq:alpha2_bound} that for $\oT_1$ chosen such that $\alpha^{\oT_1}_1 \approx 1$, then $\ln ( \alpha_2^{t+1} / \alpha_2^0) \lesssim 1$. Hence, during the dynamics, the weights $\alpha_2^t$ not aligned with the support of $h_*$ remain small, of order $1/\sqrt{d}$, while the weights $\alpha_1^t$ aligned with the support of $h_*$ become of order $1$.
From the bounds in (i) and (ii), we need to choose $\eta_1 $ and $\oT_1$ such that $\eta_1 \kappa \sqrt{\oT_1} \lesssim d^{-1/2}$ (martingale part) and $\eta_1 \kappa \oT_1 \gtrsim d^{D/2 - 1}$ (drift part), i.e., we can take
\[
\oT_1 \approx   d^{D - 1}\, , \qquad \eta_1 \approx \frac{1}{\kappa d^{D/2}} \, , 
\]
which matches the scaling in Theorem \ref{thm:alignment_one_monomial}.

\paragraph*{Adding a projection step:} While the above heuristic derivation was useful to get intuitions, the assumption that the weights remain equal (or approximately equal) on and outside the support is not valid. Because of the statistical fluctuations over $\widetilde{\Theta} (d^{D - 1})$ steps, different coordinates over different neurons will grow to be order 1 on the support at a stochastic time (with high probability between $d^{D - 1}/(C \log(d)^C)$ and $d^{D - 1}C \log(d)^C$ for some large enough constant $C$). To prevent these coordinates to continue growing (because we neglected the interaction term in the dynamics, which could otherwise prevent this growth), we introduce the projection step
\begin{equation}\label{eq:correlation_dynamics_oneM}
\begin{cases}
\tbw^{t+1} =& \bw^t + a y^t  \eta_1 \cdot \grad_{\bw^t} \sigma ( \< \bw^t, \bx^t \>) \, ,\\
\bw^{t+1} = & \proj^{t+1} \proj_{\infty}  \tbw^{t+1}  \, ,
\end{cases}
\end{equation}
where $\proj^{t+1} \proj_{\infty}$ is the projection step defined in Eq.~\eqref{eq:first-layer-projection}, and we use the spherical gradient defined in Eq.~\eqref{eq:first-layer-sphere-grad}. Note that because of the choice $\Delta > r$ and the definition of the set $S_{j,t}$ on which we do the projection on the sphere, $\proj^{t+1}$ and $ \proj_{\infty}$ commute.

Thanks to the spherical gradient, we can show that the projection steps only have a negligible impact on the dynamics (similarly to the analysis in \cite{arous2021online}). By carefully arranging these additional terms, we can essentially recover the drift plus martingale analysis presented heuristically above.

\subsection{Going beyond sparsity}
\label{app:beyond_sparse}

In the $P=O_d(1)$ regime the complexity scaling in $d$ is dominated by the `hard' part of learning the low-dimensional latent space on which the function depends, and the complexity of fitting the function on the support is secondary and only results in constants. This also makes the conjecture fairly general in terms of architecture choices as long as there is enough expressivity to fit the function on the support. One could also consider functions that depend on a finite number of basis elements, without necessarily involving a finite number of coordinates. For instance the full parity $\prod_{i\in [d]}x_i$ function is such an example. For SQ algorithms, the class of monomials of degree 0 (more generally $k$) has equivalent complexity to the class of monomials of degree $d$ (more generally degree $d-k$), and the SQ-dimension is symmetrical for these dual cases. However for SGD learning on regular nets, this is not exactly the case. It is true that the full parity can be learned by regular nets under a specific setting; \cite{abbe2022non} provides a regular 2-layer neural net that can learn the full parity if the weight measure of the first layer at initialization is i.i.d.\ Rademacher(1/2) and the activation is a ReLU. A constant number of step can also be sufficient in such cases, as for the 0-degree monomial. It is however conjectured that this is not achievable with a polynomial number of steps for weights that have a Gaussian initialization. Thus, for isotropic layers, it is possible that the full parity is not polytime learnable. This means that the generalized notion of leap to non-coordinate sparse may depend on more specific choices of the parameters. Further, in the non-isotropic case where the full parity is efficiently learnable, one may define the leap with basis sets that can either grow from the 0-monomial or descend from the full-monomial, with the mirror symmetry as for SQ algorithms.

Another notion to factor in when considering non-coordinate sparse function is the fitting of the function once the support is learned. First of all, there may be a non-polynomial number of coefficients to handle, although one can probably cover enough interesting cases with functions that are well-approximated by polynomially many coefficients \citep{odonnell_2014}. Further, there is the fitting of the function by the neural net that may now turn non-trivial. Consider even a function with few basis elements, $h_*(\bz) = \sum_{i=1}^P ix_i + \prod_{i=1}^P x_i^{g(i)}$, where $g : [P] \to \{0,1\}$ is an arbitrary, but known function, and $P \gg 1$ is large. SGD on a regular neural network would first pick up the $P$ coordinates in the support and then learn the monomial $\prod_{i=1}^P z_i^{g(i)}$ based on that support. The latter part may not be trivial for SGD on a regular net, while it would require 0 queries for an SQ algorithm (once the linear part is learned, the permutation is identified and the coefficients in front of each variable would allow us to calculate $g(i)$). Thus the complexity of learning the second monomial on the detected support set is likely to factor in for such cases, and this is likely going to depend more on the model hyperparameters and architecture choice. In less contrived cases, the naive generalization of the leap applied verbatim to non-constant $P$ remains likely relevant.

\subsection{Lower-bounds: beyond noisy GD}\label{app:beyond_CSQ}
Note that the CSQ and noisy-GD models do not exactly match the SGD learning model; we do prove in this paper that the drift of the population gradient dominates the dynamic on the considered horizon, but the CSQ model also has noise added to the query outputs. It is nonetheless interesting that the regularity of the network model drives us to an achievability result that matches that of CSQ lower-bounds. Since it is known how to go beyond the CSQ/SQ lower-bounds with non-regular networks \cite{abbe2020universality}, e.g., learning dense parities by emulating matrix inversions with irregular networks, our results raise an intriguing question: may the model ``regularity'' act comparably to a CSQ constraint? We leave this to future work.

A result in  \cite{abbe2022merged} does derive a lower-bound that does not require additive noise and that applies to  online-SGD. This work requires however a few restrictions: the mean-field parametrization (and not just any isotropic distribution) and a linear sample complexity (e.g.,  finite number of time steps with linear batches). These are used to derive a specialization of the mean-field PDE approximation \citep{mei2018mean} to the coordinate sparse setting. In turn, this allows to show that the sample complexity for SGD learning on functions that do not satisfy the merged-staircase property ($\mathrm{Leap}=1$) cannot be learned with a linear sample complexity.

\clearpage

\section{Proof of Theorem \ref{thm:alignment_one_monomial}: alignment with a single monomial}\label{app:proof_alignment_one_monomial}

In this appendix, we prove the alignment of the first layer's weights with the support of one monomial. The proof will follow from a similar proof strategy as in \cite{tan2019online,arous2021online}, namely decomposing the dynamics into drift and martingale terms. However, it will differ in a key aspect: while \cite{arous2021online} considers a single-index model, we will need to track for each neuron $P$ parameters (the first $P$ coordinates of $\bw_j$) and show that the $d-P$ other parameters remain well behaved along their whole trajectories, which requires a tighter control of the different contributions to the dynamics.

Recall that we denote by $K$ a constant that only depends on $\sigma$ (Assumption \ref{ass:sigma_I}) and the sub-Gaussianity of the label noise $\eps$. Throughout the proofs, we will write $C,c>0$ for generic constants that only depend on $D$ and $K$. The values of these constants are allowed to change from line to line or within the same line.

\subsection{Preliminaries}

In the proof, we will consider $0<r\leq \Delta \leq 1$ to be small enough constants that can depend on $D$ and $K$, but are independent of $d$. We will track the dependency in $r,\Delta$ when necessary, and otherwise use that they are bounded by $1$ (in particular, the constants $c,C$ in the proof will be independent of $r,\Delta$). These constants $r,\Delta$ will be fixed in Theorem \ref{cor:learning_one_monomial}.

We will show that we can take initialization scale $\kappa$ of second layer weights $a^0$ and step size $\eta$ such that the dynamics of the first layer training can be approximated by a correlation dynamics, with no interactions between the neurons, so that we can analyze each neuron independently. We consider below an arbitrary neuron $(a_j,b_j,\bw_j)$ for $j \in [M]$. In the case that $a_j^0 \mu_D(\sigma ( \cdot + b_j)) (w_{j,1})^{k_1} \cdots (w_{j,P}^0)^{k_P} > 0$ we 
prove that the event claimed in Theorem~\ref{thm:alignment_one_monomial}.(b) and (c) holds with probability at least $1 - d^{-C_*}$ for neuron $j$. Theorem~\ref{thm:alignment_one_monomial}.(a) will follow from a similar analysis. The result for all neurons follows by a union bound. 

We further consider $|b_j^0| \leq \rho \leq \Delta $ small enough such that $1/2 \leq |\mu_k(\sigma ( \cdot + b_j))|/ |\mu_k (\sigma) | \leq 3/2$ for $k=0, \ldots , D+2$ (see comments below Lemma \ref{lem:gradient_formula}). Hence, the biases will not impact the training of the first layer weights and for the simplicity, we will fix $b_j = 0$ in the proof.

\paragraph{Nonnegative first layer weights} Without loss of generality, we assume that all of the first-layer coordinates of neuron $j$ have positive sign at initialization $w_{j,1}^0 = \ldots = w_{j,d}^0 = 1/\sqrt{d}$ (and therefore $a^0 \mu_D (\sigma)>0$ by our choice of $(a_j^0, \bw_j^0)$). 
To see why, define $s_0 = \prod_{i \in [P]} \sign (w_{j,i}^0)^{k_i}$ and consider instead initializing the network at $\breve{\bTheta}^0 = (\breve{\ba}^0, \breve{\bW}^0)$ where $\breve{\ba}^0 = s_0 \ba^0$ and $\breve{\bw}_{j'}^0 = \bw_{j'} \odot \sign(\bw_j^0)$ for all $j'$. Then consider training the network with samples $(\breve{y}^t,\breve{\bx}^t)$ where $\breve{\bx}^t = \bx^t \odot \sign(\bw_j^0)$ and $\breve{y}^t = f_*(\breve{\bx}^t) + \eps_t$. The distribution of data $(\breve{y}^t,\breve{x}^t)$ is the same as that of $(y^t,\bx^t)$, and the the training dynamics $\breve{\bTheta}$ match those of $\bTheta$ up to sign flips, and $\breve{\bw}_j^0 = [1/\sqrt{d}, \ldots, 1/\sqrt{d}]$.
\\

For brevity of notation, let us drop the subscript $j$ in the remainder of the analysis of this section, and 
write $(a,\bw)$ to denote $(a_j,\bw_j)$ whenever it can be inferred from context.

\paragraph*{Stopping times on the dynamics:} Recall that the loss is the square loss $\ell(y,\hat y) = \frac{1}{2} (y - \hat y)^2$. We denote $\bv^t$ the (negative) stochastic gradient at time $t$:
\begin{align*}
\bv^t := -  \nabla_\bw \ell \big(y^t, \hf_\NN(\bx^t;\bTheta^t)) \big) = (y^t - \fNN(\bx^t;\bTheta^t)) a^0 \sigma' (\< \bw^t , \bx^t \>) \bx^t \, ,
\end{align*}

 Further recall that we constrain the dynamics of $\bw^t$ in two ways: the $\ell_\infty$ constraint $\| \bw^t \|_\infty \leq \Delta$ and the spherical constraint $\| \cS_t ( \bw^t ) \|_2 = 1$, where $\cS_t (\bw^t) = ( w_i^t \ind_{i \in S_t})_{i \in [d]}$ is the projection on the subset of coordinates $S_t$ which contains all coordinates that verify $|\tw^{t'}_i | < r$ for all times $t ' \leq t$. Let us define $\tbv^t$ the spherical gradient update on support $S_t$:
 \[
\begin{aligned}
\tbv^t = &~ - \grad_{\bw^t} \ell(y^t,\fNN(\bx^t;\bTheta^t)) = \bv^t - \cS_t (\bw^t) \< \cS_t ( \bw^t) , \bv^t \>
\end{aligned}
\]
The update equations are given by
\[
\begin{cases}
\tbw^{t+1} =& \bw^t + \eta_1 \tbv^t \, ,\\
\obw^{t+1} = & \proj_\infty \tbw^{t+1}  \, , \\
\bw^{t+1} = & \proj^{t+1} \obw^{t+1} \, ,
\end{cases}
\]
where we recall that we defined
\[
(\proj^{t+1} \obw^{t+1})_i = \begin{cases}
    \ow_i^{t+1} & \text{if $i \not\in S_{t+1}$}\, , \\
    \frac{\ow_i^{t+1}}{\| \cS_{t+1} (\obw^t ) \|_2} &\text{if $i \in S_{t+1}$}\, , \\
\end{cases}
\]
with $S_{t+1} = S_t \setminus \{ i \in [d] : | \tw_i^{t+1} | \geq r\}$.

Let us introduce the following stopping times on the dynamics:
\[
\begin{aligned}
\tau^+  = &~ \inf \Big\{ t \geq 0 : \max_{ i= P+1 , \ldots , d}  \{ |\tw_i^{t+1}| \vee | w_i^{t+1} |\} \geq 3/(2 \sqrt{d}) \Big\}\, , \\
 \tau^- =&~ \inf \Big\{ t \geq 0 : \min_{ i\in [d]} \{ | \tw_i^{t+1} | \wedge | w_i^{t+1} |\} \leq 1/(2\sqrt{d} ) \Big\} \, , \\
 \tau^0 = &~ \inf \Big\{ t \geq 0 : \max(\| \bv^t / a^0 \|_\infty, \| \tbv^t / a^0 \|_\infty, | y^t |, \| \bx^t \|_\infty) \geq C_0 \log (d)^{C_0} \Big\} \, .
\end{aligned}
\]
Note that $\{ \tau = t\} \in \cF_t := \sigma \big( \bTheta^0, \{ \bx^s,y^s \}_{s \leq t} \big)$ for $\tau \in \{ \tau^+, \tau^-, \tau^0\}$ and $\sigma(\bw^{t+1}),\sigma(S_{t+1}) \subseteq \cF_t$. For $t \leq \tau^+$ and $r\geq 3/(2\sqrt{d})$, we have $\{P+1,\ldots , d\} \subseteq S_t$, and $\| \bw^t\|_2 \leq \| \bw_{1:P}^t\|_2 +\| \cS_t (\bw^t ) \|_2 \leq  \sqrt{P} \Delta + 1$. We further define for all $i\in [d]$,
\[
\begin{aligned}
\tau_i^r =&~ \inf \Big\{ t \geq 0:|  \tw_i^{t+1} | \geq r \Big\} \, , \\
\tau_i^{\Delta} =&~ \inf \Big\{ t \geq 0 :| w^{t+1}_i | \geq \Delta - |a^0| \eta_1 C_0 \log(d)^{C_0} \Big\}\, , \\
\tau^r =&~ \sup_{i \in [P]} \tau_i^r\, , \qquad \qquad\tau^\Delta =\sup_{i \in [P]} \tau_i^\Delta \, ,
\end{aligned}
\]
where $C_0$ is a constant that will be chosen large enough. In particular, at time $\tau_i^r+1$, the $i$-th coordinate is removed from the set on which we do the projection, i.e., $\{ i \} \subseteq S_{\tau_i^r } \setminus S_{\tau_i^r +1 }$. We will show in the proof that $\tau^+ \wedge \tau^- \wedge \tau^0 > \oT_1$ with high probability. 

By concentration of polynomials of Gaussian variables, we have:

\begin{lemma}\label{lem:bound_on_tau_0}
Assume that $\Delta \leq 1$. Then for any $C_* >0$, there exists $C_0$ large enough that only depends on $C_*,D$ and $K$, such that for $d \geq 2$,
\[
 \P ( \tau^0 \leq \tau^+  \wedge d^{D} ) \leq d^{-C_*} \, . 
\]
\end{lemma}

\begin{proof}[Proof of Lemma \ref{lem:bound_on_tau_0}]
For $t \leq \tau^+$, we must have $\| \bw^t \|_2 \leq \sqrt{P} + 1 $. Using the bounds \eqref{eq:tail_bounds_fct_gaussian} in Lemma \ref{lem:bounds_gradient} and a union bound, there exists a constant $C_0$ such that
\[
\begin{aligned}
\P ( \tau^0 \leq \tau^+  \wedge d^{D} )  
\leq&~ \sum_{t \leq \tau^+  \wedge d^{D} } \P_{\bw^t} \big( \max(\| \bv^t / a^0 \|_\infty, \| \tbv^t / a^0 \|_\infty, | y^t  |, \| \bx^t \|_\infty) > z_* + c  \big) \\
\leq&~ c d^{D+1} \exp ( - C (z_*)^{2/(3D +3)} ) \leq d^{-C_*}\, ,
\end{aligned}
\]
where $z_* = C_0 \log (d)^{C_0} - c$.
\end{proof}

\paragraph*{Reducing to the correlation flow}
We now show that for second-layer initialization scale $\kappa$ small enough, the updates $\bv^t$ and $\tilde{\bv}^t$ mostly come from correlation term in the square loss, and the self-interaction term contributes negligibly. Define the gradient and spherical gradient from the correlation term as:
\begin{align*}
\bu^t = a^0 y^t \sigma'(\<\bw^t,\bx^t\>) \bx^t\, \qquad \mbox{ and } \qquad \tbu^t = \bu^t - \cS_t(\bw^t)\<\cS_t(\bw^t),\bu^t\>\,.
\end{align*}
Then $\bv^t$ is close to $\bu^t$ and $\tilde{\bv}^t$ is close to $\tilde{\bu}^t$.
For $t < \tau^0$, we have the following bounds. Use (a) $\|\sigma\|_{\infty}, \|\sigma'\|_{\infty} \leq K$ and $\|\ba^0\|_{\infty} = \kappa$, and (b) $\|\cS_t(\bw^t)\|_{\infty} < r < 1$ and  $\|\cS_t(\bw^t)\|_1 < d$,
\begin{equation}\label{eq:v-vs-u}
\begin{aligned}
\|\bv^t - \bu^t\|_{\infty} / |a^0| &\stackrel{(a)}{\leq} \|\fNN(\bx^t;\bTheta^t) \sigma'(\<\bw^t,\bx^t\>)\bx^t\|_{\infty} \leq \kappa K^2MC_0 \log(d)^{C_0} \leq \tilde{\kappa}\, \\
\|\tilde{\bv}^t - \tilde{\bu}^t\|_{\infty} / |a^0| &= \|\bv^t - \bu^t - \cS_t(\bw^t)\<\cS(\bw^t), \bv^t - \bu^t\>\|_{\infty} / |a^0| \\
&\stackrel{(b)}{\leq} (1+d)\|\bv^t - \bu^t\|_{\infty} \leq \tilde{\kappa}\,,
\end{aligned}
\end{equation}
for
$\tilde{\kappa} = 2\kappa dK^2MC_0$.

\paragraph*{Simplifying the update equations:} Note that for $t < \tau_i^\Delta \wedge \tau_0$, we have $|\tw_i^{t+1}| = |w_i^t + \eta_1 \tv_i^t | \leq \Delta$, and therefore $\ow_i^{t+1} = \tw_i^{t+1}$.  Let us introduce the truncated spherical gradient $\bg^t$ defined by
\[
g^t_i := \begin{cases}
    \tv_i^t & \text{for $t < \tau_i^\Delta \wedge \tau_0$}\; , \\
    \gamma_i^t \tv_i^t & \text{for $t \geq \tau_i^\Delta \wedge \tau_0$}\; , \\
\end{cases}
\]
where $\gamma_i^t \in [0,1]$ is a multiplicative factor that models the projection step $\obw^{t+1} = \proj_{\infty} \tbw^{t+1}$,
\[
\gamma_i^t = \min \Big( \frac{\Delta - \sign ( \tv_i^t ) w_i^t }{\eta_1 | \tv_i^t | }, 1 \Big) \, .
\]
It is easy to check that $\sigma (\bg^t) \subseteq \cF_t$ and $\obw^{t+1} = \proj_{\infty} ( \bw^t + \eta_1 \tbv^t ) = \bw^t + \eta_1 \bg^t$ for all $t \geq 0$. With these notations, our dynamics are now simply given by
\begin{equation}\label{eq:dyn_eq_simplified}
\begin{cases}
\obw^{t+1} =& \bw^t + \eta_1 \bg^t \, ,\\
\bw^{t+1} = & \proj^{t+1} \obw^{t+1}    \, .
\end{cases}
\end{equation}

\paragraph*{Population gradients:} Let us define the population spherical gradient $\obg^t = \E_{\bx^t,y^t} [\bg^t ]$. We have the following formula on $\og_i^t$ for $t < \tau_i^\Delta \wedge \tau_0$:

\begin{lemma}\label{lem:gradient_formula}
Denote $\chi_* ( \bw^t )= \prod_{j \in [P]} (w_j^t)^{k_j}$.
For $i \in [P]$ and $t < \tau_i^\Delta \wedge \tau_0$, we have the following formulas that approximate the population gradient: if $t \leq \tau_i^r$ (i.e., $i \in S_t$), then $\og_i^t = \E_{\bx^t,y^t} \big[ \tv_i^t \big]$ and
\begin{equation}\label{eq:grad_formula_P_St}
\Bigg|\og_i^t - a^0 \frac{\chi_* (\bw^t)}{w_i^t} \left( k_i - (w_i^t)^2 \sum_{j \in S_t \cap [P]} k_j \right) \E_{G} \Big[ \sigma^{(D)} (\| \bw^t \|_2 G ) \Big]\Bigg| \leq |a^0|\tilde{\kappa} \, ,
\end{equation}
while if $t> \tau^r_i$ (i.e., $i \not\in S_t$), then $\og_i^t = \E_{\bx^t,y^t} [ v_i^t]$ and 
\begin{equation}\label{eq:grad_formula_P_notSt}
 \Bigg|\og_i^t -  a^0 \frac{\chi_* (\bw^t)}{w_i^t} \left( k_i  \E_{G} \Big[ \sigma^{(D)} (\| \bw^t \|_2 G ) \Big] + (w_i^t)^2 \E_{G} \Big[ \sigma^{(D+2)} (\| \bw^t \|_2 G ) \Big] \right)\Bigg| \leq |a^0|\tilde{\kappa} \, .
\end{equation}
For $i > P$ and $t < \tau^+ \wedge \tau_0$,
\begin{equation}\label{eq:grad_formula_notP}
\Bigg|\og_i^t + a^0 w_i^t \chi_* (\bw^t) \left( \sum_{j \in S_t \cap [P]} k_j \right)\E_{G} \Big[ \sigma^{(D)} (\| \bw^t \|_2 G ) \Big]\Bigg| \leq |a^0|\tilde\kappa\, .
\end{equation}
\end{lemma}

\begin{proof}[Proof of Lemma \ref{lem:gradient_formula}]
We recall the following useful identities (where $G \sim \normal (0,1)$)
\[
\E_G [ \He_k(G) g (G) ] = \E_G [g^{(k)} (G)]\, , \qquad  x \He_k (x) = \He_{k+1} (x) +k \He_{k-1} (x) \, .
\]
In particular, by integration by parts, we have
\[
\E \Big[ \prod_{ j \in [P]} \He_{v_j} (x_j ) \sigma ' ( \< \bw^t , \bx \> ) \Big] = \Big( \prod_{j \in [P]} (w_j^t)^{v_j} \Big) \cdot \E_{G} \big[ \sigma^{(1 + v_1 + \ldots + v_P)} (\| \bw^t \|_2 G ) \big] \, .
\]
Furthermore, if $i \in [P]$,
\[
x_i f_* (\bx) = \Big( \prod_{j \in [P], j \neq i} \He_{k_j} ( \bx_j)\Big) \big\{ k_i \He_{k_i - 1} (x_i) + \He_{k_i + 1} (x_i) \big\} \, .
\]
Hence, for $i \in [P]$ and $i \in S_t$, we have $|\og_i^t - \E_{\bx^t,y^t}[\tilde{u}_i^t]| \leq |a^0|\tilde{\kappa}$ by \eqref{eq:v-vs-u}, and
\[
\begin{aligned}
\E_{\bx^t,y^t}[\tilde{u}_i^t] =&~ \E_{\bx^t,\eps^t} \Big[ a^0 (f_* (\bx^t) + \eps^t) \Big\{ x_i^t - w_i^t \sum_{ j \in S_t} w_j^t x_j^t \Big\} \sigma ' ( \< \bw^t , \bx^t \> ) \Big] \\
=&~ a^0 \frac{\chi_* (\bw^t) }{w_i^t} \Big\{ k_i - (w_i^t)^2 \sum_{ j \in S_t \cap [P]} k_j \Big\}  \E_{G} \big[ \sigma^{(D)} (\| \bw^t \|_2 G ) \big] \\
&~+ a^0 \chi_* (\bw^t) \Big\{ w_i^t  - w_i^t \sum_{j \in S_t} (w_j^t)^2 \Big\} \E_{G} \big[ \sigma^{(D+2)} (\| \bw^t \|_2 G ) \big] \, ,
\end{aligned}
\]
which gives Eq.~\eqref{eq:grad_formula_P_St} by using $\| \cS_t (\bw^t) \|_2^2 = 1$. Eqs.~\eqref{eq:grad_formula_P_notSt} and \eqref{eq:grad_formula_notP} are obtained similarly.
\end{proof}

From Assumption \ref{ass:sigma_I}, we can choose $\Delta $ small enough and depending only on $K$ and $D$ such that that for all $u,v \in [-P\Delta,P\Delta]$ and $0 \leq k \leq D$
\begin{equation}\label{eq:choice_Delta_sigma}
 \Big\vert \E_{G} [ \He_k (G) \sigma ((1+v)G + u) ] - \mu_k (\sigma) \Big\vert \leq  \frac{|\mu_k (\sigma )|}{2}\, .
\end{equation}
and for $k =D+1$ or $D+2$,
\begin{equation}\label{eq:choice_Delta_sigma_2}
 \Big\vert \E_{G} [ \He_k (G) \sigma ((1+v)G + u) ] \Big\vert \leq  2K \, .
\end{equation}

We further assume that $\Delta$ is chosen small enough such that $\Delta^2 \leq 1/(2D)$ and $\Delta^2 \leq 1 / (4 K^2)$. With this choice of $\Delta$, there exist constants $C,c > 0$ that only depend on $D,K$ such that for all $t < \tau_i^\Delta \wedge \tau^+ \wedge \tau^-$, if $i \in [P]$,
\[
c a^0 \frac{\chi_* (\bw^t)}{w_i^t} \mu_D (\sigma) - |a^0| \tilde{\kappa} \leq \og_i^t \leq C a^0  \frac{\chi_* (\bw^t)}{w_i^t} \mu_D (\sigma) + |a^0| \tilde{\kappa}\, ,
\]
where we recall that we are now assuming that $\sign (w_i^0) = 1$ and therefore $\sign (w_i^t) =1$ for $t < \tau^-$, and $a^0 \chi_* (\bw^0) \mu_D (\sigma) >0$. Because $t < \tau^{-}$, and because $\kappa$ is chosen small enough so that $\tilde{\kappa} \ll (1/(2\sqrt{d}))^{D-1}$, we have
\begin{equation}\label{eq:og-on-first-P-bounds}
c a^0 \frac{\chi_* (\bw^t)}{w_i^t} \mu_D (\sigma) \leq \og_i^t \leq C a^0  \frac{\chi_* (\bw^t)}{w_i^t} \mu_D (\sigma)\, .
\end{equation}
And if $i >P$,
\[
|\og_i^t| \leq   C a^0 \mu_D (\sigma) w_i^t \chi_* (\bw^t) \Big\{ \sum_{j \in S_t \cap [P]} k_j \Big\} + |a^0| \tilde{\kappa}\, . 
\]
Notice that $|\og_i^t| \leq |a^0| \tilde{\kappa}$ for all $i>P$ and $t \geq \tau^r +1$ (i.e., $S_t \cap [P] = \emptyset$).

\subsection{Bounding the different contributions to the dynamics}
\label{sec:bounding_contributions_dynamics}

The following lemma tracks the contribution of the projection on the sphere $\bw^{t+1} = \proj^{t+1} \obw^{t+1}$:

\begin{lemma}\label{lem:proj_step_bounds}
Assume that $C_0 \log(d)^{C_0}/\sqrt{d} \leq r \leq \Delta /2 \leq 1/(\sqrt{8P})$, $|a^0| \leq 1$ and $\eta_1 \leq 1/d$. Then there exist constants $C, C' >0$ (that only depend on $D, K$ and $C_0$) such that for $d \geq C'$ and all $t < \tau^0 \wedge \tau^+$, if $S_{t+1} = S_t$,
\begin{equation}\label{eq:proj_step_bound_I}
\frac{1}{2} \leq 1 - C \eta_1^2 \|  \bg^t  \|_2^2 \leq \frac{1}{\| \cS_t ( \obw^{t+1} ) \|_2 } \leq 1 + C \eta_1^2 \|  \bg^t  \|_2^2 \, ,
\end{equation}
and if $S_{t+1} \neq S_t$, then
\begin{equation}\label{eq:proj_step_bound_II}
\frac{1}{2} \leq 1 - C r^2 \leq \frac{1}{\| \cS_t ( \obw^{t+1} ) \|_2 }  \leq 1 + Cr^2 \, .
\end{equation}
\end{lemma}

\begin{proof}[Proof of Lemma \ref{lem:proj_step_bounds}]
First consider the case $S_{t+1} = S_t$. We have $\cS_{t+1} ( \obw^{t+1} ) = \cS_t ( \bw^t) + \eta_1 \cS_t ( \bg^t)$. Note that on $i \in S_t$, we have $t < \tau_i^{\Delta}$ and therefore $\gamma_i^t ( \bw^t) = 1$ and $\cS_t ( \bg^t) = \cS_t ( \tbv^t)$. We therefore have
\begin{equation}\label{eq:proj_step_I}
\| \cS_{t+1} (\obw^{t+1} ) \|_2^2 = 1 + \eta_1^2 \| \cS_t ( \tbv^t) \|_2^2 \, ,
\end{equation}
where we used that $\| \cS_t ( \bw^t)\|_2^2 = 1$ and $\<\cS_t ( \bw^t), \cS_t ( \tbv^t) \> = 0$ by definition of the spherical gradient.
Furthermore, $\eta_1^2 \| \cS_t ( \tbv^t ) \|_2^2 \leq \eta_1^2 \| \bg^t \|_2^2 \leq d^{-2} \|\bg^t\|_2^2 \leq d^{-1} \| \bg^t \|_\infty^2 \leq C d^{-1} \log(d)^C \leq 1/4$ for $t < \tau_0$. Therefore, there exists a constant $C >0$ such that bound \eqref{eq:proj_step_bound_I} holds.

In the case $S_{t+1} \neq S_t$, we have $|S_{t} \setminus S_{t+1}| \leq P$ for $t < \tau^+$ and the coordinates that are removed at time $t+1$ satisfy $w_i^t + \eta_1 g_i^t \leq r + C_0 d^{-1/2} \log(d)^{C_0} \leq 2r$. Hence
\[
- 4 Pr^2  + \| \cS_{t} (\obw^{t+1} ) \|_2^2 \leq \| \cS_{t+1} (\obw^{t+1} ) \|_2^2 \leq  4 Pr^2  + \| \cS_{t} (\obw^{t+1} ) \|_2^2 \, .
\]
We can then use Eq.~\eqref{eq:proj_step_I} and that $\eta_1^2 \| \cS_t ( \tbv^t) \|_2^2 \leq Pr^2$ to derive Eq.~\eqref{eq:proj_step_bound_II}.
\end{proof}

Let us decompose the different contributions to the dynamics. We define $\bm^t = \bg^t - \obg^t$ the martingale updates. Let us bound the change of a coordinate after one update. For $t < \tau^0 \wedge \tau^+ \wedge \tau^-$, if $S_{t+1} = S_t$, then by Eq.~\eqref{eq:proj_step_bound_I}, we have for $i \in S_{t+1}$
\begin{equation}\label{eq:update_StSt_pre}
\begin{aligned}
w_i^{t+1} =&~ \frac{w_i^t + \eta_1 g_i^t}{\| \cS_{t+1} (\bw^t + \eta_1 \bg^t) \|_2}  \geq w_i^t + \eta_1 g_i^t - C \eta_1^2 \| \bg^t \|_2^2 | w_i^t | - C \eta_1^3 \| \bg^t \|_2^2 | g_i^t| \, ,\\
w_i^{t+1} \leq&~ w_i^t + \eta_1 g_i^t + C \eta_1^2 \| \bg^t \|_2^2 | w_i^t | + C \eta_1^3 \| \bg^t \|_2^2 | g_i^t|  \, .
\end{aligned}
\end{equation}
(Note that for $t < \tau^-$, we have $\sign (w_i^{t+1}) = \sign (w_i^t) = 1$.) 
For $t < \tau^0$, we have $\| \bg^t \|_\infty \leq C |a^0| \log(d)$ and $|w_i^t|/|w_i^{t+1}| \leq C$ because $\eta_1 \leq 1/d$. Hence, we can rearrange Eqs.~\eqref{eq:update_StSt_pre} and obtain
\begin{equation}\label{eq:update_StSt_pre_2}
\begin{aligned}
\big( 1 + C|a^0|^2 \eta_1^2 d \log(d)^{C}  \big) w_i^{t+1} \geq &~ w_i^t + \eta_1g_i^t -  C \eta_1^3 |a^0|^3 d \log(d)^C \, , \\
\big( 1 - C|a^0|^2 \eta_1^2 d \log(d)^{C}  \big) w_i^{t+1} \leq &~ w_i^t + \eta_1g_i^t +  C \eta_1^3 |a^0|^3 d \log(d)^C \, , 
\end{aligned}
\end{equation}
On the other hand, if $S_{t+1} \neq S_t$, by Eq.~\eqref{eq:proj_step_bound_II}, we have for $i \in S_{t+1}$,
\[
\begin{aligned}
\frac{w_i^{t+1}}{1-Cr^2} \geq&~  w_i^t + \eta_1 g_i^t \geq w_i^t + \eta_1 g_i^t - C \eta_1^2 \| \bg^t \|_2^2 | w_i^t | - C \eta_1^3 \| \bg^t \|_2^2 | g_i^t| \, , \\
\frac{w_i^{t+1}}{1+Cr^2} \leq &~ w_i^t + \eta_1 g_i^t + C \eta_1^2 \| \bg^t \|_2^2 | w_i^t | + C  \eta_1^3 \| \bg^t \|_2^2 | g_i^t| C \, .
\end{aligned}
\]
Rearranging these equations, we obtain for $t < \tau^0 \wedge \tau^-$ and $S_{t+1} \neq S_t$ (using $|Cr^2|\leq 1/2$),
\begin{equation}\label{eq:update_StnotSt_pre_2}
\begin{aligned}
\frac{1 + C|a^0|^2 \eta_1^2 d \log(d)^{C}}{1-Cr^2}  \cdot w_i^{t+1} \geq &~ w_i^t + \eta_1g_i^t -  C \eta_1^3 |a^0|^3 d \log(d)^C \, , \\
\frac{ 1 - C|a^0|^2 \eta_1^2 d \log(d)^{C}  }{1+ Cr^2} \cdot w_i^{t+1} \leq &~ w_i^t + \eta_1g_i^t +  C \eta_1^3 |a^0|^3 d \log(d)^C \, , 
\end{aligned}
\end{equation}
On the other hand, if $i \not\in S_{t+1}$, then
\[
w_i^{t+1} = w_i^t + \eta_1 g_i^t \, .
\]

Define $X_t = \# \{ i \in [P]: \tau_i^r < t \}$, and
\[
\up^t := \frac{ (1 - C|a^0|^2 \eta_1^2 d \log(d)^{C} )^t }{(1+ Cr^2)^{X_t}}\, , \qquad \op^t :=  \frac{ (1 + C|a^0|^2 \eta_1^2 d \log(d)^{C} )^t }{(1- Cr^2)^{X_t}} \, .
\]
Note $\sigma(\up^t), \sigma (\op^t) \in \cF_{t-1}$, so that $\up^t \bm^t$ and $\op^t \bw^t$ are still martingale updates.
By induction on Eqs~\eqref{eq:update_StSt_pre_2} and \eqref{eq:update_StnotSt_pre_2}, we deduce that for $t <\tau^0 \wedge \tau^+ \wedge \tau^-$,
\begin{equation}
    \begin{aligned}
    \up^{t \wedge \tau_i^r} w_i^t \leq&~ w_i^0 + \eta_1 \sum_{s= 0}^{t-1}  \up^{s \wedge \tau_i^r} \og_i^s + \eta_1 \sum_{s= 0}^{t-1} \up^{s \wedge \tau_i^r} m_i^s + C (t \wedge \tau_i^r) \up^{t \wedge \tau_i^r} \eta_1^3 |a^0|^3 d \log(d)^C \, , \\
    \op^{t \wedge \tau_i^r} w_i^t \geq&~ w_i^0 + \eta_1 \sum_{s= 0}^{t-1}  \op^{s \wedge \tau_i^r} \og_i^s + \eta_1 \sum_{s= 0}^{t-1} \op^{s \wedge \tau_i^r} m_i^s -   C (t \wedge \tau_i^r) \op^{t \wedge \tau_i^r} \eta_1^3 |a^0|^3 d \log(d)^C \, .
    \end{aligned}
\end{equation}
Let us introduce the following quantities:
\[
\begin{aligned}
D_i^{t,t'} =  \sum_{s = t}^{t'-1}  \og_i^s\, , \qquad \uD_i^{t,t'} = &~ \sum_{s = t}^{t'-1} \up^{s \wedge \tau_i^r} \og_i^s\, , \qquad \oD_i^{t,t'} = &~ \sum_{s = t}^{t'-1} \op^{s \wedge \tau_i^r} \og_i^s \, ,
\end{aligned}
\]
and 
\[
\begin{aligned}
M_i^{t,t'} =  \sum_{s = t}^{t'-1}  m_i^s\, , \qquad \uM_i^{t,t'} = &~ \sum_{s = t}^{t'-1} \up^{s \wedge \tau_i^r} m_i^s\, , \qquad \oM_i^{t,t'} = &~ \sum_{s = t}^{t'-1} \op^{s \wedge \tau_i^r} m_i^s \, .
\end{aligned}
\]
The term $D^{t,t'}$, which is the sum of population gradients, plays the role of a drift term, while the term $M^{t,t'} $ is a martingale and corresponds to the comparison between the stochastic and the population gradients.

We can choose $C'$ a constant large enough, depending only on $K,D$, such that for
\begin{equation}\label{eq:cond_eta1_1}
\eta_1^2 T \leq \frac{1}{C' |a^0|^2 \log(d)^{C'} d}\, , 
\end{equation}
we have
\[
1 - \frac{1}{\log(d)} \leq (1 - C|a^0|^2 \eta_1^2 d \log(d)^{C} )^T \leq (1 + C|a^0|^2 \eta_1^2 d \log(d)^{C} )^T \leq 1 + \frac{1}{\log(d)} \, . 
\]
In particular, this implies that for any $t < T\wedge \tau^+ \wedge \tau^-$ and $r$ constant sufficiently small,
\begin{equation}\label{eq:opt_upt}
\frac{1}{\op^t} \geq 1 - Cr^2 - \frac{C}{\log(d)}\, , \qquad \frac{1}{\up^t} \leq 1 + Cr^2 + \frac{C}{\log(d)} \, .
\end{equation}
Similarly, we can choose $C'$ a constant large enough, depending only on $K,D$, such that for
\begin{equation}\label{eq:cond_eta1_2}
\eta_1^3 T \leq \frac{1}{C' |a^0|^3 \log(d)^{C'} d^{3/2}}\, , 
\end{equation}
we have
\[
 \sup_{t < T \wedge \tau^+ \wedge \tau^-}  C (t \wedge \tau_i^r) \op^{t \wedge \tau_i^r} \eta_1^3 |a^0|^3 d \log(d)^C \leq \frac{1}{\sqrt{d} \log(d) } \, .
\]

Hence for $\eta_1$ and $T$ satisfying Eqs~\eqref{eq:cond_eta1_1} and \eqref{eq:cond_eta1_2}, we get the following bounds on the trajectory for $t < T \wedge \tau^+ \wedge \tau^- \wedge \tau^0$: for $i \geq P+1$ or $i \in [P]$, $ t\leq \tau_i^r$,
\begin{equation}\label{eq:trajectories_w_bounds}
\begin{aligned}
w_i^t \geq&~ \Big( 1 - Cr^2 - \frac{C}{\log(d)} \Big) \Big[ (1 - \log(d)^{-1}) w_i^0 + \eta_1 \oD^{0,t}_i + \eta_1 \oM^{0,t}_i \Big] \, , \\
w_i^t \leq &~ \Big( 1 + Cr^2 + \frac{C}{\log(d)} \Big) \Big[ (1 + \log(d)^{-1}) w_i^0 + \eta_1 \uD^{0,t}_i + \eta_1 \uM^{0,t}_i  \Big] \, , \\
\end{aligned}
\end{equation}
while for $i \in [P]$ and $t > \tau_i^r$,
\begin{equation}\label{eq:traj_w_P_Delta}
\begin{aligned}
w_i^t =&~ w_i^{\tau_i^r} + \eta_1 D^{\tau_i^r,t}_i + \eta_1 M^{\tau_i^r,t}_i \, .
\end{aligned}
\end{equation}

We prove the following bounds on the martingale part:

\begin{lemma}[Martingale part $M^t$]\label{lem:bound_M}
Assume $\Delta \leq 1$ and $r$ are chosen as in Lemma \ref{lem:proj_step_bounds}.  Fix $T\leq d^D$ and $C_* >0$. There exists a constant $C$ that only depends on $D$, $K$, and $C_*$, such that if we choose 
\begin{equation}\label{eq:cond_eta1_martingale}
\eta_1^2 T \leq \frac{1}{C|a^0|^2 \log(d)^C d }\, ,
\end{equation}
then with probability at least $1 - d^{-C_*}$, we have 
\begin{equation}\label{eq:bounds_martingale_lem}
\max_{t< t' \leq T \wedge \tau^+} \max_{i \in [d]}  \big\{ |\eta_1 M_i^{t,t'}| \vee |\eta_1 \oM_i^{t,t'}| \vee | \eta_1 \uM_i^{t,t'} | \big\} \leq \frac{1}{\sqrt{d} \log(d)}\, .
\end{equation} 
\end{lemma}

\begin{proof}[Proof of Lemma \ref{lem:bound_M}] We will show the theorem for $\max_{0<t\leq T \wedge \tau^+} \uM_i^{0,t}$. The result for $\max_{t<t'\leq T \wedge \tau^+} \uM_i^{t,t'}$ will follow from an union bound on all $t \leq T$, which are also martingales (the proofs for $\oM_i^{t,t'}$ and $M_i^{t,t'}$ will follow by the same argument). 

Denote $\uM^t_i := \uM^{0,t}_i$. We will use a truncation argument. For some $\tC$, define for all $t \geq 1$ and $i \in [d]$,
\[
U_i^t =  \sum_{s = 0}^{t-1} \up^s \Big\{ g_i^t  \ind_{|g_i^t /a^0| < \tC \log(d)^{\tC}} - \E \Big[ g_i^t  \ind_{|g_i^t /a^0 | < \tC \log(d)^{\tC}} \Big]\Big\} \, ,
\]
so that
\[
| U_i^t - \uM_i^t | \leq (1 - Cr^2)^{-P} \sum_{s=0}^{t-1} \Big\{ |g_i^t | \ind_{|g_i^t /a^0| \geq \tC \log(d)^{\tC}} + \E \Big[ |g_i^t|  \ind_{|g_i^t /a^0| \geq \tC \log(d)^{\tC}} \Big]\Big\} \, .
\]
For $t \leq \tau^+$, we have $\| \bw \|_2 \leq \sqrt{P}\Delta + 1$ and we can use Lemma \ref{lem:bounds_gradient} to choose $\tC$ that only depends on $D,K,C_*$ such that
\[
\P \big( \exists t \leq T \wedge \tau^+, \exists i \in [d], | g_i^t/a^0| \geq \tC \log (d)^{\tC} \big) \leq d^{-C_*}/2\, ,
\]
and  for all $t \leq T\wedge \tau^+$ and $i \in [d]$
\[
\eta_1 (1 - Cr^2)^{-P} \E \Big[ |g_i^t|  \ind_{|g_i^t/a^0 | \geq \tC \log(d)^{\tC}} \Big] \leq d^{-D} \cdot \frac{1}{2\sqrt{d \log (d)}} \, .
\]
Hence with probability at least $1 - d^{-C_*}/2$,
\[
\max_{t \leq T \wedge \tau^+} \max_{i \in [d]}  | \eta_1 \uM_i^t|  \leq \frac{1}{2\sqrt{d} \log(d)} + \max_{t \leq T \wedge \tau^+} \max_{i \in [d]}  | \eta_1 U_i^t|   \, .
\] 
Let us now apply Doob's maximal inequality on $U_i^t$: the increments are bounded by $2 |a^0| \tC \log (d)^{\tC}$, hence we have
\[
\P \Big( \max_{1 \leq t \leq T} | \eta_1 U_i^t | \geq \eps \Big) \leq 2 \exp \Big\{  - \frac{\eps^2}{C (\eta_1 a^0 \tC \log(d)^{\tC})^2 T } \Big\} \, .
\]
Choosing $\eps = 1/(2 \sqrt{d} \log (d))$ and $\eta_1 $ as in Eq.~\eqref{eq:cond_eta1_martingale}, as well as a union bound, yields the result.
\end{proof}

\subsection{Proof of Theorem \ref{thm:alignment_one_monomial}}
\label{sec:subection_Proof_prop_one_monomial}

\noindent
\textbf{Step 0: Bounds on the dynamics.} 

We consider the dynamics \eqref{eq:dyn_eq_simplified} up to time $T \wedge \otau$ where $\otau := \tau^0 \wedge \tau^+ \wedge \tau^-$. We assume $T$ and $\eta_1$ satisfy conditions \eqref{eq:cond_eta1_1}, \eqref{eq:cond_eta1_2} and \eqref{eq:cond_eta1_martingale}. In particular, with probability at least $1 - d^{-C_*}$, the dynamics of $w_i^t$ satisfies the bounds in Eqs~\eqref{eq:trajectories_w_bounds} and \eqref{eq:traj_w_P_Delta}, with $M_i^{t,t'}, \oM_i^{t,t'}, \uM_i^{t,t'}$ satisfying the bounds \eqref{eq:bounds_martingale_lem}. In the rest of the proof, we show that on this high probability event, we can choose $\eta_1$ and $T$ such that $T < \otau$ and Theorem \ref{thm:alignment_one_monomial}.(a), (b) and (c) are satisfied.

\noindent
\textbf{Step 1: Controlling the coordinates $i \in [P]$ at the end of the dynamics.}

Let us first show that as soon as $t > \tau^\Delta_i$, then $w_i^t$ stays close to $\Delta$. Note that we can choose $C > 0$ constant large enough independent of $d$, such that if $w_i^t \leq \Delta - |a^0| C \eta_1 \log(d)^C$, then by \eqref{eq:og-on-first-P-bounds}
\begin{equation}\label{eq:positive_g}
\begin{aligned}
\E_{\bx} [ g_i^t ] = \E_{\bx} [ \gamma_i^t \tv_i^t] \geq&~ \E_{\bx} [ \tv_i^t] - \E_\bx [ |\tv_i^t |^2]^{1/2}  \P\big\{| v_i^t | \geq |a^0| C \eta_1 \log(d)^C \big\} \\
\geq&~ C d^{-(D-1)/2} - C d^{-D} > 0  \, .
\end{aligned}
\end{equation}
Hence, for any $\otau > t > \tau_i^\Delta$, consider $t' = \sup \{ t' \leq t : w_i^{t'} \geq \Delta - |a^0| C \eta_1 \log(d)^C \}$ (in particular, $t' \geq \tau^\Delta_i +1 $). From Eq.~\eqref{eq:traj_w_P_Delta} and by Lemma \ref{lem:bound_M}, we have 
\[
w_i^{t} = w_i^{t' + 1} + \eta_1 D_i^{t' +1,t } + \eta_1 M_i^{t' +1,t} \geq \Delta - |a^0| C \eta_1 \log(d)^C - \frac{1}{\sqrt{d\log d}} \geq \Delta - \frac{2}{\sqrt{d\log(d)}} \, ,
\]
where we used that $w_i^s < \Delta - |a^0| C \eta_1 \log(d)^C$ for $t'+1 \leq s < t$ and therefore by Eq.~\eqref{eq:positive_g}, we have $D_i^{t' +1,t } \geq 0$. We deduce that
\begin{equation}\label{eq:lb_w_aafter_Delta}
\inf_{\tau_i^\Delta < t < \otau \wedge T} w_i^t \geq \Delta - \frac{2}{\sqrt{d \log (d)}} \, .
\end{equation}

Similarly, we show that for any $t \geq \tau^r_i + 1$, we have $w_i^t \geq r/2$. Indeed, for any $\tau^r_i + 1 \leq t \leq \tau_i^\Delta \wedge \otau \wedge T$, we have
\[
w_i^t = w_i^{\tau_i^r +1 } + \eta_1 D_i^{\tau_i^r +1 , t} + \eta_1 M_i^{\tau_i^r+1, t} \geq r - \frac{1}{\sqrt{d \log(d)}}\, , 
\]
where we used that $w_i^{\tau_i^r +1 } \geq r$ by definition of $\tau_i^r$, and $\og_i^t \geq 0 $ for all $s \leq \tau_i^\Delta \wedge \otau \wedge T$. We deduce that
\begin{equation}\label{eq:lb_w_aafter_r}
    \inf_{\otau \wedge T >t > \tau_i^r} w_i^t \geq r - \frac{1}{\sqrt{d\log(d)}} \geq \frac{r}{2} \, .
\end{equation}

\noindent
\textbf{Step 2: Bounding the growth of $w_i^t$ for $i \in [P]$.}

Define $\alpha_t = \min \{ w_i^t : i \in S_t \cap [P]\} $ (i.e., the minimum of $w_i^t$ that have $\tau_i^r \geq t$). Note that $w_i^t \geq r/2$ for $i\not\in S_t$ by Eq.~\eqref{eq:lb_w_aafter_r} and therefore $w_i^t \geq \alpha_t /2$. By Eq.~\eqref{eq:opt_upt} and Lemma \ref{lem:gradient_formula}, we have for $s \leq \tau_i^r \wedge \otau \wedge T$,
\[
\up^s \og_i^s \geq C a^0 \frac{\chi_* (\bw^t)}{w_i^t} \mu_D (\sigma) - |a^0| \tilde{\kappa} \geq C a^0 \mu_D (\sigma) \alpha_t^{D-1} - |a^0| \tilde{\kappa} \, .
\]
Combining this lower bound with Eq.~\eqref{eq:trajectories_w_bounds} and the bound on the martingale in Lemma~\ref{lem:bound_M}, we get that for all $t \leq \tau_i^r \wedge \otau \wedge T$
\begin{equation}\label{eq:lower_bound_alpha}
\begin{aligned}
\alpha_t &\geq \frac{1 - Cr^2}{\sqrt{d}} - \eta_1 t|a^0|\tilde{\kappa} + C \eta_1 a^0\mu_D (\sigma) \sum_{s= 0}^{t-1} \alpha_t^{D - 1} \\
&\geq \frac{1}{2\sqrt{d}} + C \eta_1 a^0\mu_D (\sigma) \sum_{s= 0}^{t-1} \alpha_t^{D - 1}\,,
\end{aligned}
\end{equation}
where we have used that $r$ is sufficiently small, and $\eta_1 t |a^0| \tilde{\kappa} \leq \tilde{\kappa} / (C' \eta_1 |a^0| \log(d)^{C'} d) \leq 1/(4\sqrt{d})$ for all $t \leq T$ since we take $\kappa$ sufficiently small. We can use the bound on this sequence derived in Lemma \ref{lem:bound_sequences}: for $D> 2$,
\[
\alpha_t \geq r \wedge \left\{ \frac{1}{\big( (4d)^{D/2 - 1} - C \eta_1 a^0\mu_D (\sigma) t \big)^{\frac{1}{k-2}} } \right\} \, ,
\]
and we deduce that we must have
\begin{equation*}
\tau^r \wedge \otau \leq \frac{C d^{D/2 -1}}{\eta_1 a^0 \mu_D (\sigma)} \, .
\end{equation*}
For $D= 2$, we use that 
\[
\alpha_t \geq \frac{1}{2\sqrt{d}} \big(1 + C \eta_1 a^0 \mu_D (\sigma) \big)^t \, ,
\]
and deduce that
\begin{equation*}
\tau^r \wedge \otau \leq \frac{C \log (d) }{\eta_1 a^0 \mu_D (\sigma)} \, .
\end{equation*}
Similarly, consider $\alpha_t = \min \{ w_i^t : i \in [P], \tau^{\Delta}_i \geq t \} $ for all $\tau^r +1 \leq t \leq \tau^\Delta \wedge \otau$. By Eq.~\eqref{eq:lb_w_aafter_Delta}, we have $w_i^t \geq \Delta/2$ for $t > \tau_i^\Delta$. Hence by Eq.~\eqref{eq:traj_w_P_Delta}, we get 
\[
\alpha_t \geq \frac{r}{2} + C \eta_1 a^0\mu_D (\sigma) \sum_{s= 0}^{t-1} \alpha_t^{D - 1} \, ,
\]
and we deduce that if $\tau^{\Delta} < \otau$ then
\begin{equation*}
\tau^{\Delta} - \tau^r \leq \frac{C}{\eta_1 a^0 \mu_D (\sigma)}\,  .
\end{equation*}
Combining the above bounds, we deduce that
\begin{equation}\label{eq:cond_T_1}
    \begin{aligned}
    &\text{for $D = 2$:} \qquad \tau^\Delta \wedge \otau \leq \frac{C \log (d) }{\eta_1 a^0 \mu_D  (\sigma)} \, , \\
    &\text{for $D > 2$:} \qquad \tau^\Delta \wedge \otau \leq \frac{C d^{D/2 -1} }{\eta_1 a^0 \mu_D  (\sigma)} \, , \\
    \end{aligned}
\end{equation}

On the other hand, consider $\beta_t = \max \{ w_i^t : i \in [P] \}$ and let us lower bound the time $\ot = \inf \{ t : \beta_{t} \geq 2/\sqrt{d} \}$. For $t \leq \otau \wedge \ot $, we have $S_t = [d]$ and $\frac{1}{\up^t} \leq 1 + \frac{C}{\log(d)}$. Hence by Eq.~\eqref{eq:trajectories_w_bounds},
\[
\beta_t \leq (1 + C/\sqrt{\log(d)} ) \frac{2}{\sqrt{d}} + C \eta_1 a^0 \mu_D (\sigma) \sum_{s = 0}^{t-1} \beta_s^{D-1}\, ,
\]
and therefore by Lemma \ref{lem:bound_sequences}, we get for $D>2$,
\begin{equation}\label{eq:bound_beta_D}
\beta_t \leq \frac{1}{\big( (d/2(1+C/\sqrt{\log d}))^{D/2-1} - C \eta_1 a^0 \mu_D (\sigma) t \Big)^{ \frac{1}{D-2}}}\, .
\end{equation}
We deduce that
\begin{equation}\label{eq:lower_bound_beta_1}
 \ot \geq \otau \wedge \frac{ d^{D/2 -1}}{C\eta_1 a^0 \mu_D (\sigma)} \, .
\end{equation}
For $D =2$, we have by Lemma \ref{lem:bound_sequences},
\begin{equation}\label{eq:bound_beta_D2}
\beta_t \leq \frac{3}{\sqrt{d}} \big(1 + C \eta_1 a^0 \mu_D(\sigma) \big)^t\, ,
\end{equation}
and therefore 
\begin{equation}\label{eq:lower_bound_beta_2}
 \ot \geq \otau \wedge \frac{ 1 }{C\eta_1 a^0 \mu_D (\sigma)} \, .
\end{equation}

\noindent
\textbf{Step 3: Bounding the coordinates $P+1 \leq i \leq d$.}

From Eq.~\eqref{eq:trajectories_w_bounds} and Lemma \ref{lem:gradient_formula}, we have for all $i \geq P+1$ and $t < \otau \wedge T$,
\[
\begin{aligned}
w_i^t \geq&~ (1 - Cr^2) \frac{1}{\sqrt{d}} - C |\eta_1  \oD_i^{0,t \wedge ( \tau^r +1 )}| - Ct \eta_1 |a_i^0| \tilde\kappa \, ,\\
w_i^t \leq &~ (1 + Cr^2) \frac{1}{\sqrt{d}}  + C |\eta_1 \uD_i^{0,t \wedge (\tau^r +1) }| + Ct \eta_1 |a_i^0| \tilde\kappa \, .
\end{aligned}
\]
We have
\begin{equation}\label{eq:bound_drift_P_d}
|\eta_1  \oD_i^{0,t \wedge ( \tau^r +1 )}| \vee |\eta_1  \uD_i^{0,t \wedge ( \tau^r +1 )}| \leq   C \eta_1 a^0 \mu_D (\sigma) \sum_{s=0}^{t \wedge ( \tau^r +1 ) - 1} w_i^s \chi_* (\bw^s ) + C t \eta_1 |a_i^0| \tilde\kappa\, .
\end{equation}
Consider $j \in [P]$ such that $\tau^r_j = \tau^r$. Then by Eq.~\eqref{eq:trajectories_w_bounds}, we have, for any $t < (\tau^r+1) \wedge \otau \wedge T$, that
\[
2r \geq w^{t + 1 }_j - (1 -Cr^2) \frac{1}{\sqrt{d}} \geq C \eta_1 a^0 \mu_D(\sigma) \sum_{s = 0}^{t \wedge \tau_j^r} \frac{\chi_* (\bw^s) }{w_j^s}\, .
\]
Hence we deduce that 
\begin{equation}\label{eq:bound_drift_P_d_2}
\eta_1 a^0 \mu_D(\sigma) \sum_{s = 0}^{t \wedge \tau^r_j} \frac{\chi_* (\bw^s) }{w_j^s} \leq Cr \, .
\end{equation}
Using that $w_j^t\leq Cr$ for $t \leq \tau_j^r$ and $w_i^s \leq 3 /(2 \sqrt{d})$ for $s < \otau$ in Eq.~\eqref{eq:bound_drift_P_d}, we get that, for any $t < \otau \wedge T$,
\[
\begin{aligned}
|\eta_1  \oD_i^{0,t \wedge ( \tau^r +1 )}| \vee |\eta_1  \uD_i^{0,t \wedge ( \tau^r +1 )}| \leq&~  Ct\eta_1|a^0|\tilde{\kappa} + \frac{C}{\sqrt{d}} r \eta_1 a^0 \mu_D (\sigma) \sum_{s=0}^{t \wedge ( \tau^r +1 ) - 1} \frac{ \chi_* (\bw^s ) }{w_j^s} \\
\leq&~  \frac{Cr^2}{\sqrt{d}} \, . 
\end{aligned}
\]
We deduce that for all $P+1 \leq i \leq d$ and $t < \otau$ and $t+1 < \tau^0$
\begin{equation}\label{eq:bound_w_P_d}
    \frac{1 - Cr^2}{\sqrt{d} } \leq w_i^{t+1} \leq \frac{1 + Cr^2}{\sqrt{d}} \, ,
\end{equation}
and therefore taking $r$ sufficiently small, $\tau^+ \wedge \tau^- \geq \tau^0$.

\noindent
\textbf{Step 4: Proof of Theorem \ref{thm:alignment_one_monomial}.(b) and (c).}

Choose $\oT_1 := T$ and $\eta_1 $ that satisfy Eqs.~\eqref{eq:cond_eta1_1}, \eqref{eq:cond_eta1_2}, \eqref{eq:cond_eta1_martingale} and \eqref{eq:cond_T_1}. We have with probability at least $1 - C d^{-C_*}$ by Lemma \ref{lem:bound_on_tau_0} and Lemma \ref{lem:bound_M} that $\tau^0 > \tau^+ \wedge \tau^- \wedge T$. And Eqs.~\eqref{eq:bound_w_P_d} and \eqref{eq:lower_bound_alpha} imply that $\tau^+ \wedge \tau^- \wedge \tau^0 > \oT_1$. Furthermore, by Eq.~\eqref{eq:cond_T_1}, we have $\tau^{\Delta} < \oT_1$ which implies Theorem \ref{thm:alignment_one_monomial}.(b) by Eq.~\eqref{eq:lb_w_aafter_Delta}. Theorem \ref{thm:alignment_one_monomial}.(c) follows from Eq.~\eqref{eq:bound_w_P_d}.

\noindent
\textbf{Step 5: Upper bound for all neurons with early stopping.}

Theorem \ref{thm:alignment_one_monomial}.(a) follows from Eqs.~\eqref{eq:lower_bound_beta_1} and \eqref{eq:lower_bound_beta_2} for neurons with initialization satisfying 
\[
a_j^0 \mu_D (\sigma) (w_{j,1}^0)^{k_1} \cdots (w_{j,P}^0)^{k_P} > 0 \, .
\]
For neurons that do not satisfy this condition, the analysis in Section \ref{sec:bounding_contributions_dynamics} still holds and we get bounds on the dynamics similar to the ones in Eq.~\eqref{eq:trajectories_w_bounds}, with the difference that $a^0 \mu_D (\sigma) < 0$, and therefore the drift has a negative contribution to the dynamics, and $\tau^+, \tau^-$ are now defined on all coordinates instead of only $i = P+1 , \ldots , d$.

We can upper bound the drift contribution using that $\beta_t = \max_{i \in [P]} w_i^t$ satisfy for $t < \otau$
\begin{equation}\label{eq:beta_t_dynamics}
\beta_t  \leq (1 + C/\sqrt{\log(d)} ) \frac{1}{\sqrt{d}} + C \eta_1 |a^0 \mu_D (\sigma)| \sum_{s = 0}^{t-1} \beta_s^{D-1}\, ,
\end{equation}
and therefore, taking the same bounds \eqref{eq:bound_beta_D} and \eqref{eq:bound_beta_D2}, we get for $t \leq \otau \wedge \oT_1/ (C \log (d)^C)$ for $C$ a constant sufficiently large that
\begin{equation}\label{eq:bound_beta_t}
\beta^t \leq \frac{1}{\sqrt{d}} (1 + C'/\sqrt{\log(d)} ) \, .
\end{equation}
Furthermore, denoting $\alpha_t = \min_{i \in [P]} w_i^t$, we have
\[
\alpha_t \geq  (1 - C/\sqrt{\log(d)} ) \frac{1}{\sqrt{d}} - C \eta_1 |a^0 \mu_D (\sigma)| \sum_{s = 0}^{t-1} \beta_s^{D-1} \, .
\]
Using the analysis of Lemma \ref{lem:bound_sequences}, we get that the drift has the same upper bound as $\beta_t$ in Eq.~\eqref{eq:bound_beta_t},
\[
(1 + C/\sqrt{\log(d)} ) \frac{1}{\sqrt{d}} + C \eta_1 |a^0 \mu_D (\sigma)| \sum_{s = 0}^{t-1} \beta_s^{D-1}  \leq \frac{1}{\sqrt{d}} (1 + C'/\sqrt{\log(d)} )\, 
\]
for $t \leq \otau \wedge \oT_1/ (C \log (d)^C)$ and therefore
\[
\alpha_t \geq \frac{1}{\sqrt{d}} (1 - C/\sqrt{\log(d)} )\, .
\]
The bound on coordinates $i > P$ follows similarly to step 3. In particular, we deduce that for $t < \otau \wedge \oT_1/ (C \log (d)^C)$ and $t+1 < \tau^0$, we have for all $i \in [d]$
\[
(1 - C/\sqrt{\log(d)} ) \frac{1}{\sqrt{d}}  \leq w_i^{t+1} \leq (1 + C/\sqrt{\log(d)} ) \frac{1}{\sqrt{d}} \, , 
\]
and therefore $\tau^+ \wedge \tau^- > \tau^0 \wedge \oT_1/ (C \log (d)^C) $, which concludes the proof of Theorem \ref{thm:alignment_one_monomial}.(a).

\subsection{Technical lemmas}

\begin{lemma}[Tail bounds on functions of Gaussians]\label{lem:bounds_gradient}
Assume that $\| \bw^t \|_2 \leq 1 + \sqrt{D}$. Then there exist constants $c,C$ that only depend on $D$ and $K$ such that
\begin{equation}\label{eq:tail_bounds_fct_gaussian}
\begin{aligned}
&\P(|y^t | \geq z + c) \leq \exp\big\{-Cz^{2/D}\big\}, \\
&\P_{\bw^t} \left( |y^t \sigma ( \< \bw^t , \bx^t \>) | \geq z + c \right) \leq \exp \big\{ - C z^{2/D} \big\} \, , \\
&\P_{\bw^t} \big( |g_i^t/a^0| \geq z + c\big) \leq \exp \big\{ - C z^{2/(D+1)} \big\}  \, , \\ 
&\P_{\bw^t} \big( \| \bg^t /a^0\|_2 \geq \sqrt{d} (z + c )   \big)\leq  \exp \big\{ - C z^{2/(D+1)} \big\}  \, , \\
& \P_{\bw^t} \big( \| \bg^t /a^0\|_2^2 \cdot | g_i^t /a^0 | \geq d (z + c )  \big) \leq   \exp \big\{ - C z^{2/(3D+3)} \big\}  \, ,
\end{aligned}
\end{equation}
where $\P_{\bw^t} (\cdot) := \P (\cdot | \bw^t)$ denotes the conditional probability. Furthermore, for any $q \geq 2$, there exists a constant $C_q$ that only depends on $q,D,K$ such that
\begin{equation}
    \E_{\bw^t} \Big[ \big\Vert \bg^t/a^0 \big\Vert_q^q \Big]^{1/q} \leq C_q \sqrt{d} \, ,
\end{equation}
where $\E_{\bw^t} [\cdot] := \E [ \cdot | \bw^t ]$ denotes the conditional expectation.
\end{lemma}

\begin{proof}[Proof of Lemma \ref{lem:bounds_gradient}]
Recall that for polynomials $f:\R^d \to \R$ of degree $D$ on Gaussian variables, we have the hypercontractivity inequality $\| f \|_{L^q} \leq (q - 1)^{D/2} \| f \|_{L^2}$ for any $q \geq 2$. Hence, there exist constants $C, C'$ that only depend on $D$ such that
\begin{equation}\label{eq:tail_poly_gaussian}
\P \big( | f(\bx) - \E_{\bx} [f(\bx)] | \geq z\sqrt{ \text{Var}_\bx (f)} \big) \leq C' \exp \big\{ -C z^{2/D} \big\}  \, .
\end{equation}
Recall that $y = \He_{\bk} (\bx) + \eps$, where $\He_{\bk}$ is a degree-$D$ multivariate Hermite polynomial and $\eps$ is $K$-subgaussian. Further, recall that we assumed $\| \sigma \|_\infty, \| \sigma ' \|_\infty \leq K$. Applying Eq.~\eqref{eq:tail_poly_gaussian}, there exist constants $C,c$ that only depend on $D$ and $K$ such that for any $\bw \in \R^d$,
\[ 
\P ( |y | \geq z + c) \leq \exp \big\{ - Cz^{2/D} \big\} \, , \qquad \P ( |y \sigma ( \< \bw , \bx \>) | \geq z + c ) \leq \exp \big\{ - Cz^{2/D} \big\} \, ,
\]
where we used that $\E[ |y \sigma ( \< \bw , \bx \>) | ] \leq K \E [ |y|^2]^{1/2} \leq c $. Following a similar reasoning, we get for any $\| \bw \|_2^2 \leq 1 + \sqrt{D}$ and $i \in [d]$, 
\[
\begin{aligned}
&\P ( | y x_i \sigma ' (\< \bw , \bx \>)| \geq z + c) \leq \exp \big\{ - C z^{2/(D+1)} \big\} \, , \\ &\P ( | y w_i \< \bw , \bx \> \sigma ' (\< \bw , \bx \>)| \geq z + c ) \leq \exp \big\{ - C z^{2/(D+1)} \big\} \, .
\end{aligned}
\]
Recall that $|g_i^t /a^0| \leq |\gamma_i^t | \big(| y x_i \sigma ' (\< \bw , \bx \>)| +  | y w_i \< \bw , \bx \> \sigma ' (\< \bw , \bx \>)| \big)$. Conditioning on $\bw^t$ and assuming that $\| \bw^t \|_2 \leq 1 + \sqrt{D}$, we obtain
\[
\begin{aligned}
&\P_{\bw^t} \big( |g_i^t /a^0| \geq z + c\big) \leq \exp \big\{ - C z^{2/(D+1)} \big\} \, , \\ 
&\P_{\bw^t} \big( \| \bg^t  /a^0\|_2 \geq \sqrt{d} (z + c )   \big)\leq \P_{\bw^t} \big( \| \bg^t  /a^0\|_2^2 \geq d (z^2 + c )   \big) \exp \big\{ - C z^{2/(D+1)} \big\} \, , \\
& \P_{\bw^t} \big( \| \bg^t /a^0\|_2^2 | g_i^t /a^0| \geq d (z + c )  \big) \leq d \exp \big\{ - C z^{2/(3D+3)} \big\} \, .
\end{aligned}
\]
Furthermore, again assuming that $\| \bw^t \|_2 \leq 1 + \sqrt{D}$, we have for any $q \geq 2$,
\[
\E_{\bw^t} \big[ \| \bg^t/a^0 \|_{q}^q  \big]^{1/q} \leq K \E_{\bw^t} \big[ \| y^t  \bx^t  \|_{q}^q  \big]^{1/q}  + K \E_{\bw^t} \big[ \| y^t  \bw^t \< \bw^t , \bx^t \>  \|_{q}^q \big]^{1/q} \leq C_q \sqrt{d} \, ,
\]
which concludes the proof.
\end{proof}

The following lemma provides simple upper and lower bounds on sequences satisfying some geometric bound on their evolution. The upper bound can be seen as a discrete version of Bihari–LaSalle inequality. This upper bound was proven in \cite[Appendix C]{arous2021online}, and we modify their proof to obtain a lower bound.

\begin{lemma}[Bounds on sequences]\label{lem:bound_sequences}
Let $k\geq 2$ be an integer and $a_0,a_1,b_0,b_1>0$ be four positive constants with $a_0 \leq b_0$ and $a_1 \leq b_1$. Consider a sequence $(u_t)_{t \in \naturals}$ that satisfy for all $t \in \naturals$,
\[
\begin{aligned}
a_0 + a_1 \sum_{s = 0}^{t-1} u_s^{k-1} \leq u_t \leq b_0 + b_1 \sum_{s = 0}^{t-1} u_s^{k-1} \, .
\end{aligned}
\]
If $k = 2$, then for any $t \in \naturals$,
\[
a_0 (1 +a_1)^t \leq u_t \leq b_0 (1 + b_1)^t \, .
\]
If $k > 2$, then for any $\Delta > 0$ and any $t \in \naturals$,
\[
\Delta \wedge \left\{ \frac{1}{\left(a_0^{-(k-2)} - \frac{k-2}{(1 +a_1 \Delta^{k-2} )^{k-1}} a_1 t \right)^{\frac{1}{k-2}}} \right\}  \leq u_t \leq \frac{1}{\left(b_0^{-(k-2)} -  (k-2) b_1  t \right)^{\frac{1}{k-2}}} \, .
\]
\end{lemma}

\begin{proof}[Proof of Lemma \ref{lem:bound_sequences}] Note that by induction, we have $w_t \leq u_t \leq v_t$ for any $t \geq 0$ where
\[
v_t = b_0 + b_1 \sum_{s = 0}^{t-1} v_s^{k-1} \, , \qquad w_t = a_0 + a_1 \sum_{s = 0}^{t-1} w_s^{k-1} \, .
\]
For $k = 2$, it is straightforward to get $v_t = v_{t-1} (1 + b_1) = b_0 (1 + b_1)^t$ and $w_t = a_0 (1 + a_1 )^t$. 

For $k > 2$, we consider the upper bound on $v_t$. First, notice that
\[
b_1 = \frac{v_t - v_{t-1}}{v_{t-1}^{k-1}} \geq \int_{v_{t-1}}^{v_t} \frac{1}{x^{k-1} } \de x = \frac{1}{k-2} \left[ \frac{1}{v_{t-1}^{k-2}} - \frac{1}{v_t^{k-2}} \right] \, .
\]
Hence, rearranging the terms, we get for any $t$,
\[
\frac{1}{v_t^{k-2}} \geq \frac{1}{v_{t-1}^{k-2}} - (k-2) b_1 \geq  \frac{1}{v_{0}^{k-2}}  - (k-2)b_1 t \, .
\]
We deduce that
\[
v_t \leq \frac{1}{\left(b_0^{-(k-2)} -  (k-2) b_1  t \right)^{\frac{1}{k-2}}}\, .
\]
Let us now lower bound $w_t$. We have
\[
\begin{aligned}
a_1 = \frac{w_t - w_{t-1}}{w_{t-1}^{k-1}} =&~  \int_{w_{t-1}}^{w_t} \frac{1}{x^{k-1} } \de x + \int_{w_{t-1}}^{w_t} \frac{x^{k-1} - w_{t-1}^{k-1}}{w_{t-1}^{k-1} x^{k-1} } \de x   \\
\leq&~   \frac{w_{t}^{k-1} }{ w_{t-1}^{k-1}} \int_{w_{t-1}}^{w_t} \frac{1}{x^{k-1} } \de x  \\
=&~  \frac{(1 + a_1 w_{t-1}^{k-2} )^{k-1} }{k-2} \left[ \frac{1}{w_{t-1}^{k-2}} - \frac{1}{w_t^{k-2}} \right] \, .
\end{aligned}
\]
Hence, as long as $w_t \leq \Delta$, we get
\[
\frac{1}{w_t^{k-2}} \leq \frac{1}{w_{t-1}^{k-2}} - \frac{k-2}{(1 + a_1 \Delta^{k-2} )^{k-1}}  a_1 \leq  \frac{1}{v_{0}^{k-2}}  - \frac{k-2}{(1 + a_1 \Delta^{k-2} )^{k-1}} a_1 t \, ,
\]
and therefore
\[
w_t \geq \frac{1}{\left(a_0^{-(k-2)} - \frac{k-2}{(1 +a_1 \Delta^{k-2} )^{k-1}} a_1 t \right)^{\frac{1}{k-2}}} \, .
\]
\end{proof}

\clearpage

\section{Proof of Theorem \ref{thm:sequential_alignment}: sequential alignment to the support}
\label{sec:proof_alignment_several_monomials}

In this appendix, we consider the sequential alignment to the support in Section \ref{sec:multimonomials}. The proofs will follow from a similar argument as in the single monomial case. However, the dynamics will be now split in $L$ phases corresponding to the alignment to each of the $L$ monomials. 

Recall that throughout the proofs, we will denote for simplicity $C,c>0$ generic constants that only depend on $D$ and $K$ (note that all the other constants $P_l,D_j,\oD_l \leq D$). The values of these constants are allowed to change from line to line or within the same line.

\subsection{Proof of Theorem \ref{thm:sequential_alignment}: alignment to the full support}\label{app:proof_theorem_2}

We will use notations and results from Appendix \ref{app:proof_alignment_one_monomial} and outline the main difference with the proof of Theorem \ref{thm:alignment_one_monomial}. We can again reduce the problem to tracking one neuron, and we assume without loss of generality that $w_1^0 = \ldots = w_d^0 = 1/\sqrt{d}$ and $a^0 \mu_{\oD_l} (\sigma ) >0$ for all $l \in [L]$. 

Let us introduce the following new stopping times on the dynamics: for $l\in [L]$,
\[
\tau^{r,l} = \sup_{i \in [P_l]} \tau_i^r \, , \qquad \qquad  \tau^{\Delta,l} = \sup_{i \in [P_l]} \tau_i^{\Delta} \, .
\]
The population gradients for $t \leq \tau_i^\Delta \wedge \tau_0$ are now given by:

\begin{lemma}\label{lem:gradient_formula_multi} Denote $\chi_{*,l} (\bw^t) = \prod_{j \in [P_l]} (w_j^t )^{k_j} $ for $j \in [L]$. For $i \in [P_l] \setminus [P_{l-1}]$ and $t < \tau_i^\Delta$, the population gradient is given by: if $t \leq \tau_i^r$ (i.e., $i \in S_t$),
\begin{equation}\label{eq:grad_formula_multiP_St}
\begin{aligned}
\og_i^t =&~ a^0 \sum_{l' \geq l} \frac{\chi_{*,l'} (\bw^t)}{w_i^t} \left( k_i - (w_i^t)^2 \sum_{s \in S_t \cap [P_{l'}]} k_s \right) \E_{G} \Big[ \sigma^{(\oD_{l'})} (\| \bw^t \|_2 G ) \Big] \\
&~ - a^0 w_i^t \sum_{l' < l} \chi_{*,l'} (\bw^t) \left( \sum_{s \in S_t \cap [P_{l'}]} k_s \right)\E_{G} \Big[ \sigma^{(\oD_{l'})} (\| \bw^t \|_2 G ) \Big] + O(|a^0|\tilde{\kappa})\, ,
\end{aligned}
\end{equation}
while if $t > \tau_i^r$ (i.e., $i \not\in S_t$)
\begin{equation}\label{eq:grad_formula_multiP_notSt}
\begin{aligned}
\og_i^t =&~  a^0 \sum_{l' \geq l} \frac{\chi_{*,l'} (\bw^t)}{w_i^t} \left( k_i  \E_{G} \Big[ \sigma^{(\oD_{l'})} (\| \bw^t \|_2 G ) \Big] + (w_i^t)^2 \E_{G} \Big[ \sigma^{(\oD_{l'}+2)} (\| \bw^t \|_2 G ) \Big] \right) \\
&~ + a^0 \sum_{l' < l} w_i^t \chi_{*,l'} (\bw^t) \E_{G} \Big[ \sigma^{(\oD_{l'}+2)} (\| \bw^t \|_2 G ) \Big] + O(|a^0|\tilde{\kappa}) \, .
\end{aligned}
\end{equation}
For $i > P$ and $t < \tau^+ \wedge \tau_0$,
\begin{equation}\label{eq:grad_formula_notmultiP}
\og_i^t = -a^0  \sum_{j \in [L]} w_i^t \chi_* (\bw^t) \left( \sum_{s \in S_t \cap [P_{l}]} k_s \right)\E_{G} \Big[ \sigma^{(\oD_{l})} (\| \bw^t \|_2 G ) \Big] + O(|a^0|\tilde{\kappa}) \, .
\end{equation}
\end{lemma}

\begin{proof}[Proof of Lemma \ref{lem:gradient_formula_multi}]
    The proof follows from Lemma \ref{lem:gradient_formula} applied to a sum of monomials.
\end{proof}

Again, by Assumption \ref{ass:sigma_I}, we can choose $\Delta$ small enough and depending only on $K$ and $\oD$ such that Eqs~\eqref{eq:choice_Delta_sigma} and \eqref{eq:choice_Delta_sigma_2} are satisfied (with $D$ replaced by $\oD$). We can further chose $\Delta$ small enough (only depending on $K$ and $\oD$ such that there exists constants $C,c$ such that for all $t < \tau_i^\Delta \wedge \tau^+ \wedge \tau^-$, if $i \in [P_l] \setminus [P_{l-1}]$, then if $t > \tau^{r,l-1}$, 
\[
0 < c a^0 \frac{\chi_{*,l} (\bw^t)}{w_i^t}  \leq \og_i^t \leq C a^0 \frac{\chi_{*,l} (\bw^t)}{w_i^t}  \, ,
\]
and if $\tau^{r,l' - 1} < t \leq \tau^{r,l'}$ for $l' \leq l -1$, 
\[
c a^0 \frac{\chi_{*,l} (\bw^t)}{w_i^t}  - C a^0 w_i^t \chi_{*,l'} (\bw^t) \leq \og_i^t \leq C a^0 \frac{\chi_{*,l} (\bw^t)}{w_i^t} (\sigma) - c a^0 w_i^t \chi_{*,l'} (\bw^t)  \, ,
\]
while for $i \geq P+1$ and $\tau^{r,l - 1} < t \leq \tau^{r,l}$,
\[
| \og_i^t | \leq C a^0 w_i^t \chi_{*,l} (\bw^t) \, ,
\]
and $\og_i^t = 0$ for $t > \tau^{r,L}$.

\begin{proof}[Proof of Theorem \ref{thm:sequential_alignment}]
\textbf{Step 0: Bounds on the dynamics.}

We consider the dynamics up to time $T \wedge \otau$ where $\otau := \tau^0 \wedge \tau^+ \wedge \tau^-$. We again assume that $T$ and $\eta_1$ satisfy conditions \eqref{eq:cond_eta1_1}, \eqref{eq:cond_eta1_2} and \eqref{eq:cond_eta1_martingale}, so that the dynamics of $w_i^t$ satisfies the bounds in Eqs~\eqref{eq:trajectories_w_bounds} and \eqref{eq:traj_w_P_Delta}, with $M_i^{t,t'}, \oM_i^{t,t'}, \uM_i^{t,t'}$ satisfying the bounds \eqref{eq:bounds_martingale_lem}, with probability at least $1 - d^{-C_*}$. The following steps will follow closely the proof of Theorem~\ref{thm:alignment_one_monomial}.

\noindent
\textbf{Step 1: Controlling the coordinates $i\in[P_l]\setminus [P_{l-1}]$ during the first $l-1$ phases.}

Note that for $i\in [P_l] \setminus [P_{l-1}]$, during the $l' \leq l-1$ phase, we have for $\tau^{r,l'-1}+1 < t \leq (\tau^{r,l'} +1)\wedge \tau^{\Delta}_i$,
\begin{equation}\label{eq:coordinates_previous_phases}
\begin{aligned}
w_i^t \geq&~ (1 -Cr^2) w^{\tau^{r,l'-1}+1}_i + c \eta_1 \oD^{\tau^{r,l'-1}+1, t} \\
\geq&~ (1 -Cr^2) w^{\tau^{r,l'-1}+1}_i + a^0 \eta_1 \sum_{s = \tau^{r,l'-1}+1}^{t-1} \Big[c \frac{\chi_{*,l} (\bw^s)}{w_i^t} - C w_i^t \chi_{*,l'} (\bw^s)  \Big] \, , \\
w_i^t \leq &~ (1 + Cr^2) w^{\tau^{r,l'-1}+1}_i + C \eta_1 \uD^{\tau^{r,l'-1}+1, t} \\
\geq &~ (1 +Cr^2) w^{\tau^{r,l'-1}+1}_i + a^0 \eta_1 \sum_{s = \tau^{r,l'-1}+1}^{t-1} \Big[C \frac{\chi_{*,l} (\bw^s)}{w_i^t} - c w_i^t \chi_{*,l'} (\bw^s)  \Big] \, .
 \end{aligned}
\end{equation}
Assume that $\min_{i \in [P_l] \setminus [P_{l-1}]} \tau_i^\Delta > \tau^{r,l'-1}+1$ and
\[
\max_{i \in [P_l] \setminus [P_{l-1}]} | w^{\tau^{r,l'-1}+1}_i - 1/\sqrt{d} | \leq \frac{Cr^2}{\sqrt{d} } \, .
\]
Denote 
\[
\tau^{+,l} = \min \{ t \geq \tau^{r,l'-1}+1 : \max_{i \in [P_j] \setminus [P_{l-1}]} |w_i^t - 1/\sqrt{d} | > 1/(2 \sqrt{d} ) \}  \, .
\]
As long as $t < \tau^{+,l}$, then
\[
\frac{1 - Cr^2}{\sqrt{d}} - C\frac{a^0 \eta_1}{\sqrt{d}} \sum_{s = \tau^{r,l'-1}+1}^{t-1} \chi_{*,l'} (\bw^s) \leq w_i^t \leq \frac{1 + Cr^2}{\sqrt{d}} +C \frac{a^0 \eta_1}{\sqrt{d}} \sum_{s = \tau^{r,l'-1}+1}^{t-1} \chi_{*,l'} (\bw^s) \, ,
\]
where we used the assumption that $D_{l'} \geq 2$.
Using the same argument as in Step 3 of Section \ref{sec:subection_Proof_prop_one_monomial}, we deduce that as long as $t \leq (\tau^{r,l'} +1) \wedge \tau^{+,l}$, then 
\[
\frac{1 - Cr^2}{\sqrt{d}}  \leq w_i^t \leq \frac{1 + Cr^2}{\sqrt{d}}\, .
\]
In particular, we deduce that we must have $\tau^{+,l} > \tau^{r,l'} +1$ and therefore $\tau^{\Delta}_i > \tau^{r,l'} +1$ for all $i \in [P_l] \setminus [P_{l-1}]$.

By induction, we deduce that for all $i \in [P_l] \setminus [P_{l-1}]$, $\tau^{\Delta}_i > \tau^{r,l-1} +1$ and
\[
\frac{1 - Cr^2}{\sqrt{d}}  \leq w_i^{\tau^{r,l-1} +1} \leq \frac{1 + Cr^2}{\sqrt{d}}\, .
\]

\noindent
\textbf{Step 2: Bounding the growth of $w_i^t$ for $i \in [P_l]\setminus [P_{l-1}]$.}

The same argument as in Step 1 of Section \ref{sec:subection_Proof_prop_one_monomial} (recalling that by the previous argument, $\tau^{r}_i \wedge \tau^{\Delta}_i > \tau^{r,l-1} +1$ for all $i \in [P_l]\setminus [P_{l-1}]$) yields
\[
\inf_{\otau \wedge T >t > \tau_i^r} w_i^t \geq r - \frac{1}{\sqrt{d\log(d)}} \geq \frac{r}{2}  \, , \qquad \inf_{\tau_i^\Delta < t < \otau \wedge T} w_i^t \geq \Delta - \frac{2}{\sqrt{d \log (d)}} \, .
\]

Denote $\alpha_t = \min \{ w_i^t : i \in S_t \cap [P_l] \setminus [P_{l-1}] \}$ (noting that $w_i^t \geq r/2$ for $i \in [P_l] \setminus [P_{l-1}]$ but $i\not\in S_t$). Furthermore, for $i \in [P_{l-1}]$ and $t > \tau^{r,l-1}$, we have $w_i^t \geq r/2$. Hence, for $t \leq \tau^{r,j} +1$,
\[
\alpha_t \geq \alpha^{\tau^{r,l-1}+1} + C \eta_1 a^0 r^{\oD_{l-1}} \sum_{s = \tau^{r,l-1}+1}^{t-1} \alpha_s^{D_l-1}\, .
\]
Furthermore, $\alpha^{\tau^{r,l-1}+1} \geq (1 - Cr^2)/\sqrt{d}$ by the previous step.
We deduce by Lemma \ref{lem:bound_sequences}: for $D_l> 2$,
\[
\alpha_t \geq r \wedge \left\{ \frac{1}{\big( C d^{D_l/2 - 1} - C \eta_1 a^0  r^{\oD_{l-1}}  t \big)^{\frac{1}{k-2}} } \right\} \, ,
\]
which implies
\[
\tau^{r,l} \leq \frac{C d^{D_l/2-1}}{\eta_1 a^0  r^{\oD_{l-1}}} \, .
\]
Similarly, for $D_l = 2$,
\[
\tau^{r,l} \leq \frac{C \log(d)}{\eta_1 a^0  r^{\oD_{l-1}}} \, .
\]
Similarly, we obtain similar bounds on $\tau^{\Delta,l}$ (see Step 2 of Section \ref{sec:subection_Proof_prop_one_monomial}).

\noindent
\textbf{Step 3: Concluding the proof.}

Theorem~\ref{thm:sequential_alignment}.(a) follows by Step 2 and taking $\eta_1$ and $\oT_1 :=T$ that satisfy \eqref{eq:cond_eta1_1}, \eqref{eq:cond_eta1_2} and \eqref{eq:cond_eta1_martingale}, and the growth conditions in Step 2. Theorem~\ref{thm:sequential_alignment}.(b) follows by the same argument as in Step 3 of Section \ref{sec:subection_Proof_prop_one_monomial}.
\end{proof}

\subsection{Extending the analysis: adaptive step size and non-nested monomials}\label{app:extending_analysis}

\subsubsection{Adaptive step-size}
\label{sec:adaptive_step_size}

Let us consider 
\begin{equation}\label{eq:nested_monomials_increasingleap}
\begin{aligned}
h_* (\bz) =&~ \sum_{l = 1}^L \prod_{s \in [P_l]} \He_{k_s} (z_s) \, ,
\end{aligned}
\end{equation}
with increasing\footnote{The case $D_1 = 1$ in the first phase of the dynamics can be studied easily by modifying the proof of Theorem \ref{thm:sequential_alignment} and noting that the 
drift is now just a sum of constant terms.} leaps $1 \leq D_1 < D_2 < \ldots < D_L =: D$, so that neurons align with the support sequentially at increasing time scales. As mentioned below Theorem \ref{thm:sequential_alignment}, the time complexity to escape each of these leaps is only tight for the biggest leap if we take a constant step size $\eta \propto d^{-D/2}$. Indeed, for the first phases of the dynamics, SGD requires a number of steps $d^{(D_l+D)/2 - 1}$ much smaller than $d^{D-1}$ to align to the $l$-th monomial. In that case, we can take bigger step sizes and still have negligible contribution from the martingale part of the dynamics. In practice, such as in Figure \ref{fig:StoS_punchy}, we can see a saddle-to-saddle dynamics\footnote{Again, we expect this saddle-to-saddle dynamic to occur in the case of increasing leaps, otherwise we might have mixing of the different phases for different neurons and no plateaus, except at the biggest leap.} to occur, with a number $O(d^{D_l - 1})$ of steps to escape each saddle even for constant step size. 

To prove these tight scalings for each plateau with constant step size, we would need to study the joint training of the two layers, which is currently out of reach of our proof techniques. Instead, we show in the next theorem that we can use a learning rate schedule $\eta^t$, i.e.,
\[
\tbw_j^{t+1} = \bw_j^t - \eta^t \cdot \grad_{\bw_j^t} \ell \big( y^t , \hf_{\NN} (\bx^t ; \bTheta^t ) \big)\, ,
\]
to get a scaling $\widetilde{\Theta} (d^{D_l -1})$ to align to each new monomial.

\begin{theorem}[First layer training adaptive step size]\label{thm:sequential_alignment_adaptive} Let $h_* : \R^P \to \R$ be defined as in Eq.~\eqref{eq:nested_monomials_increasingleap} and assume $\sigma$ satisfy Assumption \ref{ass:sigma_I}. Then for $0 < r < \Delta$ sufficiently small (depending on $D,K$) and $\rho \leq \Delta$ the following holds. For any constant $C_*>0$, there exist $C_i$ for $i= 0, \ldots , 5$, that only depend on $D,K$ and $C_*$ such that, by splitting our learning rate schedule in $L$ phases with step sizes $\eta^t = \eta_l$ for $t \in \in \{\oT_{l-1},\oT_{l-1}+1, \ldots ,\oT_l- 1\}$, with
\[
\begin{aligned}
 \oT_l = C_0 d^{D_l-1} \log(d)^{C_0}\, , \qquad \eta_l = \frac{1}{C_1 \kappa d^{D_l/2} \log(d)^{C_1} }\, , \qquad \kappa \leq \frac{1}{C_2 d^{C_2}}\,, \\
\end{aligned}
\]
the following events hold with probability at least $1 - MC_3 d^{-C_*}/r$. For any neuron $j \in [M]$, 
\begin{itemize}
    \item[(a)] Early stopping for $l\in [L]$: $|w_{j,i}^t - w_{j,i}^0| \leq C_4/\sqrt{d \log (d)}$ for all $i = P_{l-1} +1 , \ldots , d$ and $t \leq \oT_l / (C_5 \log (d)^{C_5} )$.
\end{itemize}

For any neuron $j\in[M]$ such that $a^0 \mu_{\oD_l } (\sigma) (w_{j,1}^0)^{k_1} \cdots (w_{j,P_l}^0)^{k_{P_l}} > 0$ for all $l \in [L]$,
\begin{itemize}
    \item[(b)] On the support: $\big\vert  w_{j,i}^{T_l}  - \sign (w_{j,i}^0) \cdot \Delta \big\vert \leq C_4 /\sqrt{d \log (d)}$ for $i = 1, \ldots , P_l$ and $l \in [L]$.

    \item[(c)] Outside the support: $| w_{j,i}^{T_l} - w_{j,i}^0| \leq C_5 r^2 /\sqrt{d}$ for $i = P_l+1, \ldots , d$ and $l \in [L]$. Furthermore, $\sum_{i >P_l} (w_{j,i}^{\oT_l} )^2 = 1$.
\end{itemize}
\end{theorem}

There are two key differences between Theorem \ref{thm:sequential_alignment_adaptive} and Theorem \ref{thm:sequential_alignment}. First we prove a tighter scaling $\widetilde{\Theta} (d^{D_l -1})$ of number of steps for the first phases of the training. Second we show that the alignment is sequential for all the neurons at the same time: at the end $\oT_l$ of each phase, we exactly picked up the support $[P_l]$ and nothing else. In particular, using a similar proof as in Corollary \ref{cor:learning_one_monomial}.(b), we can show that the neural network at time $\oT_l$ cannot fit the remaining $L-l$ monomials at all using the second layer weights. This agrees with the picture obtained in the numerical simulation in Figure \ref{fig:StoS_punchy}.

The $\oT_l$ and $\eta_l$ are chosen such that the martingale term remain negligible during the whole dynamics. Furthermore, because of the separation of time scales between the different phases of the dynamics, we can show that for $t \leq \oT_l$ and step size $\eta_l$, the contribution of the drift terms coming from the next monomials remains small. The proof follows almost identically to the proofs of Theorems \ref{thm:alignment_one_monomial} and \ref{thm:sequential_alignment}.

\subsubsection{Non-nested monomials}
\label{sec:non_nested}

While we wrote Theorems \ref{thm:sequential_alignment} and \ref{thm:sequential_alignment_adaptive} in the case of $h_*$ a nested sum of monomials, we note that this is foremost a convenient assumption that help simplify the equations in the proofs. However, the compositionality of the monomials in the decomposition of $h_*$ is not a required structure for the leap complexity to hold. Note that this compositionality might be favorable if we consider large $P = \omega_d (1)$ (such as in \cite{abbe2021staircase}) or for the dependency in $\eps$ and the Hermite coefficients of $h_*$ in the prefactor of $\widetilde{\Theta} ( d^{(\Leap (h_*) -1) \vee 1})$.

Below we describe how we can modify the proof of Theorem \ref{thm:sequential_alignment} for non-compositional $h_*$ and leave the task of proving Conjecture \ref{conj:leap} for general leap functions to future works. 

Consider $\bk_l = (k^{(l)}_1, \ldots , k^{(l)}_{P_l} ) \in \naturals^{P_l}$ for $l \in [L]$ such that
\begin{equation}\label{eq:non_nested}
h_* ( \bz) = \sum_{l = 1}^L \He_{\bk_l} (\bz) \, , \qquad \He_{\bk} (\bz) = \prod_{s \in [|\bk|]} \He_{k_s} (z_s) \, ,
\end{equation}
and $k^{(l)}_s >0$ for $s \in [P_l]\setminus [P_{l-1}]$ (each new coordinates appear in the next monomial) and $\oD_l = \| \bk_l \|_1$ with $\oD_1 < \oD_2 < \ldots < \oD_L$, and denote $D_l = \oD_l - \oD_{l-1}$ (with $\oD_0 = 0$). Denote $D = \max_{l \in [L]} D_l$ which corresponds to the leap complexity of $h_*$. 

First note that the same formulas as in Lemma \ref{lem:gradient_formula_multi} hold with 
\[
\chi_{*,l} (\bw^t) = \ind \{k_i^{(l)} >0 \} \prod_{j \in [P_l]} (w_j^t)^{k_j^{(l)}}\, , 
\]
however, we cannot simplify the gradient to be of order $\chi_{*,l} (\bw^t)/w_i^t$ during the $l$-th phase. Below, we outline how to modify the proof of Theorem \ref{thm:alignment_one_monomial} in Section \ref{app:proof_theorem_2} to the case \eqref{eq:non_nested}. The bounds on the martingale terms and on the dynamics from Section \ref{sec:bounding_contributions_dynamics} still hold in that case, with the difference being in the formulas of the population gradients.

By taking $\Delta$ small enough, there exists constants $C,c$ such that we can upper and lower bound $\og_i^t$ as follows. For $i \in [P_l] \setminus [P_{l-1}] $, during the $l' \leq l- 1$ phase, we have for $\tau^{r,l' - 1} + 1 <t \leq ( \tau^{r,l'}+1) \wedge \tau_i^\Delta$,
\[
\begin{aligned}
\og_i^t \leq&~ C a^0 \sum_{q \geq l} \frac{\chi_{*,q} (\bw^t)}{w_i^t} - ca^0 \sum_{l' \leq q < l} w_i^t \chi_{*,q} (\bw^t)\, , \\
\og_i^t \geq&~ c a^0 \sum_{q \geq l} \frac{\chi_{*,q} (\bw^t)}{w_i^t} - Ca^0 \sum_{l' \leq q < l} w_i^t \chi_{*,q} (\bw^t)\, .
\end{aligned}
\]
If $t > \tau^{r,l-1}$, then 
\[
c a^0 \sum_{q \geq l} \frac{\chi_{*,q} (\bw^t)}{w_i^t} \leq \og_i^t \leq C a^0 \sum_{q \geq l} \frac{\chi_{*,q} (\bw^t)}{w_i^t} \, .
\]
We can plug these population gradients in steps 1 and 2 in Theorem \ref{thm:sequential_alignment}, and control the contribution of each of these terms using $\alpha_{l,t} = \min \{ w_i^t : i \in S_t \cap [P_l] \setminus [P_{l-1}] \}$ and $\beta_{l,t} = \max \{ w_i^t : i \in S_t \cap [P_l] \setminus [P_{l-1}] \}$, with similar arguments as in step 2 of Section \ref{sec:subection_Proof_prop_one_monomial}.

\clearpage

\section{Fitting the second layer weights: proof of Corollaries \ref{cor:learning_one_monomial} and \ref{cor:sequential_learning}}

\subsection{Proof of Corollary \ref{cor:learning_one_monomial}: fitting one monomial}\label{app:second_layer_monomial}

We first focus on the case $h_* (\bz) = z_1 \cdots z_P$ and prove parts (a) and (b) separately in Sections \ref{sec:cor_1_part_a_parity} and \ref{sec:cor_1_part_b_parity}. The case of $h_* (\bz) = \He_D (z_1)$ follows from a similar argument and we outline the differences in Section \ref{sec:single_hermite_case}.

\subsubsection{Second-layer fitting}
\label{sec:cor_1_part_a_parity}

Recall that in this case $P = D$ and we can use both interchangeably. We consider the case of no biases in this part, i.e., fixing $b_j = 0, j \in [M]$.

\paragraph*{Phase I: first layer weights.}

By Theorem~\ref{thm:alignment_one_monomial}, with probability at least $1 - M d^{-C_*}$, for each neuron $(a,\bw)$ satisfying $a^0 \mu_P(\sigma ) w_1^0 \cdots w_P^0 >0$ at initialization, we get at the end of the dynamics:
\begin{align}
&\text{For $i \in [P]$:} \qquad  &&| w_i^{\oT_1} - \sign (w_i^0) \cdot \Delta | \leq C / \sqrt{d\log(d)} \, , \label{eq:first-layer-monomial-1} \\
&\text{For $i \not\in [P]$:} \qquad  &&| w_i^{\oT_1} - w_i^0  | \leq C r / \sqrt{d} \, , \qquad \sum_{i = P+1}^d (w_i^{\oT_1})^2 = 1 \label{eq:first-layer-monomial-2} \, .
\end{align}

For the remainder of the proof, assume the above event is true.

\paragraph*{Constructing good features.} Now we show that for any sign vector $\bdelta \in \{\pm 1\}^P$ we can combine multiple trained neurons to approximate the function $\E_{G}[\sigma(\Delta\<\bdelta,\bz\> + G)]$.
\begin{lemma}\label{lem:construct-smoothed-activation}
There exists a constant $C$ that depends only on $D,K$ such that the following is true. For any $R$ weights $\{\bw_{j_1}^0, \ldots , \bw_{j_R}^0\}$ which coincide on the first $P$ coordinates $\bw_{j_s,1:P}^0 = \bdelta / \sqrt{d}$ where $\bdelta \in \{\pm 1\}^P$, and with biases $b_{j_s}$ such that $|b_{j_s} - b| \leq 1/R$, and with $\sign(a^0_{j_s}) = \sign(\mu_P(0) w_{j_s,1}^0 \dots w_{j_s,P}^0)$, there exists a constant $C$ that only depends on $P,K$ such that 
\begin{align*}
\E_\bx \Big[ \Big( \frac{1}{R} &\sum_{s \in [R]} \sigma ( \<\bw_{j_s}^{\oT_1} , \bx \> + b_{j_s}) -  \E_G [ \sigma ( \Delta \<\bdelta , \bz\> +G +b)] \Big)^2 \Big] \\
&\leq  \frac{C}{\sqrt{d\log(d)}} + C r + \frac{C}{R} + \frac{C}{R^2} \sum_{s, s' \in [R]} |\<\bw_{j_s,P+1:d}^{0},\bw_{j_{s'},P+1:d}^{0}\>|\,
\end{align*}
\end{lemma}
\begin{proof}
First if we replace $\bw_{j_s}^{\oT_1}$ by $\Delta \bdelta$ and $b_{j_s}$ by $b$, the error is bounded by
\begin{align*}
|\sigma(\<\bw_{j_s}^{\oT_1},\bx\> + b_{j_s}) - \sigma(\Delta \<\bdelta,\bz\>+\<\bw_{j_s,P+1:d}^{\oT_1},\bx_{P+1:d}\> + b)| \leq  \frac{2PK}{\sqrt{d\log(d)}} + \frac{K}{R},
\end{align*}
which is accounted for by the first two terms since we can take $C$ large enough depending on $P,K$. For the last two error terms, 
\begin{align*}
&\E_\bx \Big[ \Big( \frac{1}{R} \sum_{s \in [R]} \sigma ( \Delta \<\bdelta , \bz\> + \< \bw_{j_s,P+1:d}^{\oT_1} , \bx_{P+1:d} \>) -  \E_G [ \sigma ( \Delta \<\bdelta , \bz\> +G )] \Big)^2 \Big] \\
&= \E_\bx \Big[\Big(\frac{1}{R} \sum_{s \in [R]} h(\bz,\<\bx_{P+1:d},\bw_{j_s,P+1:d}^{\oT_1}\>))\Big)^2\Big] = (\ast)
\end{align*}
where $h(\bz,u) = \sigma ( \Delta \<\bdelta , \bz\> + u + b) - \E_G [ \sigma ( \Delta \<\bdelta , \bz\>+G +b)]$.

If $\bu$ satisfies $\|\bu\| = 1$, then $\<\bu,\bx_{P+1:d}\>$ is distributed as $\normal(0,1)$. So for all $\bz$,
\begin{align*}
\E_{\bx_{P+1:d}}[h(\bz,\<\bx_{P+1:d},\bu\>)] = 0\,
\end{align*}
Furthermore, for any $\bu,\bv$ satisfying $\|\bu\| = \|\bv\| = 1$, let $\tilde{\bu} = \bu - \bv \<\bu,\bv\>$, which satisfies $\<\tilde{\bu},\bv\> = 0$. Then $\<\bx_{P+1:d},\tilde{\bu}\>$ and $\<\bx_{P+1:d},\bv\>$ are independent so
\begin{align*}
&|\E_{\bx_{P+1:d}}[h(\bz,\<\bx_{P+1:d},\bu\>)h(\bz,\<\bx_{P+1:d},\bv\>)]| \\
&\leq |\E_{\bx_{P+1:d}}[h(\bz,\<\bx_{P+1:d},\tilde{\bu}\>)h(\bz,\<\bx_{P+1:d},\bv\>)]| + |\E_{\bx_{P+1:d}}[|K\<\bx_{P+1:d},\tilde{\bu} - \tilde{\bv}\>|h(\bz,\<\bx_{P+1:d},\bv\>)]| \\
&= K|\E_{\bx_{P+1:d}}[|\<\bx_{P+1:d},\tilde{\bu} - \bu\>|h(\bz,\<\bx_{P+1:d},\bv\>)]| \\
&\leq K \sqrt{\E_{\bx_{P+1:d}}[\<\bx_{P+1:d},\tilde{\bu} - \bu\>^2]} \\
&= K |\<\bu,\bv\>|\,.
\end{align*}
So
\begin{align*}
(\ast) \leq \frac{K}{R^2} \sum_{s,s' \in [R]} |\<\bw_{j_s,P+1:d}^{\oT_1},\bw_{j_{s'},P+1:d}^{\oT_1}\>| \leq Cr + \frac{K}{R^2} \sum_{s,s' \in [R]} |\<\bw_{j_s,P+1:d}^{0},\bw_{j_{s'},P+1:d}^{0}\>|\,,
\end{align*}
which concludes the proof of the lemma.
\end{proof}

\paragraph*{Certificate.} 
We now write $h_*(\bz) = \prod_{i=1}^P z_i$ as a linear combination of functions of the form $\E_G[\sigma(\Delta\<\bdelta,\bz\>+G)]$ for different $\bdelta \in \{+1,-1\}^P$. By a Taylor approximation, for any $0 < s < 1$ and $x \in [-s, s]$,
\begin{align*}
\E_G [ \sigma (x + G) ] =&~  \sum_{k = 0}^P \frac{\mu_k(\sigma) }{k!} x^k + O(s^{P+1})\,,
\end{align*}
with a constant in the $O(\cdot)$ that depends only on $P,K$.
So if we define the coefficient 
\begin{align*}
c_{\bdelta} = \frac{P!}{2^P \Delta^P \mu_P(\sigma)} \prod_{i=1}^P \delta_i
\end{align*}
then we can approximate $h_*(\bz) = \prod_{i=1}^P z_i$ as follows for any $\bz$ such that $\Delta|\<\bdelta,\bz\>| < 1$,
\begin{align*}
\sum_{\bdelta \in \{\pm 1\}^P} c_{\bdelta} \E_G [ \sigma (\Delta\<\bdelta,\bz\> + G) ] &= \sum_{k = 0}^P \frac{\mu_k(\sigma)}{k!} \sum_{\bdelta \in \{\pm 1\}^P} \Delta^k\<\bdelta,\bz\>^k c_{\bdelta} + O(\Delta|\<\bdelta,\bz\>|^k) \\
=&~ \prod_{i=1}^P z_i + O(\Delta|\<\bdelta,\bz\>|^k)\,,
\end{align*}
where we use that for any $S \subseteq [P]$, we have $\frac{1}{2^P} \sum_{\bdelta} (\prod_{i=1}^P \delta_i)(\prod_{i \in S} \delta_i) = \begin{cases} 0, & S \neq [P] \\ 1, & S = [P] \end{cases}$.

Putting this together with Lemma~\ref{lem:construct-smoothed-activation} and the guarantees on the first layer weights after training \eqref{eq:first-layer-monomial-1} and \eqref{eq:first-layer-monomial-2}, we obtain the following lemma.
\begin{lemma} \label{lem:certificate-formal-monomial} There exists a constant $C > 0$ depending only on $P,K$ such that with probability at least $1 - d^{-C_*} - C\eps$ there exists a set of weights $\bTheta^{cert} = (\bW^{cert}, \ba^{cert})$ satisfying
\begin{itemize}
\item (First layer weights are the trained weights) For all $j \in [M]$, we have $\bw_j^{\oT_1} = \bw^{cert}_j$.
\item (Second-layer weights are small) We have $\|\ba^{cert}\| \leq C / (\Delta^P \sqrt{M})$.
\item (Squared error is small) We have $R^{sq}(\bTheta^{cert}) \leq \eps / 4$.
\end{itemize}
\end{lemma}
\begin{proof}
Consider the event that for each $\bdelta \in \{\pm 1\}^P$ the set $S_{\bdelta} = \{j : a_j^0 \mu_D(\sigma) w_{j,1}^0 \cdots w_{j,P}^0 > 0\}$ is of size $|S_{\bdelta}| \geq R := M / 2^{P+2}$. This holds with probability at least $1 - O(\eps)$ by a union bound and a Hoeffding bound, so we condition on it from now on. Consider the event that for all $\bdelta$ we have
\begin{align*}
\frac{1}{|S_{\bdelta}|^2} \sum_{j,j' \in S_{\bdelta}} |\<\bw_{j,P+1:d}^0, \bw_{j',P+1:d}^0\>| \leq \frac{1}{R} + C_{11}\sqrt{\frac{\log(d)}{d}}
\end{align*}
and note that this holds with probability at least $1 - d^{-C^*}$ by a Hoeffding bound for a constant $C_{11}$ depending on $P,K,C_*$, so we also condition on it.

Let $\ba^{cert}$ be given by $a^{cert}_j = c_{\bdelta} / |S_{\bdelta}|$ if  $j \in S_{\bdelta}$, and  0 otherwise. From this it follows that
\begin{align*}
\|\ba^{cert}\| \leq (\sqrt{M} /R) \max_{\bdelta} |c_{\bdelta}| \leq C / (\Delta^P \sqrt{M})\,,
\end{align*}
for a constant $C$ depending only on $P,K$. By Lemma~\ref{lem:construct-smoothed-activation},
\begin{align*}
R^{sq}(\bTheta^{cert}) &= 
\mathbf{E}_{\bx}\Big[\Big(\prod_{i=1}^P x_i - \sum_{\bdelta} c_{\bdelta} \frac{1}{|S_{\bdelta}|} \sum_{j \in S_{\bdelta}} \sigma(\<\bw_j^{\oT_1},\bx\>)\Big)^2\Big] \\
&\leq C \mathbf{E}_{\bx}\Big[\Big(\prod_{i=1}^P x_i - \sum_{\bdelta} c_{\bdelta} \E_{G}[\sigma(\Delta \<\bdelta,\bx_{1:P}\> + G)]\Big)^2\Big]\\
&\qquad \qquad+ C \Big(\sqrt{\frac{C_{11}\log(d)}{d}} + \frac{1}{R} + r\Big)^2 \\
&\leq \min_{s \in (0,1/\Delta)} C\Big\{ K^2 \P_{\bx}[|\<\bdelta,\bx_{1:P}\>| > s] + \Delta^2 s^{2k} + \frac{C_{11}\log(d)}{d} + \frac{1}{R^2} + r^2 \Big\} \\
&\leq \min_{s \in (0,1/\Delta)} C \Big\{ K^2 \exp(-(s/P)^2) + \Delta^2 s^{2k} + \frac{C_{11}\log(d)}{d} + \frac{1}{R^2} + r^2 \Big\} \\
&\leq \eps / 4\,,
\end{align*}
by taking a small enough choice of parameters $\Delta, r$ and large enough $d, M$.
\end{proof}

\paragraph*{Concluding fitting of the monomial:} Now that we have constructed the certificate $\bTheta^{cert}$, we 
show that SGD on the second layer converges quickly to a solution with low population loss by a bias-variance analysis of SGD for ridge-regularized least-squares linear regression in Lemma~\ref{lem:second-layer-certificate-square-loss}. We train the second-layer while keeping the weights of the first layer fixed, which corresponds to linear regression with input embedding $$\phi(\bx) = [\sigma(\<\bw_{j}^{\oT_1},\bx\>)]_{j \in [m]} \in \R^m\,.$$
Because of the boundedness of $\sigma$, we have $\|\phi(\bx)\| \leq K \sqrt{M}$ almost surely over $\bx$. Also, the initialization of the second layer implies $\|\ba^0\| \leq 1 / \sqrt{M}$. Finally, the labels $y^{\oT_1},y^{\oT_1+1},\ldots,y^{\oT_2-1}$ satisfy $\E[(y^t)^2] \leq K$ and $|y^t| \leq C_0 \log(1/\eps)^{C_0}$ for all $\oT^1 < t \leq \oT^2-1$ with probability at least $1 - d^{-C_*}$ by \eqref{eq:tail_bounds_fct_gaussian} and a union bound. So applying Lemma~\ref{lem:second-layer-certificate-square-loss} from Section~\ref{sec:technical-fitting}, there is a constant $C_{12}$ depending on $D,K$ such that if $\lambda_a \leq M$,
\begin{equation}\label{eq:second-layer-fitting-square-loss-monomial}
\begin{aligned}
\P\Big[R^{sq}(\bTheta^{\oT_1 + \oT_2})& \geq R^{sq}(\bTheta^{cert}) + \frac{\lambda_a}{2} \|\ba^{cert}\|^2  \nonumber \\
&+ C_{12}\log(1/\eps)^{C_{12}} M\left((1-\lambda_a \eta)^{2\oT^2} (\frac{1}{M} + \frac{1}{\lambda_a}) + \log(\oT_2 / \delta) \frac{\eta M^2}{\lambda_a^2}\right)\Big] \leq \delta \,. 
\end{aligned}
\end{equation}
So, plugging in Lemma~\ref{lem:certificate-formal-monomial} and taking $\lambda_a = \eps \Delta^{2P} M / (4C)$, $\delta = \eps$
\begin{align*}
&~\P\Big[R^{sq}(\bTheta^{\oT_1 + \oT_2}) \geq \eps /2 \\
&\qquad \qquad + C_{12}\log(1/\eps)^{C_{12}} M \Big((1-\frac{\eps \Delta^{2P}M}{4C}\eta)^{2\oT^2} (\frac{1}{M} + \frac{4C}{\eps M\Delta^{2P}}) + \log(\frac{\oT_2}{\eps}) \frac{16C^2\eta}{\eps^2 \Delta^{4P}}\Big)\Big] \\ \leq&~ C\eps + d^{-C_*}\,. 
\end{align*}
By taking $\eta = \frac{\eps^4 \Delta^{4P}}{16MC^2}$ and $\oT^2 = \frac{64C^3}{\eps^6 \Delta^{6P}}$, for small enough $\eps$,
\begin{align*}
\P\Big[R^{sq}(\bTheta^{\oT_1 + \oT_2}) \geq & \eps \Big] \leq C\eps + d^{-C_*}\,. 
\end{align*}
This proves part (a) of Corollary~\ref{cor:learning_one_monomial}.

\subsubsection{Converse if early stopping}
\label{sec:cor_1_part_b_parity}

We now prove the converse. The proof will follow very similarly to the proof of \cite[Theorem 1]{ghorbani2021linearized}. By Theorem~\ref{thm:alignment_one_monomial}, if we train the first layer for time $\oT_1' \leq \oT_1 / (C_8 \log(d)^{C_8})$ steps for a large enough $C_8 > 0$, then with probability at least $1 - Md^{-C_*}$ for each neuron $j \in [M]$,
\begin{align}\label{eq:event_earlystop_small}
| w_{j,i}^{\oT_1'} - w_{j,i}^0| \leq C_4 /\sqrt{d \log (d)}\mbox{ for all }i \in [d]\,,
\end{align}
and some constant $C_4$. In particular, this implies that for large enough $d$, 
\begin{align*}
\Big\vert w_{j,i}^{\oT_1'} \Big\vert \leq 2 / \sqrt{d}\,.
\end{align*}
For ease of notations, denote $\bw_j := \bw_j^{\oT_1'}$. 
Let us introduce $\phi (\bx) = [ \sigma ( \< \bw_1 , \bx\>), \ldots ,  \sigma ( \< \bw_M , \bx\>)]$ and $\phi_0 (\bx) = [ \sigma ( \< \bw_1^0 , \bx\>), \ldots ,  \sigma ( \< \bw_M^0 , \bx\>)]$.

By a simple calculation, we have
\[
\min_{\ba \in \R^M} \E_{\bx} \Big[ \big( f_* (\bx) - \ba^\sT \phi (\bx) \big)^2 \Big] = \| f_* (\bx) \|_{L^2}^2 - \bV^\sT \bU^{-1} \bV \, ,
\]
where we denoted
\[
\bV = \E_{\bx} [ \phi (\bx) f_* (\bx)] \in \R^M\, , \qquad \bU = \E_{\bx} [ \phi (\bx) \phi (\bx)^\sT] \in  \R^{ M \times M}\, .
\]
Corollary \ref{cor:learning_one_monomial} will follow by showing that there exist constants $c,C$ that only depend on $K,P$ such that with high probability, we have 
\[
\lambda_{\min} ( \bU ) \geq c,\, \qquad \| \bV \|_2^2 \leq C M d^{-P}\, .
\]
These are proved in the following two lemmas.

\begin{lemma}\label{lem:lower_bound_U}
    Under the same setting as in Corollary \ref{cor:learning_one_monomial}, there exist  constants $c,C>0$ such that with probability at least $1 - CMd^{-C_*}$,
    \[
    \lambda_{\min} ( \bU ) \geq c \, .
    \]
\end{lemma}

\begin{proof}[Proof of Lemma \ref{lem:lower_bound_U}]
    Consider the event described in Eq.~\eqref{eq:event_earlystop_small}. By rotational invariance of the distribution of $\bx$, the entries $\bU = (U_{ij})_{i,j \in [n]}$ are given by
    \[
    U_{ij} = \E_{\bx} [ \sigma ( \< \bw_i, \bx \>) \sigma ( \< \bw_j , \bx \> ) ] = \E_{G_1,G_2} \Big[\sigma ( \alpha_{i} G_1 ) \sigma \Big( \alpha_j \beta_{ij} G_1 + \alpha_j \sqrt{1 - \beta_{ij}^2 } G_2 \Big)   \Big]\, , 
    \]
    where $(G_1,G_2) \sim \normal (0, \id_2)$, $\alpha_i = \| \bw_i\|_2 $ and $\beta_{ij}= \< \bw_i , \bw_j \>/(\alpha_i \alpha_j )$.

    For $i =j$, we have
    \[
    \begin{aligned}
    U_{ii} =&~ \E_G [ \sigma ( \alpha_{i} G )^2] =  \E_G [ \sigma ( G )^2] +  \E_G [ \sigma ( \alpha_{i} G )^2 - \sigma ( G )^2] \, .
    \end{aligned}
    \]
    We can do a Taylor expansion and bound the second term
    \[
    \begin{aligned}
      \big\vert \E_G [ \sigma ( \alpha_{i} G )^2 - \sigma ( G )^2] \big\vert \leq &~ \big\vert \E_G \Big[ \Big( \sigma ( G ) + (\alpha_i - 1)G \sigma' (c(G))\Big)^2 -\sigma ( G )^2 \Big] \big\vert \\
      \leq&~ 2 K^2 \vert \alpha_i - 1 \vert \E [ |G| ] + K^2 \vert \alpha_i - 1 \vert^2 \E [ |G|^2 ] \\
      \leq&~ \frac{C}{\sqrt{\log(d)}}\, ,
    \end{aligned}
    \]
    where we used that $|\| \bw_i \|_2 -1| \leq C /\sqrt{\log (d) }$ by Eq.~\eqref{eq:event_earlystop_small}.

    Consider now $i \neq j$. Note that
    \[
    h(t) = \E_{G_1,G_2} \Big[\sigma ( G_1 ) \sigma \Big( t G_1 + \sqrt{1 - t^2 } G_2 \Big)   \Big]\, ,
    \]
    has derivative
    \[
   h'(t)  = \E_{G_1,G_2} \Big[\sigma ( G_1 ) \sigma ' \Big( t G_1 + \sqrt{1 - t^2 } G_2 \Big)  (G_1 + t/\sqrt{1- t^2} G_2 ) \Big]\, .
    \]
    Hence, for $|t|\leq 1/2$, we have $|h'(t)| \leq C |t|$. Note that $|\beta_{ij} - \< \bw_i^0 , \bw_j^0\>| \leq C/\sqrt{\log(d)}$. By standard concentration, using that $\bw_i^0 \sim \Unif ( \{ \pm 1 / \sqrt{d} \}^d)$, there exists constants $c,C$ such that with probability at least $1 - e^{-cd}$, we have
    \[
    \max_{i\neq j \in [M]} | \< \bw_i^0 , \bw_j^0 \> | \leq C \log(M) / \sqrt{d} \, .
    \]
    Using the same computation as above, we can replace $\alpha_i$ and $\alpha_j$ by $1$ while only incurring an error $C/\sqrt{\log(d)}$, and show that
    \[
    | U_{ij} - h(0) | \leq C/\sqrt{\log(d)} + C \log(M) / \sqrt{d} \, .
    \]

    From the above bounds, we deduce (using $\| \bM \|_{\text{op}} \leq \| \bM\|_F$) that with high probability
    \[
    \| \bU - h(0)^2 \bones \bones^\sT - (h(1) - h(0)) \id \|_{\text{op}} \leq CM /\sqrt{\log(d)}\, .
    \]
    For $\sigma$ not constant, $h(1) > h(0)$ and using that $M = O_d(1)$, we deduce that
    \[
    \lambda_{\min} ( \bU) \geq \frac{h(1) - h(0)}{2}\, , 
    \]
    which concludes the proof.
\end{proof}

\begin{lemma}\label{lem:correlation-at-initialization-bound}
Under the same setting as in Corollary \ref{cor:learning_one_monomial}, there exists  constants $C>0$ such that with probability at least $1 - CMd^{-C_*}$,
    \[
    \| \bV \|_2^2 \leq C M d^{-P} \, .
    \]
\end{lemma}
\begin{proof}[Proof of Lemma \ref{lem:correlation-at-initialization-bound}]
Note that we have
\[
\| \bV \|_2^2 = \sum_{j \in [M]} \E_{\bx} [f_*(\bx) \sigma ( \< \bw_j , \bx\>)]^2
\]
First note that for any $\bw$, the correlation of $\sigma(\<\bw,\bx\>)$ with $f_*(\bx) = \prod_{i=1}^{P} x_i$ is bounded by
\begin{align*}
|\E_{\bx}[\sigma(\<\bw,\bx\>)f_*(\bx)]| \leq K \prod_{i \in [P]} |w_i|\,.
\end{align*}
Indeed, as in the proof of Lemma~\ref{lem:gradient_formula}, we use the formula from integration by parts:
\begin{align*}
\E_{\bx} \Big[\sigma(\<\bw,\bx\>)\prod_{i \in [P]} x_i \Big] &= \Big(\prod_{i \in [P]} w_i \Big) \cdot \E_G[\sigma^{(P)}(\|\bw\|_2 G)]\,,
\end{align*}
using $\|\sigma^{(P)}\|_{\infty} \leq K$.

We conclude by noting that on the high probability event \eqref{eq:event_earlystop_small}, we have $| w_{j,i} | \leq 2 / \sqrt{d}$.
\end{proof}

\subsubsection{Proof for a single-index Hermite monomial}
\label{sec:single_hermite_case}

Let's now consider $h_* (\bz) = \He_D (z_1)$. In this case, we consider the biases $b_j \sim \Unif ( [ -\Delta, \Delta])$, where $\Delta$ is chosen sufficiently small as discussed in Theorem \ref{thm:alignment_one_monomial}. We can use the same proof strategy as in Section \ref{sec:cor_1_part_a_parity} and construct good features 
\[
\E_G [ \sigma ( \Delta z_1 + b + G )]
\]
for any $b \in [-\Delta, \Delta]$, by considering neurons with initializations $\{ (\bw_{j_1}^0,b_{j_1}^0) , \ldots , (\bw_{j_R}^0,b_{R}^0) \}$ with $w_{j_s}^1 = 1/\sqrt{d}$ and $\sign (a_{j_s}^0 ) = \sign (\mu_D (0) w_{j_s , 1}^D)$, and $|b_{j_s}^0 - b | \leq r$ (by an easy modification of Lemma \ref{lem:construct-smoothed-activation}). We will take sufficiently many neurons (but still independent of $d$) so that we have a sufficiently large $R$ for any intervals of size $r$ for $b \in [-\Delta, \Delta]$ with high probability.

Let us now construct a certificate for $\He_D (z_1)$ based on these good features. By a Taylor approximation, for any $0<s<1$ and $x  \in [-s,+s]$,
\[
\begin{aligned}
\E_G [  \sigma ( x + b + G )] =&~ \sum_{k=0}^D \frac{\mu_k (\sigma)}{k!} (x+b)^k + O(s^{D+1} + \Delta^{D+1}) \\ 
=&~ \sum_{k = 0}^D b^k \left[ \sum_{s = 0}^{D -k} \frac{\mu_{k+s}(\sigma)}{(k+s)!} {{k+s}\choose{s}} x^s \right] +  O(s^{D+1} + \Delta^{D+1}) \\ 
=: &~ \sum_{k = 0}^D b^k Q_{D-k} (x) + O(s^{D+1} + \Delta^{D+1}) \, .
\end{aligned}
\]
We can consider measures with density $\nu_\ell(b)$ with respect to $b \sim \Unif([-\Delta , \Delta])$ such that
\[
\begin{aligned}
\int_{-\Delta}^\Delta \E_G [  \sigma ( x+ b + G )] \nu_\ell (b) \de b =&~ \sum_{k=0}^P Q_{P-k} (x) \int_{-\Delta}^\Delta b^k \nu_\ell (b) \de b + O(s^{D+1} + \Delta^{D+1}) \\
=&~ Q_{D-\ell} (x) + O(s^{D+1} + \Delta^{D+1})  \, .
\end{aligned}
\]
Note that the polynomials $\{ Q_k\}_{k =0, \ldots , D}$ are linearly independent (distinct degrees) with coefficients that only depend on $D,K$. Hence we can take a linear combination of $\nu_\ell(b)$ with coefficients that only depend on $D$ and $K$ such that we have $\Tilde{\nu}_\ell$ with
\[
\int_{-\Delta}^\Delta \E_G [  \sigma ( x+ b + G )] \Tilde{\nu}_\ell (b) \de b = x^\ell + O(s^{D+1} + \Delta^{D+1}) \, . 
\]
In particular, we can rescale and sum these coefficients such that for some $\nu_* (b)$ that has second moment bounded by $1/\Delta^{CD}$,
\[
\int_{-\Delta}^\Delta \E_G [  \sigma ( \Delta z + b + G )] \nu_* (b) \de b =\He_D (z_1) + O(s(s/\Delta)^{D} + \Delta) \, .
\]
We can now construct a certificate by sampling $b_s$ from the signed measure $\nu_* (b)$, and for each $b_s$ constructing an approximate good feature, as described in Lemma \ref{lem:construct-smoothed-activation}. The proof for the low test error then follows from applying the bound on the least squares linear regression of Lemma \ref{lem:second-layer-certificate-square-loss}.

For the lower bound with early stopping, we use that
\[
| \E_\bx [ \sigma( \< \bw, \bx\> + b) \He_D (x_1) ]| = | \E_\bx [ w_1^D \sigma^{(D)} ( \< \bw, \bx\> + b) ] | \leq K | w_1|^D \, ,
\]
and we can conclude using the same argument as in Section \ref{sec:cor_1_part_b_parity}.

\subsection{Proof of Corollary \ref{cor:sequential_learning}: sequential learning of monomials}\label{sec:proof_seq_learning}
Let us formally state Corollary~\ref{cor:sequential_learning} and prove it.
\begin{corollary}[Second layer training, sum of monomials; formal statement]\label{cor:sequential_learning_formal}
Let $$h_* (\bz) = \sum_{l \in [L]} \prod_{i=1}^{P_l} z_l$$ for some $P_1 < P_2 < \dots < P_L = P$. Then there exists an activation function $\sigma : \R \to \R$ such that the following is true. For any constants $C_*>0$ and $\eps>0$, there exist $C_i$ for $i= 0, \ldots , 10$, that only depend on $D, K$ and $C_*$ such that taking width $M = C_0 \eps^{-C_0}$,  bias initialization scale $\rho = C_1$, second-layer initialization scale $\kappa = \frac{1}{C_2 M d^{C_2}}$, second-layer regularization $\lambda_a = M\eps / C_3$, and $\Delta = \eps^{C_4} / C_4 $, and $r = \eps^{C_{5}} / C_{5} $, and
\[
\begin{aligned}
\oT_1 =&~ C_6 d^{D-1} \log(d)^{C_6}\, , \qquad &&\eta_1 = \frac{1}{C_7 \kappa d^{D/2}\log(d)^{C_7}}\, , \\
\oT_2 =&~ C_8 \eps^{-C_8} \, , \qquad &&\eta_2 =  \eps^{C_9}/(C_9 M)\, ,
\end{aligned}
\]
we have for large enough $d \geq C_{10}\eps^{-C_{10}}$, that with probability at least $1 - d^{-C_*} - \eps$ at the end of the dynamics,
\[
R ( \bTheta^{\oT_1 + \oT_2} ) \leq \eps \, .
\]
\end{corollary}

In contrast to the proof of Corollary~\ref{cor:learning_one_monomial}, we only prove this result for ``diverse'' enough activation functions. For the proof, we will construct a specific activation function that have this ``diversity'' property. This activation depends on $P$ (or upper bound on $P$), but otherwise is independent of $h_*$. The idea is that we will use biases of different magnitudes, which will change the signs of the Hermite coefficients of the activation, in order to ensure enough neurodiversity to learn the sum of increasing monomials. This is required due to the specific choice of training of the first layer weights considered in this paper. However, we show in simulations that standard ReLus activations are enough to learn these functions.

\paragraph*{Construction of activation function}
For any bias $b \in \R$, define $\sigma_b(x) = \sigma(x + b)$. We construct the activation function such that for all $\bs \in \{+1,-1\}^P$ there is a bias $ b(\bs) \in [-C,C]$ satisfying 
\begin{align}\label{eq:fancy-activation-hermite-guarantee}
\mu_k(\sigma_{b(\bs)}) = s_k \mbox{ for all } k \in [P],
\end{align}
for all $i \in [P]$. This can be achieved as follows. Let $\tau > 0$ be a constant that we will take large enough. Then for any $k$, define the ``truncated Hermite function''
\begin{align*}
p_{k,\tau}(x) = \He_k(x) m_{\tau}(x)\,,
\end{align*}
where $m_{\tau} : \R \to [-1,1]$ is a compactly-supported smooth function such that $$m_{\tau}(x) = \begin{cases} 0, & x \not\in [-\tau,\tau] \\ 1, & x \in [-\tau/2,\tau/2] \\ \in [-1,1], & \mbox{ otherwise}\end{cases}.$$
We order the sign vectors $\bs^{(1)},\ldots,\bs^{(2^P)} \in \{+1,-1\}^P$ arbitrarily. The bias $b(\bs^{(i)})$ is given by $b(\bs^{(i)}) = -4i\tau$. The activation function $\sigma : \R \to \R$ is given by
\begin{align*}
\sigma(x) = \sum_{\bs \in \{+1,-1\}^P} \sum_{k \in [P]} \gamma_{\bs,k} p_{k,\tau}(x - b(\bs))\,,
\end{align*}
for some choice of coefficients $\gamma_{\bs,k} \in \R$ depending only on $P$. This satisfies Assumption~\ref{ass:sigma_I} because $p_{k,\tau}$ is uniformly-bounded and has uniformly-bounded first $P+3$ derivatives. It remains to show that we can choose the coefficients $\gamma_{\bs,k}$ so that \eqref{eq:fancy-activation-hermite-guarantee} holds. This is because for any $\bs$ we have
\begin{align*}
\mu_k(\sigma_{b(\bs)}) = \E_G[\He_k(G) \sigma(G + b(\bs))] = \sum_{\bs' \in \{+1,-1\}} \sum_{k' \in [P]} \gamma_{\bs',k'} A_{(\bs',k'),(\bs,k)},
\end{align*}
where
\begin{align*}
A_{(\bs',k'), (\bs,k)} = \E_{G}[p_{k',\tau}(G - b(\bs') + b(\bs)) \He_k(G)]\,.
\end{align*}
And we show that $A_{(\bs',k'),(\bs,k)}$ is invertible when viewed as a $P2^P \times P2^P$ matrix.
For large enough $\tau$ depending on $k$, the diagonal elements are lower-bounded by a constant:
\begin{align*}
A_{(\bs,k),(\bs,k)} = \E_G[\He_k(G)\He_k(G)m_{\tau}(G)] > 1/2\,,
\end{align*}
And the off-diagonal elements are small. When $\bs \neq \bs'$, for large enough $\tau$ we have
\begin{align*}
|A_{(\bs',k'),(\bs,k)}| &= |\E_{G}[\He_{k'}(G - b(\bs') + b(\bs)) m_{\tau}(G - b(\bs') + b(\bs)) \He_k(G) ]| \\
&\leq C\int_{-\tau+b(\bs') - b(\bs)}^{\tau+b(\bs') - b(\bs)} \exp(-x^2/2) |x-b(\bs')+b(\bs)|^{k'} |x|^{k}dx \\
&\leq C\tau \min_{s \geq 3\tau} \exp(-s^2/2) |\tau|^{k'} |s+2\tau|^{k}dx \\
&< 1/\tau\,.
\end{align*}
And similarly when $\bs = \bs'$ but $k \neq k'$, for large enough $\tau$ we have
\begin{align*}
|A_{(\bs,k'),(\bs,k)}| &= |\E_G[\He_{k'}(G) \He_{k}(G) m_{\tau}(G)]| = |\E_G[\He_{k'}(G) \He_{k}(G) (1 - m_{\tau}(G))]| \\
&\leq |\E_G[|\He_{k'}(G)| |\He_k(G)| 1(|G| > \tau/2)]| \leq C |\E_G[|G|^{k'+k} 1(|G| > \tau/2)]| \\
&< 1/\tau\,.
\end{align*}
So if we take large enough $\tau$ the system of equations defined by $A_{(\bs',k'),(\bs,k)}$ is invertible, so coefficients $\gamma_{\bs,k}$ exist such that $\sigma$ satisfies \eqref{eq:fancy-activation-hermite-guarantee}.

\paragraph*{Certificate.} 
Now we provide a certificate for learning $h_*(\bz) = z_1 \cdots z_{P_1} + z_1 \cdots z_{P_2} + \dots + z_1 \cdots z_{P_L}$, which is a linear combination of functions of the form $\E_G[\sigma(\Delta\<\bdelta,\bz\>+G+b)]$ for different $\bdelta \in \{+1,-1\}^P$ and biases $b \in [-C,C]$.

The main difficulty is that we no longer have access to $\E_G[\sigma(\Delta\<\bdelta,\bz\> + G)]$ for each $\bdelta \in \{+1,-1\}^d$, so we have to compensate by using the biases. For each $\bdelta \in \{+1,-1\}^P$, let $\bs \in \{+1,-1\}^P$ be a sign vector such that $s_l \prod_{i=1}^{P_l} \delta_{i} > 0$ for all $l \in [L]$. Then, by Theorem~\ref{thm:sequential_alignment} and the guarantee from \eqref{eq:fancy-activation-hermite-guarantee} a constant fraction of neurons $j$ after training the first layer have $\bw^{\oT_1}_{j,1:P} \approx \Delta \bdelta$ and bias $b_j \approx b(\bs) + \zeta$ for any $\zeta \in [-\Delta,\Delta]$. (Note that we can apply Theorem~\ref{thm:sequential_alignment} despite its restriction that $\rho \in [-\Delta,\Delta]$, since we only care about the result holding for neurons whose bias is in $b(\bs) + [-\Delta,\Delta]$ for different $\bs$). So by Lemma~\ref{lem:construct-smoothed-activation}, we can combine them first layer to approximate $\E_{G}[\sigma(\Delta\<\bdelta,\bz\> + G + b(\bs) + \zeta)]$ for any $\zeta\in [-\Delta,\Delta]$. By an analogous argument to Section~\ref{sec:single_hermite_case}, we can find a measure with density $\nu_k$ with respect to $\zeta \sim \Unif([-\Delta,\Delta])$ that allows us to approximate
\begin{align*}
\int_{-\Delta}^{\Delta}\E_{G}[\sigma(\Delta\<\bdelta,\bz\> + G + b(\bs) + \zeta)] \nu_k(\zeta) d\zeta = \<\bdelta,\bz\>^k + O(\Delta)\,,
\end{align*}
and where $\nu_k(\zeta)$ has second moment bounded by $1/\Delta^{Ck}$. Since we can estimate $\<\bdelta,\bz\>^k$ to $O(\Delta)$ error for each $\bdelta \in \{+1,-1\}^P$, we can approximate $h_*$ via a linear combination
\begin{align*}
\sum_{\bdelta \in \{+1,-1\}^P} \sum_{l=1}^{L} \left(\prod_{i=1}^{P_l} \delta_i \right) \big(\<\bdelta,\bz\>^{P_l} + O(\Delta)\big) = h_*(\bz) + O(\Delta)\,.
\end{align*}
We conclude analogously to the proof of Corollary~\ref{cor:learning_one_monomial}, using the bounded-norm certificate to obtain a generalization guarantee.

\subsection{Technical result: last iterate convergence of SGD on linear models}\label{sec:technical-fitting}

We analyze of the last iterate for online-SGD on a linear model with ridge-regularized least-squares loss by using the well-known bias-variance decomposition \cite{jain2017markov}. A very similar analysis also appears in the appendix of \cite{abbe2022merged}; the key difference is that we analyze online gradient descent with one sample per iteration (as opposed to online minibatch gradient descent) with a small learning rate in order to match the setting of the theorem. Compare also to \cite{zhang2004solving} which gives final-iterate bounds for the final risk, but these hold in expectation instead of with exponentially high probability.

Given an embedding of data $\phi(\bx) \in \R^N$, consider training a linear model $\<\ba, \phi(\bx)\>$ with online-SGD. In this section, write the square loss as
\begin{align*}
\cL(\ba) = \frac{1}{2}\E_{\bx,y}[(y - \<\ba,\phi(\bx)\>)^2]\,.
\end{align*}
For a parameter $\lambda_a > 0$, the ridge-regularized square loss is 
\begin{align*}
\cL_{\lambda_a}(\ba) = \cL(\ba) + \frac{\lambda_a}{2} \|\ba\|^2
\end{align*}
Each iteration of the dynamics of online-SGD on the ridge-regulariezd square loss is is given by
\begin{align*}
\ba^{t+1} &= (1 - \lambda_a) \ba^t + \eta (y^t - \<\ba^t, \phi(\bx^t)\>) \phi(\bx^t)\,.
\end{align*}

\begin{lemma}[Analysis of online-SGD on linear model with ridge-regularized square loss]\label{lem:second-layer-certificate-square-loss}
There is a universal constant $C > 0$ such that following holds. Suppose there is $B_1 \geq 1$ such that $\|\phi(\bx)\| \leq B_1 \sqrt{N}$ almost surely, and $|y^s| \leq B_1$ for all $0 \leq s \leq t$, and $\E[y^2] \leq B_1^2$, and $\lambda_a \leq N$. Then for any $\ba^{cert} \in \R^N$
\begin{align*}
\P\left[\cL(\ba^t) \geq \cL_{\lambda_a}(\ba^{cert}) + C B_1^2 N \left((1-\lambda_a \eta)^{2t} (\|\ba^0\|^2 + \frac{B_1^2}{\lambda_a}) + \log(t / \delta) \frac{\eta B_1^6 N^2}{\lambda_a^2}\right)\right] \leq \delta\,.
\end{align*}
\end{lemma}
\begin{proof}
Let $\ba^*$ be the minimizer of $\cL_{\lambda_a}$, which is unique by strict convexity when $\lambda_a > 0$. We prove the following convergence to the optimum. For any iteration $t$, define the gap to optimality
\begin{align*}
\balpha^t = \ba^t - \ba^*\,.
\end{align*}
Defining $\bH = \E_\bx[\phi(\bx) \otimes \phi(\bx)] + \lambda_a \bI$ and $\bv = \E[\phi(\bx)y]$, the excess loss at iteration $t$ equals
\begin{align*}\cL_{\lambda_a}(\ba^t) - \cL_{\lambda_a}(\ba^*) &= \frac{1}{2} \<\ba^t \otimes \ba^t, \bH\> - \frac{1}{2} \<\ba^* \otimes \ba^*, \bH\>-\<\bv,\ba^t - \ba^*\> \\
&= \frac{1}{2} \<\balpha^t \otimes \balpha^t, \bH\> ,
\end{align*}
by the first-order optimality condition $\bH \ba^* = \bv$. So 
\begin{align}\label{eq:linear-optimality-gap}
\cL_{\lambda_a}(\ba^t) - \cL_{\lambda_a}(\ba^*) &\leq \frac{1}{2} \|\balpha^t\|^2 (\E_\bx \|\phi(\bx)\|^2 + \lambda_a)\,.
\end{align}
It remains to bound $\|\balpha^t\|$. We write the evolution of $\balpha^t$ as:
\begin{align*}
\balpha^{t+1} = \bP^t \balpha^t + \eta \bzeta^t
\end{align*}
where
\begin{align*}
\bP^t = \bI - \eta (\phi(\bx^t) \otimes \phi(\bx^t) + \lambda_a \bI) \\
\bzeta^t = y^t \phi(\bx^t) - (\phi(\bx^t) \otimes \phi(\bx^t) + \lambda_a \bI) \ba^*\,.
\end{align*}

Inductively, one obtains the well-known ``bias-variance'' decomposition $$\balpha^t = \underbrace{\Big(\prod_{l=t-1}^{0} \bP^l\Big) \balpha^0}_{(\mbox{Bias term})} + \underbrace{\eta \sum_{j=0}^{t-1} \Big(\bP^{t-1} \cdots \bP^{j+1} \Big) \bzeta^j}_{(\mbox{Variance term})}.$$

Notice that (a) $\bzero \lesssim \bP^l \lesssim (1 - \eta \lambda_a) \bI$, and (b) $\cL_{\lambda_a}(\ba) \geq \frac{\lambda_a}{2} \|\ba\|^2$ and $\cL_{\lambda_a}(\bzero) = \frac{1}{2} \E[y^2]$, so
\begin{align}\label{eq:bias-term-bound}
\|\mbox{(Bias term)}\| \stackrel{(a)}{\leq} (1 - \eta\lambda_a)^t \|\balpha^0\| \stackrel{(b)}{\leq} (1-\eta\lambda_a)^t(\|\ba^0\| + \sqrt{\E[y^2] / \lambda_a})\,.
\end{align}

To bound the variance term, define the norm squared of the variance term:
$$m_t = \eta^2\Big\|\sum_{j=0}^{t-1} \left(\bP^{t-1} \cdots \bP^{j+1} \right) \bzeta^j\Big\|^2\,.$$
Also, for any time $t$, define $\tilde{m}_0 = 0$ and $$\tilde{m}_{t+1} = \eta^2 \|\bzeta^t\|^2 + 2 \eta^2 \Big\<\bzeta^t, \bP^t \sum_{j=0}^{t-1} (\bP^{t-1} \dots \bP^{j+1}) \bzeta^j \Big\> + (1- \eta \lambda_a)^2 \tilde{m}_t.$$
By induction on $t$, we can show that $\tilde{m}_t \geq m_t$ at all times $t$. The base case is clear since $m_0 = \tilde{m}_0 = 0$. The inductive step is:
\begin{align*}
m_{t+1} &= \eta^2 \Big\|\bzeta^{t} + \bP^{t}\sum_{j=0}^{t-1} (\bP^{t-1} \dots \bP^{j+1}) \bzeta^j \Big\|^2 \\
&\leq \eta^2 \|\bzeta^t\|^2 + 2 \eta^2 \Big\<\bzeta^t, \bP^t \sum_{j=0}^{t-1} (\bP^{t-1} \dots \bP^{j+1}) \bzeta^j \Big\> + (1-\eta \lambda_a)^2 m_t^2 \\
&\leq \tilde{m}_{t+1}\,,
\end{align*}
where we use the inductive hypothesis $\tilde{m}_t \geq m_t$.

The reason we study $\tilde{m}_t$ instead of $m_t$ is because it satisfies these bounded differences:
\begin{align}\label{ineq:tilde-m-bounded-difference}
|\tilde{m}_{t+1} - (1-\eta \lambda_a)^2 \tilde{m}_t| \leq \eta^2\|\bzeta^t\|^2 + 2\eta\|\bzeta^t\|(1-\eta \lambda_a)\sqrt{\tilde{m}_t}\,.
\end{align}
Furthermore, let $\cF_t = \sigma \big(\{ \bx^s,y^s \}_{s \leq t} \big)$ be the history until time $t$. Since $\E[\bzeta^t \mid \cF_{t-1}] = \E[\bzeta^t \mid \cF_{t-1}] = \bv - \bH \ba^* = \bzero$,
\begin{align}\label{ineq:tilde-m-expectation-recurrence}
\E[\tilde{m}_{t+1} \mid \cF_{t}] &= \eta^2 \E[\|\bzeta^t\|^2] + (1-\eta \lambda_a)^2 \tilde{m}_t\,.
\end{align}
So the martingale concentration bound in Lemma~\ref{lem:technical-martingale-bound} applied to $\tilde{m}_t$ and using \eqref{ineq:tilde-m-bounded-difference} and \eqref{ineq:tilde-m-expectation-recurrence} with $c = (1-\eta\lambda_a)^2$, $a = \eta^2 \E[\|\bzeta^0\|^2]$ and $M = \max_{0 \leq t' \leq t-1} \eta^2\|\bzeta^{t'}\|^2$, yields, for any $\eps \geq \max(M/c, a^2 / Mc)$, and some large enough universal constant $C > 0$,
\begin{align*}
\P\Big[\tilde{m_t} \geq \frac{a}{1-c} + \eps\Big] \leq t\exp \Big(-\frac{\eps (1-c^2)}{ C M} \Big)\,.
\end{align*}
By applying $\|\ba^*\| \leq \sqrt{\E[y^2] / \lambda_a}$ and triangle inequalities, we have
\begin{gather*}
M \lesssim \eta^2 (B_1^4 N + ((B_1^3 N / \sqrt{\lambda_a} + \sqrt{\lambda_a} B_1)^2)) \lesssim \eta^2 B_1^6 N^2 / \lambda_a \, ,\\
a \lesssim \eta^2 B_1^6 N^2 / \lambda_a \, , \\
\eta \lambda_a \leq 1-c \leq 1 -c^2 \leq 4\eta \lambda_a \leq 4\eta N\,.
\end{gather*}
Plug this in and simplify,
\begin{align*}
\P\Big[\tilde{m}_t \geq C\frac{\eta B_1^6 N^2}{\lambda_a^2} + \eps \Big] \leq t\exp \Big(-\frac{\eps 
\lambda_a^2}{C 
 \eta B_1^6 N^2} \Big)\,,
\end{align*}
for all $\eps > 0$. So using $\tilde{m}_t \geq m_t = \|(\mbox{Variance term})\|^2$, there is a universal constant $C$ such that for any $0 < \delta < 1/2$,
\begin{align}\label{ineq:variance-term-bound}
\P\Big[\|(\mbox{Variance term)}\|^2 \geq C\log(t/\delta)\frac{\eta B_1^6 N^2}{\lambda_a^2} \Big] \leq \delta\,.
\end{align}

So combining \eqref{eq:bias-term-bound} and \eqref{ineq:variance-term-bound} with \eqref{eq:linear-optimality-gap},
\begin{align*}
\P\left[\cL_{\lambda_a}(\ba^t) - \cL_{\lambda_a}(\ba^*) \geq C B_1^2 N \left((1-\lambda_a \eta)^{2t} (\|\ba^0\|^2 + \frac{B_1^2}{\lambda_a}) + \log(t / \delta) \frac{\eta B_1^6 N^2}{\lambda_a^2}\right)\right] \leq \delta\,.
\end{align*}
The lemma follows by plugging in the expression for $\cL_{\lambda_a}(\ba^{cert})$ and using that $\ba^*$ is optimal, so $\cL_{\lambda_a}(\ba^*) \leq \cL_{\lambda_a}(\ba)$.
\end{proof}

\begin{lemma}[Martingale high-probability bound]\label{lem:technical-martingale-bound}
There is constant $C > 0$ such the the following holds. Suppose that $X_0,\ldots,X_t,\ldots$ are nonnegative random variables and are such that $X_0 = 0$, and $\E[X_{t+1} \mid \cF_t] \leq a + c X_t$ and almost surely $|X_{t+1} - cX_t| \leq M + 2\sqrt{cMX_t}$ for constants $M, a \geq 0$ and $0 < c < 1$. Then for any $t$ and $\eps \geq \max(M/c, a^2 / Mc)$,
\begin{align*}
\P \Big[X_t \geq \frac{a}{1-c} + \eps \Big] \leq t \exp \Big(-\frac{\eps (1-c^2)}{C M} \Big)\,.
\end{align*}
\end{lemma}
\begin{proof}
Construct $Z_t = c^{-t} (X_t - \frac{a}{1-c})$. Then $Z_t$ is a super-martingale:
\begin{align*}
\E[Z_{t+1} \mid \cF_t] \leq c^{-t-1}a + c^{-t}X_t - \frac{c^{-t-1}a}{1-c} = Z_t + a c^{-t-1} \Big(1 - \frac{1}{1-c} + \frac{c}{1-c} \Big) \leq Z_t\,.
\end{align*}

Let $\tau = \inf \{t \geq 0: X_t \geq \iota\}$ be a stopping time for some $\iota > a / (1-c)$. Then $\tilde{Z}_t = Z_{\min(t,\tau)}$ is also a super-martingale. Furthermore, we have the bounded differences:
\begin{align*}
|\tilde{Z}_{t+1} - \tilde{Z}_t| &\leq |c^{-t-1}(X_{t+1} - cX_t - a)| \\
&\leq c^{-t-1}|M + 2\sqrt{cMX_t} + a| \\
&\leq c^{-t-1} (M+a + 2\sqrt{cM\iota}) =: c^{-t-1}\tilde{M} ,
\end{align*}
if $t < \tau$ and $|\tilde{Z}_{t+1} - \tilde{Z}_t| = 0$ if $t \geq \tau$.

So by the Azuma-Hoeffding inequality, since $Z_0 \leq 0$,
\begin{align*}
\P[\tilde{Z}_t \geq \eps] \leq \exp \Big(-\eps^2 / (2\sum_{j=1}^{t} c^{-2j} \tilde{M}^2) \Big) \leq \exp \Big(-\frac{1}{2}(\eps / \tilde{M})^2 c^{2t} (1-c^2) \Big)\,.
\end{align*}
Let $E$ be the event that $\tilde{Z}_{t'} < c^{-{t'}}(\iota - (a/(1-c)) \mbox{ for all } t' \in \{0,\ldots,t\}$. By a union bound,
\begin{align*}
\P[E] \geq 1  - t \exp \Big(-\frac{(\iota - (a/(1-c)))^2(1-c^2)}{2\tilde{M}^2} \Big) \geq 1 - t\exp \Big(-\frac{\iota^2(1-c^2)}{2\tilde{M}^2} \Big)\,.
\end{align*}
Finally, note that under event $E$ we have $\tilde{Z}_t = Z_t$, and $X_t < \frac{1}{1-c} + \iota$. And for $\iota \geq \max(M / c, a^2 / Mc)$ we have $\tilde{M}^2 \leq 16 c M \iota$.
\end{proof}

\clearpage

\section{Lower bounds for linear methods and CSQ methods}\label{app:lower_bounds}

\subsection{Linear methods}
We define our general linear methods as follows (see for example \cite[Appendix H]{abbe2022merged} for additional discussion). Fix a Hilbert space $(\cH, \<\cdot, \cdot\>_{\cH})$ and a feature map $\psi : \R^d \to \cH$. Given data points $(y_i,\bx_i)_{i \in [n]}$, the linear method constructs weights $\hat{\ba} \in \cH$ by minimizing the regularized empirical risk for some loss function $L : \R^{2n} \to \R \cup \{\infty\}$ and some regularization parameter $\lambda > 0$,
\begin{align*}
\hat{\ba} = \arg\min_{\ba \in \cH} \{L((y_i, \<\ba,\psi(\bx_i)\>)_{i \in [n]}) + \lambda \|\ba\|_{\cH}^2\},
\end{align*}
and estimates the target function using the linear prediction model
\begin{align*}
\hf(\bx) = \<\hat{\ba}, \psi(\bx)\>\,.
\end{align*}

The takeaway of this section is that to learn any degree-$D$ functions with small support on isotropic data, linear methods must pay at least $\Omega(d^D)$ samples (and ``width'' $\dim(\cH) \geq d^D$) when the support is not known. 
This is proved by \cite{abbe2022merged} in the case of the binary hypercube:
\begin{proposition}[Limitations for linear methods on hypercube, cf. Proposition~11 of \cite{abbe2022merged}]\label{prop:degree-linear-boolean}
Let $h_* : \{+1,-1\}^P \to \R$ be a function given by
\begin{align*}
h_*(\bz) = \sum_{S \subseteq [P]} \hat{h}(S) \prod_{i \in S} z_i\,.
\end{align*}
Let $D = \max \{|S| : \hat{h}_*(S) \neq 0\}$ be the degree of $h_*$.
Consider the class of functions which depend as $h_*$ on some subset of coordinates $$\cF = \cup_{\sigma \in S_d}\{f_{*,\sigma} : \{+1,-1\}^d \to \R, \mbox{ where } f_{*,\sigma}(\bx) = h_*(x_{\sigma(1)},\ldots,x_{\sigma(P)})\}.$$
For any linear method, let $\hf_{\sigma}$ be the function estimated by the linear method on (possibly noisy) samples $(\bx_i,f_{*,\sigma}(\bx_i) + \epsilon_i)_{i \in [n]}$. Then there are constants $C_{h_*},c_{h_*} > 0$ such that
\begin{align*}
\frac{1}{|S_d|} \sum_{\sigma \in S_d} \E_{\bx \sim \{+1,-1\}^d}[(f_{*,\sigma}(\bx) - \hf_{\sigma}(\bx))^2] \geq c_{h_*} - C_{h_*} \min(n,\dim(\cH)) d^{-D}\,.
\end{align*}
\end{proposition}
\begin{proof}
Apply Proposition~11 of \cite{abbe2022merged}, letting $\Omega$ be the subspace of $f \in L^2(\{+1,-1\}^d)$ that are degree-$D$ homogeneous. Then $\max_{\sigma} \frac{1}{|S_d|} \sum_{\sigma' \in S_d} |\E[f_{*,\sigma}(\bx) \proj_{\Omega}f_{*,\sigma'}(\bx)]| \leq O(d^{-D})$.
\end{proof}
We now give an analogous result for the Gaussian data distribution, where the degree also drives the complexity for linear methods. This bound is new and was not derived in \cite{abbe2022merged}.
\begin{proposition}[Limitations for linear methods on Gaussian data]\label{prop:degree-linear-gaussian}
Let $h_* : \R^P \to \R$ be a function given by
\begin{align*}
h_*(\bz) = \sum_{S = (k_1,\ldots,k_P) \in \N^P} \hat{h}_*(S) \prod_{i \in [P]} \He_{k_i}(z_i)\,.
\end{align*}
Let $D = \max \{\sum_{i} k_i : \hat{h}_*(S) \neq 0\}$ be the degree of $h_*$.
Consider the class of functions which depend as $h_*$ on some subspace of coordinates $$\cF = \bigcup_{\substack{\bM \in \R^{P \times d} \\ \bM\bM^\top = \bI}}\{f_{*,\bM} : \R^d \to \R, \mbox{ where } f_{*,\bM}(\bx) = h_*(\bM\bx)\}.$$
For any linear method, let $\hf_{\bM}$ be the function estimated by the linear method on (possibly noisy) samples $(\bx_i,f_{*,\bM}(\bx_i) + \epsilon_i)_{i \in [n]}$. Then there are constants $C_{h_*},c_{h_*} > 0$ such that with respect to a uniformly random $\bM \sim \R^{P \times d}$, satisfying $\bM\bM^{\top} = \bI$, we have
\begin{align*}
\E_{\bM}[\E_{\bx \sim \normal(0,I_d)}[(f_{*,\bM}(\bx) - \hat{f}_{\bM}(\bx))^2]] \geq c_{h_*} - C_{h_*} \min(n,\dim(\cH))d^{-D}\,.
\end{align*}
\end{proposition}
\begin{proof}
First, we we can write a degree-$D$ monomial as a linear combination of functions in $\cF$.
\begin{claim}
There are semiorthogonal matrices $\bM^1,\ldots,\bM^{2^D}$ and coefficients $b_1,\ldots,b_{2^D}$ such that
\begin{align*}
\prod_{i\in [D]} x_i = \sum_{j=1}^{2^D} b_j h_*(\bM^j \bx)\,.
\end{align*}
Furthermore, for all $j$ we have $|a_j| \leq C_{h_*}$, which is a constant depending only on $h_*$.
\end{claim}
\begin{proof}[Proof of claim]

Let $S = (k_1,\ldots,k_P)$ such that $\hat{h}_*(S) \neq 0$ and $\sum_i k_i = D$. Define the prefix sums $s_i = \sum_{i' < i} k_i$. Then for each $\bdelta \in \{+1,-1\}^D$, let $\bR^{\delta} \in \R^{P \times d}$ be the matrix which for any $i \in [P]$ satisfies 
$$\bR^{\bdelta}_{i,:} = \frac{1}{\sqrt{k_i}}\sum_{j=1}^{k_i} \delta_{s_i + j} \be_{s_i + j}.$$
Notice that $\bR^{\bdelta} (\bR^{\bdelta})^{\top} = \bI$, so this is a valid semi-orthogonal matrix, and so $h_*(\bR^{\bdelta} \bx) \in \cF$. Now let us show that we can write the monomial as a linear combination of functions of the form $h_*(\bR^{\bdelta} \bx)$. Specifically,
for any $S' = (k_1',\ldots,k_P')$ with $\sum_{i} k_i' \leq D$ we have
\begin{align}\label{eq:construct-monomial-gaussian}
\E_{\bdelta \sim \{\pm 1\}^D}&\left[\left(\prod_{j \in [D]} \delta_j\right) \left(\prod_{i \in [P]} \He_{k'_i}((\bR^{\bdelta} \bx)_i)\right)\right] 
= \prod_{i \in [P]} \E_{\bdelta \sim \{\pm 1\}^{k_i}}\left[\left(\prod_{j \in [k_i]} \delta_j\right) \He_{k'_i}(\frac{1}{\sqrt{k_i}}\sum_{j=1}^{k_i} \delta_j x_{j+s_i}) \right] \nonumber \\
&\propto \prod_{i \in [P]} \begin{cases} 0, & 
k_i' < k_i \\
\prod_{j=1}^{k_i} x_{j+s_i}, & k_i' = k_i \\
\mbox{ something else}, & k_i' > k_i\end{cases} \nonumber\\
&\propto 1(S = S') \prod_{j \in [D]} x_j\,,\end{align}
with a nonzero proportionality constant that only depends on $S$. Therefore,
\begin{align*}
\sum_{\bdelta \in \{\pm 1\}} (\prod_{i=1}^D \delta_i) h_*(\bR^{\bdelta} \bx) \propto \prod_{i \in [D]} x_i\,,
\end{align*}
with a nonzero proportionality constant that only depends on $h_*$. This proves the claim.
\end{proof}

We will use this claim to lower-bound the error of the linear method on $\cF$. Notice that the linear method must predict $\<\hat{\ba}, \psi(\bx)\>$, where $\hat{\ba} \in \Span\{\psi(\bx_i)\}_{i \in [n]}$. So the error is lower-bounded by the norm of the orthogonal projection to this subspace. For $\bx \sim \normal(0,I_d)$ throughout,
\begin{align*}
\E_{\bM}[\E_{\bx}[(h_*(\bM\bx) - \hat{f}_{\bM}(\bx))^2]] &\geq \E_{\bM}[\min_{\ba \in \Span\{\psi(\bx_i)\}_{i \in [n]}}\E_{\bx}[(h_*(\bM\bx) - \<\ba, \psi(\bx)\>)^2]] = (\ast)\,.
\end{align*}
Now let $\bM^1,\ldots,\bM^{2^D}$ and $b_1,\ldots,b_{2^D}$ be the matrices and coefficients from the claim. Let $\bR \in \R^{d \times d}$ be a uniformly random rotation and let $\sigma$ be a uniformly random permutation. Since $\bM^i \sigma \bR$ has the same distribution as $\bM$,
\begin{align*}
(\ast) &= \E_{\bM}[\min_{\ba \in \Span\{\psi(\bx_i)\}_{i \in [n]}}\E_{\bx}[(h_*(\bM \bx) - \<\ba, \psi(\bx)\>)^2]] \\
&\geq \frac{1}{2^D (\sum_{i=1}^{2^D}b_i^2)}\E_{\bR,\sigma}[\min_{\ba \in \Span\{\psi(\bx_i)\}_{i \in [n]}}\E_{\bx}[(\sum_{i=1}^{2^D} b_i h_*(\bM^i \sigma \bR) - \<\ba, \psi(\bx)\>)^2]] \\
&= \frac{1}{2^D (\sum_{i=1}^{2^D}b_i^2)} \E_{\bR,\sigma}[\min_{\ba \in \Span\{\psi(\bx_i)\}_{i \in [n]}}\E_{\bx}[(\prod_{i\in [D]}(\bR \bx)_{\sigma(i)} - \<\ba, \psi(\bx)\>)^2]] = (\ast\ast)\,.
\end{align*}
However, Proposition~11 of \cite{abbe2022merged} provides a lower-bound on the error for learning the class $\cup_{\sigma \in S_d}\{\prod_{i \in [D]} (\bR \bx)_{\sigma(i)}\}$ with a linear method. Specifically, for two permutations $\sigma,\sigma'$ such that $\sigma([D]) \neq \sigma'([D])$, we have $\E_{\bx}[\prod_{i \in [D]}(\bR \bx)_{\sigma(i)} (\bR \bx)_{\sigma'(i)}] = \E_{\bx}[\prod_{i \in [D]}x_{\sigma(i)} x_{\sigma'(i)}] = 0$. So Proposition~11 of \cite{abbe2022merged} implies that there is some constant $C$ depending only on $D$, such that
\begin{align*}
(\ast\ast) \geq c_{h_*}(1 - C \min(n, \dim(\cH))d^{-D})\,.
\end{align*}
Putting together the equations proves the lemma.

\end{proof}

\subsection{Correlational Statistical Query (CSQ) methods}

A Correlational Statistical Query (CSQ) algorithm \citep{bendavid1995learning,bshouty2002using,reyzin2020statistical} accesses the data via queries $\phi : \R^d \to [-1,1]$ and returns $\E_{\bx,y}[\phi(\bx)y]$ up to some error tolerance $\tau$. In our case, since $y = f_*(\bx) + \eps$, where $\eps$ is independent zero-mean noise, the query returns a value in $\E_{\bx}[\phi(\bx)f_*(\bx)] + [-\tau,+\tau]$. The CSQ algorithm outputs a guess $\hat{f}$ of the true function $f_*$. An example of a CSQ algorithm is gradient descent on the population square loss if we inject noise in the gradients (see, e.g., \cite{abbe2022non}).

First, we give a lower bound on the CSQ complexity of learning a function with leaps when $\bx$ is drawn uniformly from the hypercube. The below lower bound is qualitatively similar to the argument in \cite{abbe2022non} based on the ``alignment'' quantity. The bounds of \cite{abbe2022non} have tighter constants in the exponents of the bound, but they have the disadvantage that they apply only to noisy population gradient descent instead of to general CSQ algorithms.
\begin{proposition}[Limitations for CSQ algorithms on hypercube]\label{prop:leap-csq-boolean}
Let $h_* : \{+1,-1\}^P \to \R$, and let $\Leap(h_*)$ be its leap. Consider the class of functions given by applying $h_*$ on some subset of coordinates
\begin{align*}
\cF = \cup_{\sigma \in S_d} \{f_{*,\sigma} : \{+1,-1\}^d \to \R, \mbox{ where } f_{*,\sigma}(\bx) = h_*(x_{\sigma(1)},\ldots,x_{\sigma(P)})\}\,.
\end{align*}
Then a CSQ algorithm with $n$ queries of error tolerance $\tau$ outputs $\hat{f}$ such that with probability $\geq 1 - C_{h_*}nd^{-\Leap(h_*)} / \tau^2$ over the random choice of $f_* \sim \cF$,
\begin{align*}
\E_{\bx \sim \{+1,-1\}^d}[(f_*(\bx) - \hat{f})^2] \geq c_{h_*} > 0\,.
\end{align*}
\end{proposition}
\begin{proof}
For any subset $T$, define $\cS^{\not\subseteq T} = \{S \subseteq [P] : S \not\subseteq T, \hat{h}_*(S) \neq 0\}$. By definition of the leap, there is a subset $T \subseteq [P]$ such that $\cS^{\not\subseteq T} \neq \emptyset$ and for all $S \in \cS^{\not\subseteq T}$ we have $|S \setminus T| \geq \Leap(h_*)$.
Without loss of generality, assume $T = \{1,\ldots,k\} \subseteq [P]$. Write
\begin{align*}
h_*(\bz) = h_{\subseteq T}(\bz) + h_{\not\subseteq T}(\bz)\,,
\end{align*}
where
\begin{align*}
h_{\subseteq T}(\bz) = \sum_{S \subseteq T} \hat{h}_*(S) \prod_{i \in S} z_i, \mbox{ and } h_{\not\subseteq T}(\bz) = \sum_{S \not\subseteq T} \hat{h}_*(S) \prod_{i \in S} z_i\,.
\end{align*}
Suppose that the CSQ algorithm knows $\sigma(1),\ldots,\sigma(k)$, which can only help it. Then the problem of learning $f_{*,\sigma}$ from CSQ queries is equivalent to the problem of learning
$f_{\not\subseteq T,\sigma}(\bx) = h_{\not\subseteq T}(x_{\sigma(1)},\ldots,x_{\sigma(P)})$ from CSQ queries.
However, for random permutations $\sigma'$ conditioned on $\sigma'(1) = \sigma(1),\ldots,\sigma'(k) = \sigma(k)$ we have
\begin{align*}
\cC &= \sup_{\phi \in L^2(\{+1,-1\}^d), \|\phi\|^2 = 1} \E_{\sigma'}[\<f_{\not\subseteq T,\sigma'},\phi\>^2] \leq \E_{\sigma'}\left[\left(\sum_{S \in \cS^{\not\subseteq T}} \hat{\phi}(\sigma'(S)) \hat{h}_*(S)\right)^2 \right] \\
&\leq C_{h_*} \max_{S \in \cS^{\not\subseteq T}} \E_{\sigma'}[\hat{\phi}(\sigma'(S))^2] \leq C_{h_*} \binom{d-k}{\Leap(h_*)} \leq C_{h_*} d^{-\Leap(h_*)}.
\end{align*}
So by a union bound, with probability $\geq 1 - \cC n / \tau^2$ all $n$ first CSQ queries can return 0. The final output $\hat{f}$ of the algorithm can also be viewed as a statistical query. So with probability at least $1 - n\cC / \tau^2 - \cC / \eps^2$,
\begin{align*}
\E_{\bx}[(f_{*,\sigma}(\bx) - \hat{f}(\bx))^2] &= \|f_{*,\sigma} - f_{\subseteq T,\sigma}\|^2 + \|\hf - f_{\subseteq T,\sigma}\|^2 - 2\<f_{*,\sigma} - f_{\subseteq T,\sigma},\hf - f_{\subseteq T,\sigma}\> \\
&\geq \|f_{\not\subseteq T,\sigma}\|^2 -(\|\E_{\sigma'}[f_{\not\subseteq T,\sigma'}]\| + \eps)^2\,.
\end{align*}
The proposition follows by letting $\eps$ be a small enough positive constant depending on $h_*$.
\end{proof}

\begin{proposition}[Limitations for CSQ algorithms on Gaussian data]\label{prop:isoleap-csq-gaussian} Let $h_* : \R^P \to \R$ be a polynomial of finite degree $D$. Let 
\[
\mathrm{isoLeap}(h_*) = \max_{R \in \cO_P} \mathrm{Leap}(h_*, R)
\]
be its isotropic leap (as defined in Appendix \ref{app:discussion_leap}). Consider the class of functions which given by applying $h_*$ on some subspace of coordinates $$\cF = \bigcup_{\substack{\bM \in \R^{P \times d} \\ \bM\bM^\top = \bI}}\{f_{*,\bM} : \R^d \to \R, \mbox{ where } f_{*,\sigma}(\bx) = h_*(\bM\bx)\}.$$
Then, for any CSQ algorithm with $n$ queries of tolerance $\pm \tau$, with probability $1 - \frac{C_{h_*}n}{\tau^2}d^{-\mathrm{isoLeap}(h_*)/2}$ over the random choice of $\bM$, the estimator $\hat{f}$ it returns for $f_* \in \cF$ satisfies
\begin{align*}
\E_{\bx \sim \normal(0,I_d)}[(f_*(\bx) - \hat{f}(\bx))^2] \geq c_{h_*} > 0\,.
\end{align*}
\end{proposition}
\begin{proof}
The proof follows a similar strategy to that of the previous proposition. Without loss of generality, suppose that $\Leap(h_*) = \mathrm{isoLeap}(h_*)$ (in other words, that we are already in a basis that maximizes the leap without having to apply a rotation). Then there must be $T \subseteq [P]$ which we can take to be $T = \{1,\ldots,r\}$ without loss of generality such that if we define
\begin{align*}h_{\subseteq T}(\bz) = \sum_{S = (k_1,\ldots,k_r) \in \N^r} \hat{h}_*(S) \prod_{i \in [P]} \He_{k_i}(z_i) \mbox{ and } h_{\not\subseteq T}(\bz) = h_*(\bz) - h_{\subseteq T}(\bz)\,,
\end{align*}
we must have $h_{\not\subseteq T} \not\equiv 0$, and for each nonzero Fourier coefficient $S = (k_1,\ldots,k_P)$ such that $\hat{h}_{\not\subseteq T}(S) \neq 0$ we must have $\sum_{i \in [P] \setminus T} k_i \geq \Leap(h_*)$. Knowing the first $r$ rows $\bM_{1,:},\ldots,\bM_{r,:}$ can only help the CSQ algorithm, so we just have to show that over the choice of random $\bM' \in \R^{P \times d}$ conditioned on $\bM'_{i,:} = \bM_{i,:}$ for $i \in [r]$ we have
\begin{align*}
\cC = \sup_{\phi \in L^2(\normal(0,I_d)), \|\phi\|^2 = 1} \E_{\bM'}[\E_{\bx \sim \normal(0,I_d)}[h_{\not\subseteq T}(\bM' \bx) \phi(\bx)]^2] \leq C_{h_*} d^{-\Leap(h_*) / 2}\,.
\end{align*}

For ease of notation, we consider the case when $T = \emptyset$ since the general case is analogous. This follows from the following claim, where for any $\beta \in \N^d$ be define $\He_{\beta}(\bx) = \prod_{i=1}^d \He_{\beta_i}(x_i)$.

\begin{claim}
Let $\alpha \in \N^d$ be such that with $\alpha_{i} = 0$ for all $i > P$. Let $\bR \in \R^{d \times d}$ be a random rotation drawn according to the Haar measure. Then
$$\sup_{\phi \in L^2(\normal (0,\id_d)), \|\phi\|^2 = 1} \E_{\bR}[\E_{\bx \sim \normal(0,I_d)}[\He_{\alpha}(\bR \bx)\phi(\bx)]^2] = O(d^{-\ceil{\|\alpha\|_1/2}})\,.$$
\end{claim}
\begin{proof}
Write $\phi(\bx) = \sum_{\substack{\beta \in \N^d}} \hat{\phi}(\beta) \He_{\beta}(\bx)$. By integration by parts,
\allowdisplaybreaks
\begin{align*}
&\;\;\;\;\;\E_{\bx \sim \normal(0,I_d)}[\He_{\alpha}(\bR\bx) \phi(\bx)] \\
&=  \E_{\bx_{-1}}\left[\prod_{i \neq 1} \He_{\alpha_i}(x_i)\E_{x_1}[\He_{\alpha_1}(x_1) \phi(\bR^{\top}\bx)]\right] \\
&= \E_{\bx_{-1}} \left[ \prod_{i \neq 1} \He_{\alpha_i}(x_i) \E_{x_1}\left[ \frac{\partial^{\alpha_1}}{\partial x_1^{\alpha_1}} \phi(\bR^{\top} \bx)\right]\right] \\
&= \E_{\bx_{-1}}\left[\prod_{i \neq 1} \He_{\alpha_i}(x_i) \E_{x_1}\left[\sum_{j_1,\ldots,j_{\alpha_1} \in [d]} \left( \prod_{k=1}^{\alpha_1} R_{1,j_k} \frac{\partial}{\partial z_{j_k}} \right) \phi(\bz) \mid_{\bz = \bR^{\top} \bx}\right]\right] \\
&= \E_{\bx}\left[\sum_{j_{1,1},\ldots,j_{1,\alpha_1} \in [d]} \dots \sum_{j_{P,1},\ldots,j_{P,\alpha_P} \in [d]} \left(\prod_{l=1}^{\alpha_1} R_{1,j_{1,l}} \frac{\partial}{\partial z_{j_{1,l}}}\right) \dots \left(\prod_{l=1}^{\alpha_P} R_{P,j_{P,l}} \frac{\partial}{\partial z_{j_{P,l}}}\right) \phi(\bz) \mid_{\bz = \bR^{\top} \bx}\right] \\
&= \sum_{\substack{\Upsilon \in \Z_{\geq 0}^{P \times d} \\ 
\Upsilon \bones = \alpha_{1:P}}}  \E_{\bx}\left[\left(\prod_{i \in [P],j \in [d]} R_{i,j}^{\Upsilon_{i,j}} (\Upsilon_{i,j}!) \frac{\partial^{\Upsilon_{i,j}}}{{\partial z_{j}^{\Upsilon_{i,j}}}}\right) \phi(\bz) \mid_{\bz = \bR^{\top} \bx}\right] \\
&= \sum_{\substack{\Upsilon \in \Z_{\geq 0}^{P \times d} \\ 
\Upsilon \bones = \alpha_{1:P}}} 
\E_{\bx}\left[\left(\prod_{i \in [P],j \in [d]} R_{i,j}^{\Upsilon_{i,j}} (\Upsilon_{i,j}!) \frac{\partial^{\Upsilon_{i,j}}}{{\partial x_{j}^{\Upsilon_{i,j}}}} \right) \phi(\bx)\right] \\
&= \sum_{\substack{\Upsilon \in \Z_{\geq 0}^{P \times d} \\ 
\Upsilon \bones = \alpha_{1:P}}}  C_{\Upsilon} \left(\prod_{i \in [P],j \in [d]} R_{i,j}^{\Upsilon_{i,j}} \right) \hat{\phi}(\bones^{\top} \Upsilon)\,.
\end{align*}

So
\begin{align*}
& \;\;\;\;\; \E_{\bR}[\E_{\bx}[\He_{\alpha}(\bR\bx) \phi(\bx)]^2] \\
&= \sum_{\substack{\Upsilon,\Upsilon' \in \Z_{\geq 0}^{P \times d} \\ \Upsilon \bones = \Upsilon' \bones = \alpha_{1:P}}} C_{\Upsilon} C_{\Upsilon'} \hat{\phi}(\bones^{\top} \Upsilon) \hat{\phi}(\bones^{\top} \Upsilon') \E_{\bR}\left[ \prod_{i \in [P],j \in [d]} R_{i,j}^{\Upsilon_{i,j} + \Upsilon'_{i,j}}\right] \\
&\lesssim \frac{1}{d^{\|\alpha\|_1}} \sum_{\substack{\Upsilon,\Upsilon' \in \Z_{\geq 0}^{P \times d} \\ \Upsilon \bones = \Upsilon' \bones = \alpha_{1:P}}} C_{\Upsilon} C_{\Upsilon'} \hat{\phi}(\bones^{\top} \Upsilon) \hat{\phi}(\bones^{\top} \Upsilon') 1(\Upsilon_{i,j} + \Upsilon'_{i,j} \in 2\N \mbox{ for all } i, j) \\
&\lesssim \frac{1}{d^{\|\alpha\|_1}} \sum_{\substack{\Upsilon,\Upsilon' \in \Z_{\geq 0}^{P \times d} \\ \Upsilon \bones = \Upsilon' \bones = \alpha_{1:P}}} C_{\Upsilon}^2 \hat{\phi}(\bones^{\top} \Upsilon)^2 1(\Upsilon_{i,j} + \Upsilon'_{i,j} \in 2\N \mbox{ for all } i, j) = (\ast)\,.
\end{align*}
We argue that for any matrix $\Upsilon \in \Z_{\geq 0}^{P \times d}$ with $\Upsilon \bones = \alpha$, there are at most $d^{\floor{\|\alpha\|_1/2}}$ matrices $\Upsilon'$ such that $\Upsilon' \bones = \balpha$ and $\Upsilon + \Upsilon'$ has all even entries. Indeed, let $S_{even}(\alpha) \subseteq [P] \times [d]$ denote the coordinates $(i,j)$ where $\Upsilon_{i,j} > 0$ is even. Let $S_{odd}(\Upsilon) \subseteq [P] \times [d]$ denote the coordinates where $\Upsilon_{i,j} > 0$ is odd. Then if $\Upsilon + \Upsilon'$ has all even entries, we must have $S_{odd}(\Upsilon) = S_{odd}(\Upsilon')$. So there are at most $\binom{Pd}{|S_{even}(\Upsilon')|} \leq \binom{Pd}{\floor{\|\alpha\|_1/2}} \lesssim d^{\floor{\|\alpha\|_1 / 2}}$ choices for $S_{even}(\Upsilon')$, where we use $S_{even}(\Upsilon') \leq \floor{\|\alpha\|_1 / 2}$. So
\begin{align*}
(\ast) &\lesssim \frac{1}{d^{\|\alpha\|_1}} \sum_{\Upsilon \in \Z^{P \times d}_{\geq 0}, \Upsilon \bones = \alpha_{1:P}} \hat{\phi}(\bones^\top \Upsilon)^2 d^{\floor{\|\alpha\|_1 / 2}} \lesssim d^{-\ceil{\|\alpha\|_1 / 2}}\,,
\end{align*}
where we use the normalization  $\|\phi\|^2 \leq 1$ which implies $\sum_{\beta} \hat{\phi}(\beta)^2 \lesssim 1$. We also use that for each $\beta \in \N^d$, there are at most $C_P$ matrices $\Upsilon \in \Z_{\geq 0}^{P \times d}$ such that $\bones^{\top} \Upsilon = \beta$.
\end{proof}
\end{proof}

\end{document}